\documentclass{article}

\PassOptionsToPackage{numbers, compress}{natbib}

\usepackage[preprint]{neurips_2021}

\usepackage[utf8]{inputenc} 
\usepackage[T1]{fontenc}    
\usepackage{hyperref}       
\usepackage{url}            
\usepackage{booktabs}       
\usepackage{amsfonts}       
\usepackage{nicefrac}       
\usepackage{microtype}      
\usepackage{xcolor}         

\usepackage{amsmath,amsfonts,bm}

\def\eqref#1{equation~\ref{#1}}

\def\1{\bm{1}}

\DeclareMathAlphabet{\mathsfit}{\encodingdefault}{\sfdefault}{m}{sl}
\SetMathAlphabet{\mathsfit}{bold}{\encodingdefault}{\sfdefault}{bx}{n}

\newcommand{\R}{\mathbb{R}}

\DeclareMathOperator*{\argmax}{arg\,max}
\DeclareMathOperator*{\argmin}{arg\,min}

\usepackage{xspace}
\usepackage{bbm}
\usepackage{algorithm}
\usepackage{algorithmicx}
\usepackage{algpseudocode}
\usepackage{etoolbox}
\usepackage{graphicx}
\usepackage{subcaption}
\usepackage{wrapfig}
\usepackage{float}
\usepackage{amsmath}
\usepackage{makecell}
\usepackage{textcomp}

\newcommand{\flood}[0]{\textsc{SEALS}\xspace}
\newcommand{\authoraffiliation}[1]{{\normalfont\textsuperscript{#1}}}

\usepackage{amsthm}
\newtheorem{theorem}{Theorem}

\def\H{\mathcal{H}}

\def\G{\mathcal{G}}
\def\S{\mathcal{S}}

\def\X{\mathcal{X}}

\def\Z{\mathcal{Z}}

\def\Z{\mathcal{Z}}
\def\1{\mathbf{1}}

\def\R{\mathbb{R}}
\def\N{\mathbb{N}}

\newcommand{\norm}[1]{\left\lVert#1\right\rVert}
\def\t{\top}
\newcommand{\dist}{\mathrm{dist}}
\def\data{\S}
\def\conv{\tt{conv}}
\def\spa{\text{span}}

\newtheorem{lemma}{Lemma}

\newtheorem{example}{Example}

\title{Similarity Search for Efficient Active Learning and Search of Rare Concepts}

\author{%
    Cody Coleman\authoraffiliation{1}\thanks{Correspondence: \texttt{cody@cs.stanford.edu}}, Edward Chou\authoraffiliation{2}, Julian Katz-Samuels\authoraffiliation{3}, Sean Culatana\authoraffiliation{2}, Peter Bailis\authoraffiliation{1}, \\
    \textbf{Alexander C. Berg\authoraffiliation{4}, Robert Nowak\authoraffiliation{3}, Roshan Sumbaly\authoraffiliation{2}, Matei Zaharia\authoraffiliation{1}, I. Zeki Yalniz\authoraffiliation{2}} \\
    \textsuperscript{1}Stanford University, \textsuperscript{2}Facebook AI, \textsuperscript{3}University of Wisconsin, \textsuperscript{4}Facebook AI Research \\
}

\begin{document}

\maketitle

\begin{abstract}
Many active learning and search approaches are intractable for large-scale industrial settings with billions of unlabeled examples.
Existing approaches search globally for the optimal examples to label, scaling linearly or even quadratically with the unlabeled data.
In this paper, we improve the computational efficiency of active learning and search methods by restricting the candidate pool for labeling to the nearest neighbors of the currently labeled set instead of scanning over all of the unlabeled data.
We evaluate several selection strategies in this setting on three large-scale computer vision datasets: ImageNet, OpenImages, and a de-identified and aggregated dataset of 10 billion publicly shared images provided by a large internet company.
Our approach achieved similar mean average precision and recall as the traditional global approach while reducing the computational cost of selection by up to three orders of magnitude, enabling \textit{web-scale active learning}.

\end{abstract}
\section{Introduction}\label{sec:intro}
Large-scale unlabeled datasets can contain millions or billions of examples covering a wide variety of underlying concepts~\citep{chelba2013billion, zhang2015character, wan2019spoiler, russakovsky2015imagenet, kuznetsova2020openimages, thomee2016yfcc, abuelhaija2016youtube8m, lee2019accelerating}.
Yet, these massive datasets often skew towards a relatively small number of common concepts, for example `cats', `dogs', and `people'~\citep{liu2019large, zhang2017range, wang2017learning, van2017devil}.
Rare concepts, such as `harbor seals', tend to only appear in a small fraction of the data (usually less than 1\%).
However, performance on these rare concepts is critical in many settings.
For example, harmful or malicious content may comprise only a small percentage of user-generated content, but it can have a disproportionate impact on the overall user experience~\citep{wan2019spoiler}.
Similarly, when debugging model behavior for safety-critical applications like autonomous vehicles, or when dealing with representational biases in models, obtaining data that captures rare concepts allows machine learning practitioners to combat blind spots in model performance~\citep{andrej2018software2, holstein2019improving, ashmawy2019searchable, andrej2020ai}.
Even a simple task, such as stop sign detection by an autonomous vehicle, can be difficult due to the diversity of real-world data.
Stop signs may appear in a variety of conditions (e.g., on a wall or held by a person), can be heavily occluded, or have modifiers (e.g., ``Except Right Turn'')~\citep{andrej2020ai}.
Large-scale datasets are essential but not sufficient; finding the relevant examples for these long-tail tasks is challenging.

Active learning and search methods have the potential to automate the process of identifying these rare, high-value data points, but often become intractable at-scale.
Existing techniques carefully select examples over a series of rounds to improve model quality (active learning~\citep{settles2012active}) or find positive examples in highly skewed settings (active search~\citep{garnett2012bayesian}).
Each selection round iterates over the entire unlabeled data to identify the optimal example or batch of examples to label based on uncertainty (e.g., the entropy of predicted class probabilities) or other heuristics~\citep{settles2011theories, settles2012active, lewis1994sequential, garnett2012bayesian, zhang2020state, beluch2018power, yoo2019learning, pmlr-v70-gal17a, he2007nearest, sinha2019variational, joshi2009multi, settles2008analysis, sener2018active}.
Depending on the selection criteria, each round can scale linearly~\citep{lewis1994sequential, joshi2009multi} or even quadratically~\citep{settles2008analysis, sener2018active} with the size of the unlabeled data.
The computational cost of this process has become an impediment as datasets and model architectures have increased rapidly in size~\citep{openai2018compute}.
Recent work has tried to address this problem with sophisticated methods to select larger and more diverse batches of examples in each selection round and reduce the total number of rounds needed to reach the target labeling budget~\citep{sener2018active, andreas2019batchbald, coleman2020selection, pinsler2019bayesian, jiang2018efficient}.
Nevertheless, these approaches still scan over all of the examples to find the optimal batch for each round, which remains intractable for web-scale datasets with billions of examples.
The selection rounds of these techniques need to scale sublinearly with the unlabeled data size to tackle these massive and heavily skewed problems.

In this paper, we propose Similarity search for Efficient Active Learning and Search (\flood) as a simple approach to further improve computational efficiency and achieve \textit{web-scale active learning}.
Empirically, we find that learned representations from pre-trained models can effectively cluster many unseen rare concepts.
We exploit this latent structure to improve the computational efficiency of active learning and search methods by only considering the nearest neighbors of the currently labeled examples in each selection round rather than scanning over all of the unlabeled data.
Finding the nearest neighbors for each labeled example in the unlabeled data can be performed efficiently with sublinear retrieval times~\citep{charikar2002similarity} and sub-second latency on billion-scale datasets~\citep{JDH17} for approximate approaches.
While this restricted candidate pool of unlabeled examples impacts theoretical sample complexity, our analysis shows that \flood still achieves the optimal logarithmic dependence on the desired error for active learning.
As a result, \flood maintains similar label-efficiency and enables selection to scale with the size of the labeled data and only sublinearly with the size of the unlabeled data, making active learning and search tractable on web-scale datasets with billions of examples.

We empirically evaluated \flood for both active learning and search on three large scale computer vision datasets: ImageNet~\citep{russakovsky2015imagenet}, OpenImages~\citep{kuznetsova2020openimages}, and a de-identified and aggregated dataset of 10 billion publicly shared images from a large internet company.
We selected 611 concepts spread across these datasets that range in prevalence from 0.203\% to 0.002\% (1 in 50,000) of the training examples.
We evaluated three selection strategies for each concept: max entropy uncertainty sampling~\citep{lewis1994sequential}, information density~\citep{settles2008analysis}, and most-likely positive~\citep{warmuth2002active, warmuth2003active, jiang2018efficient}.
Across datasets, selection strategies, and concepts, \flood achieved similar model quality and nearly the same recall of the positive examples as the baseline approaches, while reducing the computational cost by up to three orders of magnitude.
Consequently, \flood could perform several selection rounds over 10 billion images in seconds with a single  machine, unlike the baselines that needed a cluster with tens of thousands of cores.
To our knowledge, no other works have performed active learning at this scale.


\section{Related Work}\label{sec:related}

\textbf{Active learning}'s iterative retraining combined with the high computational complexity of deep learning models has led to significant work on computational efficiency.
Much of the recent work has focused on selecting large batches of data to minimize the amount of retraining and reduce the number of selection rounds necessary to reach a target budget~\citep{sener2018active, andreas2019batchbald, pinsler2019bayesian}.
These approaches introduce novel techniques to avoid selecting highly similar or redundant examples and ensure the batches are both informative and diverse, but still require at least linear work over the whole unlabeled set for each selection round.
Our work reduces the number of examples considered in each selection round such that active learning scales sublinearly with the size of the unlabeled dataset.

Others have tried to improve computational efficiency by using much smaller models as cheap proxies, generating examples, or subsampling data.
A smaller model reduces the computation required per example~\citep{yoo2019learning, coleman2020selection}; but unlike our approach, it still requires passing over all of the unlabeled examples.
Generative approaches~\citep{mayer2020adversarial, zhu2017generative, lin2018active} enable sublinear selection runtime complexities; but they struggle to match the label-efficiency of traditional approaches due to the highly variable quality of the generated examples.
Subsampling the unlabeled data as in~\citep{ertekin2007learning} also avoids iterating over all of the data.
However, for rare concepts in web-scale datasets, randomly chosen examples are extremely unlikely to be close enough to the decision boundary (see Section~\ref{appendix:random_pool} in the Appendix).
Our work both achieves sublinear selection runtimes and matches the label-efficiency of traditional approaches.

There are also specific optimizations for certain families of models.
\citet{jain2010hashing} developed custom hashing schemes for the weights from linear SVM classifiers to efficiently find examples near the decision boundary.
While this enables a sublinear selection runtime complexity similar to SEALS, the hashing hyperplanes approach is non-trivial to generalize to other families of models and even SVMs with non-linear kernels.
Our work extends to a wide variety of models and selection strategies.

$k$-nearest neighbor ($k$-NN) classifiers are also advantageous because they do not require an explicit training phase~\citep{he2007nearest, joshi2012coverage, wei2015submodularity, garnett2012bayesian, jiang2017efficient, jiang2018efficient}.
The prediction and score for each unlabeled example can be updated immediately after each new batch of labels.
These approaches still require evaluating all of the data, which can be prohibitively expensive on large-scale datasets.
Our work targets the selection phase rather than training and uses $k$-NNs to limit candidate examples and not as a classifier.

\textbf{Active search} is a sub-area of active learning that focuses on highly-skewed class distributions~\citep{garnett2012bayesian, jiang2017efficient, jiang2018efficient, jiang2019cost}.
Rather than optimizing for model quality, active search aims to find as many examples from the minority class as possible.
Prior work has focused on applications such as drug discovery, where dataset sizes are limited, and labeling costs are exceptionally high.
Our work similarly focuses on skewed distributions.
However, we consider novel active search settings in vision and text where the available unlabeled datasets are much larger, and computational efficiency is a significant bottleneck.

\section{Problem Statement}\label{sec:methods}
This section formally outlines the problems of active learning (Section~\ref{sec:active_learning}) and search (Section~\ref{sec:active_search}) as well as the selection methods we evaluated.
For both, we examined the pool-based batch setting, where examples are selected in batches to improve computational efficiency.

\subsection{Active Learning}\label{sec:active_learning}
Pool-based active learning is an iterative process that begins with a large pool of unlabeled data $U=\{\mathbf{x}_1, \ldots, \mathbf{x}_n\}$.
Each example is sampled from the space $\mathcal{X}$ with an unknown label from the label space $\mathcal{Y}=\{1,\ldots,C\}$ as $(\mathbf{x}_i, y_i)$.
We additionally assume a feature extraction function $G_z$ to embed each $\mathbf{x}_i$ as a latent variable $G_z(\mathbf{x}_i) = \mathbf{z}_i$ and that the $C$ concepts are unequally distributed.
Specifically, there are one or more valuable rare concepts $R \subset C$ that appear in less than 1\% of the unlabeled data.
For simplicity, we frame this as $|R|$ binary classification problems solved independently rather than one multi-class classification problem with $|R|$ concepts.
Initially, each rare concept has a small number of positive examples and several negative examples that serve as a labeled seed set $L^0_r$.
The goal of active learning is to take this seed set and select up to a budget of $T$ examples to label that produce a model $A^T_r$ that achieves low error.
For each round $t$ in pool-based active learning, the most informative examples are selected according to the selection strategy $\phi$ from a pool of candidate examples $\mathcal{P}_r$ in batches of size $b$ and labeled, as shown in Algorithm~\ref{algo:baseline}.

For the baseline approach, $\mathcal{P}_r = \{G_z(\mathbf{x}) \mid \mathbf{x} \in U\}$, meaning that all the unlabeled examples are considered to find the global optimal according to $\phi$.
Between each round, the model $A^t_r$ is trained on all of the labeled data $L^t_r$, allowing the selection process to adapt.

In this paper, we considered \textbf{max entropy (MaxEnt)} uncertainty sampling~\citep{lewis1994sequential}:

$$\phi_{\text{MaxEnt}}(\mathbf{z}, A_r, \mathcal{P}_r) = - \sum_{\hat{y}} P(\hat{y}|\mathbf{z}; A_r) \log P(\hat{y}|\mathbf{z}; A_r)$$

and \textbf{information density (ID)}~\citep{settles2008analysis}:

$$\phi_{\text{ID}}(\mathbf{z}, A_r, \mathcal{P}_r) = \phi_{\text{MaxEnt}}(\mathbf{z}) \times \left(\frac{1}{|\mathcal{P}_r|}\sum_{\mathbf{z}_{p} \in \mathcal{P}_r}\text{sim}(\mathbf{z}, \mathbf{z}_p) \right)^{\beta}$$

where $\text{sim}(\mathbf{z}, \mathbf{z}_p)$ is the cosine similarity of the embedded examples and $\beta = 1$.
Note that for binary classification, MaxEnt is equivalent to least confidence and margin sampling, which are also popular criteria for uncertainty sampling~\citep{settles2009active}.
While MaxEnt uncertainty sampling only requires a linear pass over the unlabeled data, ID scales quadratically with $|U|$ because it weighs each example's informativeness by its similarity to all other examples.
To improve computational performance, the average similarity scores can be cached after the first round so that subsequent rounds scale linearly.

We explored the greedy k-centers approach from \citet{sener2018active} but found that it never outperformed random sampling for our experimental setup.
Unlike MaxEnt and ID, k-centers does not consider the predicted labels.
It tries to achieve high coverage over the entire candidate pool, of which rare concepts make up a small fraction by definition, making it ineffective for our setting.

\subsection{Active Search}\label{sec:active_search}
Active search is closely related to active learning, so much of the formalism from Section~\ref{sec:active_learning} carries over.
The critical difference is that rather than selecting examples to label that minimize error, the goal of active search is to maximize the number of examples from the target concept $r$, expressed with the natural utility function $u(L_r)=\sum_{(\mathbf{x},y)\in L_r} \mathbbm{1}\{y = r\}$).
As a result, different selection strategies are favored, but the overall algorithm is the same as Algorithm~\ref{algo:baseline}.

In this paper, we consider an additional selection strategy to target the active search setting, \textbf{most-likely positive (MLP)}~\citep{warmuth2002active, warmuth2003active, jiang2018efficient}:

$$\phi_{\text{MLP}}(\mathbf{z}, A_r, \mathcal{P}_r) = P(r|\mathbf{z}; A_r)$$

Because active learning and search are similar, we evaluate all selection criteria from Sections~\ref{sec:active_learning} and~\ref{sec:active_search} in terms of the error the model achieves and the number of positives.

\begin{minipage}{0.48\textwidth}
\begin{algorithm}[H]
\begin{algorithmic}[1]
    \Require unlabeled data $U$, labeled seed set $L^0_r$, feature extractor $G_z$, selection strategy $\phi(\cdot)$, batch size $b$, labeling budget $T$
    \item[]
    \State $\mathcal{L}_r = \{(G_z(\mathbf{x}), y) \mid (\mathbf{x}, y) \in L^0_r \}$
    \State $\mathcal{P}_r = \{G_z(\mathbf{x})\mid\mathbf{x} \in U \text{ and } (\mathbf{x}, \cdot) \not \in L^0_r \}$
    \Repeat
        \State $A_r = \text{train}(\mathcal{L}_r) $
        \For{$1$ to $b$}
        \State $\mathbf{z}^* =\argmax_{\mathbf{z} \in \mathcal{P}_r} \phi(\mathbf{z}, A_r, \mathcal{P}_r)$
            \State $\mathcal{L}_r = \mathcal{L}_r \cup \{{(\mathbf{z}^*, \text{label}(\mathbf{x}^*))}\}$
            \State $\mathcal{P}_r = \mathcal{P}_r \setminus \{ \mathbf{z}^*\}$
        \EndFor
    \Until{$|\mathcal{L}_r| = T$}
\end{algorithmic}
\caption{\textsc{Baseline approach}}\label{algo:baseline}
\end{algorithm}
\end{minipage}
\hfill
\begin{minipage}{0.48\textwidth}
\begin{algorithm}[H]
\begin{algorithmic}[1]
    \Require unlabeled data $U$, labeled seed set $L^0_r$, feature extractor $G_z$, selection strategy $\phi(\cdot)$, batch size $b$, labeling budget $T$, $k$-nearest neighbors implementation $\mathcal{N}(\cdot, \cdot)$
    \State $\mathcal{L}_r = \{(G_z(\mathbf{x}), y) \mid (\mathbf{x}, y) \in L^0_r \}$
    \State $\mathcal{P}_r = \cup_{(\mathbf{z}, y) \in \mathcal{L}_r} \mathcal{N}(\mathbf{z}, k)$
    \Repeat
        \State $A_r = \text{train}(\mathcal{L}_r) $
        \For{$1$ to $b$}
            \State $\mathbf{z}^* =\argmax_{\mathbf{z} \in \mathcal{P}_r} \phi(\mathbf{z}, A_r, \mathcal{P}_r)$
            \State $\mathcal{L}_r = \mathcal{L}_r \cup \{{(\mathbf{z}^*, \text{label}(\mathbf{x}^*))}\}$
            \State $\mathcal{P}_r =(\mathcal{P}_r \setminus \{\mathbf{z}^*\}) \cup \mathcal{N}(\mathbf{z}^*, k)$
        \EndFor
    \Until{$|\mathcal{L}_r| = T$}
\end{algorithmic}
\caption{\textsc{\flood approach}}
\label{algo:sseals}
\end{algorithm}
\end{minipage}

\section{Similarity search for Efficient Active Learning and Search (\flood)}\label{sec:sseals}

This section describes \flood and how it improves computational efficiency and impacts sample complexity.
As shown in Algorithm~\ref{algo:sseals}, \flood makes two modifications to accelerate the inner loop of Algorithm~\ref{algo:baseline}:

\begin{enumerate}
    \item The candidate pool $\mathcal{P}_r$ is restricted to the nearest neighbors of the labeled examples.
    \item After every example is selected, we find its $k$ nearest neighbors and update $\mathcal{P}_r$.
\end{enumerate}

Both modifications can be done transparently for many selection strategies, making \flood applicable to a wide range of methods, even beyond the ones considered here.

By restricting the candidate pool to the labeled examples' nearest neighbors, \flood applies the selection strategy to at most $k|L_r|$ examples.
Finding the $k$ nearest neighbors for each labeled example adds overhead, but it can be calculated efficiently with sublinear retrieval times~\citep{charikar2002similarity} and sub-second latency on billion-scale datasets~\citep{JDH17} for approximate approaches.

\textbf{Computational savings.}
Each selection round scales with the size of the labeled dataset and sublinearly with the size of the unlabeled data.
Excluding the retrieval times for the $k$ nearest neighbors, the computational savings from \flood are directly proportional to the pool size reduction for $\phi_{\text{MaxEnt}}$ and $\phi_{\text{MLP}}$, which is lower bound by $|U|/k|L_r|$.
For $\phi_{\text{ID}}$, the average similarity score for each example only needs to be computed once when the example is first selected.
This caching means the first round scales quadratically with $|U|$ and subsequent rounds scale linearly for the baseline approach.
With \flood, each selection round scales according to $O((1 + bk)|\mathcal{P}_r|)$ because the similarity scores are calculated as examples are selected rather than all at once.
The resulting computational savings of \flood varies with the labeling budget $T$ as the upfront cost of the baseline amortizes.
Nevertheless, for large-scale datasets with millions or billions of examples, performing that first quadratic round for the baseline is prohibitively expensive.

\textbf{Index construction.}
Generating the embeddings and indexing the data can be expensive and slow.
However, this cost amortizes over many selection rounds, concepts, or other applications.
Similarity search is a critical workload for information retrieval and powers many applications, including recommendation, with deep learning embeddings increasingly being used~\citep{babenko2014neural, babenko2016efficient, JDH17}.
As a result, the embeddings and index can be generated once using a generic model trained in a weak-supervision or self-supervision fashion and reused, making our approach just one of many applications using the index.
Alternatively, if the data has already been passed through a predictive system (for example, to tag or classify uploaded images), the embedding could be captured to avoid additional costs.

\textbf{Sample complexity.}
To shed light on why \flood works, we analyzed an idealized setting where classes are linearly separable and examples are already embedded ($\mathbf{x} = G_z(\mathbf{x})$).
Let $\X \subset \R^d$ be some convex set and $\mathbf{w}_* \in \R^d$.
An example $\mathbf{x} \in \X$ has a label $y = 1$ if $\mathbf{x}^\t \mathbf{w}_* \geq 0$ and a label $y = -1$ otherwise.
We assume that the $k$ nearest neighbor graph $\G = (\X,E)$ satisfies the property that for each $\mathbf{x},\mathbf{x}^\prime \in \X$, if $\norm{\mathbf{x}-\mathbf{x}^\prime}_2 \leq \delta$, then $(\mathbf{x},\mathbf{x}^\prime) \in E$, so any point in a ball around an example $\mathbf{x}$ is a neighbor of $\mathbf{x}$.
We also assume that the algorithm is given $n_0$ labeled seeds points $\S =\{\mathbf{x}_1,\ldots, \mathbf{x}_{n_0}\} \subset \X$ where $n_0 \geq d -1$.
To prove a result, we consider a slightly modified version of SEALS that performs $d-1$ parallel nearest neighbor searches, each one initiated with one of the seed points $\mathbf{x}_i$ with $i \in \{1,\ldots,d-1\}$, (see Section~\ref{supp:proof} in the Appendix for a formal description and a proof).
Note, this procedure still aligns with the batch queries in SEALS.

\begin{theorem}\label{thm:main_thm}
Let $\epsilon > 0$ and let $\gamma_i$ denote the distance from the seed $\mathbf{x}_i$ to the convex hull of oppositely labeled seed points. There exists a constant $\sigma>0$ that quantifies the diversity of the seeds (defined below) such that after SEALS makes $O( \max_{i \in \{1,\ldots, d-1\}} d(\frac{\gamma_i}{\delta} + \log(\frac{ d \delta }{\epsilon \min(\sigma,1) })))$ queries, its estimate $\widehat{\mathbf{w}} \in \R^d$ satisfies $\norm{\widehat{\mathbf{w}}-\mathbf{w}_*}_2 \leq \epsilon.$ 
\end{theorem}

The sample complexity bound compares favorably to known optimal sample complexities in this setting \cite{balcan2013active}: $O(d/\epsilon)$ and $O(d\log(1/\epsilon))$ for passive and active learning, respectively. In particular, the SEALS bound has the optimal logarithmic dependence on $\epsilon$. 

The parameter $\gamma_i$ is an upper bound on the distance of $\mathbf{x}_i$ to the true decision boundary.  Let $B_i$ denote the ball of radius $\gamma_i+2\delta+\epsilon$ centered at $\mathbf{x}_i$, where $\epsilon > 0$ is fixed.  The true decision boundary must intersect $B_i$.  Let $\Z_i \subset B_i$ denote the set of points in $B_i$ that are within $\epsilon$ of the boundary.
The constant $\sigma$ is a measure of the diversity of the seed examples, defined as: 
\begin{align*}
    \sigma & =   \min_{\mathbf{z}_i \in \Z_i : i\in \{1,\ldots, d-1\}} \sigma_{d-1}([\mathbf{z}_1  \ \cdots, \ \mathbf{z}_{d-1} ]) \nonumber 
\end{align*} 
where $\sigma_{d-1}(\cdot)$ is the $(d-1)$th singular value of the matrix.
If the $\Z_i$ sets are well separated and if the centers form a well-conditioned basis for a ($d-1$)-dimensional subspace in $\R^d$, then $\sigma$ is a reasonable constant. 

Intuitively, the algorithm has two phases: a slow phase and a fast phase.
During the slow phase, the algorithm queries points that slowly approach the true decision boundary at a rate $\delta$.
After at most $O(\max_i d\frac{\gamma_i}{\delta})$ queries, the algorithm finds $d-1$ points that are within $\delta$ of the true decision boundary and enters the fast phase.
Since the algorithm has already found points that are close to the decision boundary, the constraints of the nearest neighbor graph essentially do not encumber the algorithm, enabling it to home in on the true decision boundary at an exponential rate of $O(d\log(\frac{ d \delta }{\epsilon \sigma })))$.

\section{Experiments}\label{sec:results}

We applied \flood to three selection strategies (MaxEnt, MLP, and ID) and performed active learning and search on three separate datasets: ImageNet~\citep{russakovsky2015imagenet}, OpenImages~\citep{kuznetsova2020openimages}, and a de-identified and aggregated dataset of 10 billion publicly shared images (Table~\ref{table:datasets}).
Section~\ref{sec:experimental_setup} details the experimental setup used for both the baselines that run over all of the data (*-All) and our proposed method that restricts the candidate pool (*-\flood).
Sections~\ref{results:imagenet},~\ref{results:openimages}, and~\ref{results:10bimages} provide dataset-specific details and present the active learning and search results for ImageNet, OpenImages, and the proprietary dataset, respectively.
We also evaluated \flood with self-supervised embeddings using SimCLR~\citep{chen2020simple} for ImageNet and using SentenceBERT~\citep{reimers2019sentence} for Goodreads spoiler detection in the Appendix.
\begin{table*}[t]
\centering
\caption{Summary of datasets}
\resizebox{\textwidth}{!}{
\begin{tabular}{ccccc}
\toprule
     & \makecell{\bfseries Number of \\ \bfseries Concepts ($|R|$)} & \makecell{\bfseries Embedding \\ \bfseries Model ($G_z$)} & \makecell{\bfseries Number of \\ \bfseries Examples ($|U|$)} & \makecell{\bfseries Percentage \\ \bfseries Positive} \\
\hline
ImageNet~\citep{russakovsky2015imagenet} & 450  & \makecell{ResNet-50~\citep{he2016deep} \\ (500 classes)} & 639,906 &  0.114-0.203\% \\
\hline
OpenImages~\citep{kuznetsova2020openimages} & 153  & \makecell{ResNet-50~\citep{he2016deep} \\ (1000 classes)} & 6,816,296 &  0.002-0.088\% \\
\hline
\makecell{10 billion (10B) images (proprietary)} & 8 & \makecell{ResNet-50~\citep{he2016deep} \\ (1000 classes)} & 10,094,719,767 &  - \\
\bottomrule
\end{tabular}
}
\label{table:datasets}
\end{table*}

\subsection{Experimental Setup}\label{sec:experimental_setup}

We followed the same general procedure for both active learning and search across all datasets and selection strategies.
Each experiment started with 5 positive examples because finding positive examples for rare concepts is challenging a priori.
Negative examples were randomly selected at a ratio of 19 negative examples to every positive example to form a seed set $L_r^0$ with 5 positives and 95 negatives.
The slightly higher number of negatives in the initial seed set improved average precision on the validation set across the datasets.
The batch size $b$ for each selection round was the same as the size of the initial seed set (i.e., 100 examples), and the max labeling budget $T$ was 2,000 examples.

As the binary classifier for each concept $A_r$, we used logistic regression trained on the embedded examples.
For active learning, we calculated average precision on the test data for each concept after each selection round.
For active search, we count the number of positive examples labeled so far.
We take the mean average precision (mAP) and number of positives across concepts, run each experiment 5 times, and report the mean (dotted line) and standard deviation (shaded area around the line).

For similarity search, we used locality-sensitive hashing (LSH)~\citep{charikar2002similarity} implemented in Faiss~\citep{JDH17} with Euclidean distance for all datasets aside from the 10 billion images dataset.
This simplified our implementation, so the index could be created quickly and independently, allowing experiments to run in parallel trivially.
However, retrieval times for this approach were not as fast as \citet{JDH17} and made up a larger part of the overall active learning loop.
In practice, the search index can be heavily optimized and tuned for the specific data distribution, leading to computational savings closer to the improvements described in Section~\ref{sec:sseals} and differences in the ``Selection'' portion of the runtimes in Table~\ref{table:wallclock}.
$k$ was 100 for ImageNet and OpenImages unless specified otherwise, while the experiments on the 10 billion images dataset used a $k$ of 1,000 or 10,000 to compensate for the size.

We split the data, selected concepts, and created embeddings as detailed below and summarized in Table~\ref{table:datasets}.
In the Appendix, we varied the embedding models (\ref{supp:varying_gz}), the value of $k$ for \flood (\ref{supp:varying_k}), and the number of initial positives and negatives (\ref{supp:varying_positives}) to test how robust \flood was to our choices above.
Across all values, \flood performed similarly to the results presented here.

\begin{table*}[]
\centering
\caption{Wall clock runtimes for varying selection strategies on ImageNet and OpenImages. The last 3 columns break the total time down into 1) the time to apply the selection strategy to the candidate pool, 2) the time to find the $k$ nearest neighbors ($k$-NN) for the newly labeled examples, and 3) the time to train logistic regression on the currently labeled examples. Despite using a simple LSH search index, \flood substantially improved runtimes across datasets and strategies.}
\resizebox{\textwidth}{!}{
\begin{tabular}{cclccccccc}
\toprule
\multicolumn{6}{c}{} & & \multicolumn{3}{c}{\bfseries Time Breakdown (seconds)} \\
\bfseries Dataset &  \bfseries Budget $T$ &    \bfseries Strategy $\phi$ &      \bfseries mAP/AUC &  \makecell{\bfseries Recall \\ \bfseries (\%)} & \makecell{\bfseries Pool Size \\ \bfseries (\%)} & \makecell{\bfseries Total Time \\ \bfseries (seconds)} & Selection & $k$-NN  & Training \\
\hline

 ImageNet &    2,000 &    MaxEnt-All &  $0.695$ &  $57.2$ &  $100.0$ &      45.23 &         44.65 &        - &       0.59 \\
 &     &  MaxEnt-SEALS &  $0.695$ &  $56.9$ &    $6.6$ &      12.49 &          1.73 &    10.27 &       0.50 \\
 &     &       MLP-All &  $0.693$ &  $74.5$ &  $100.0$ &      43.32 &         42.75 &        - &       0.57 \\
 &     &     MLP-SEALS &  $0.692$ &  $74.2$ &    $6.0$ &      12.03 &          1.48 &     9.94 &       0.63 \\
 &     &        ID-All &  $0.688$ &  $50.8$ &  $100.0$ &    4654.59 &       4653.55 &        - &       1.05 \\
 &     &      ID-SEALS &  $0.688$ &  $50.9$ &    $6.9$ &     104.57 &         94.22 &     9.76 &       0.60 \\
\cline{2-10}
 &    1,000 &        ID-All &  $0.646$ &  $26.3$ &  $100.0$ &    4620.04 &       4619.78 &        - &       0.28 \\
 &     &      ID-SEALS &  $0.654$ &  $27.8$ &    $4.7$ &      36.66 &         31.95 &     4.56 &       0.17 \\
\cline{2-10}
 &     500 &        ID-All &  $0.586$ &  $12.5$ &  $100.0$ &    4602.64 &       4602.57 &        - &       0.09 \\
 &      &      ID-SEALS &  $0.601$ &  $13.5$ &    $3.2$ &       9.75 &          7.75 &     1.95 &       0.05 \\
\cline{2-10}
 &     200 &        ID-All &  $0.506$ &   $4.7$ &  $100.0$ &    4588.76 &       4588.73 &        - &       0.04 \\
 &      &      ID-SEALS &  $0.511$ &   $4.8$ &    $2.0$ &       1.53 &          1.03 &     0.49 &       0.02 \\
\hline
\hline
 OpenImages &    2,000 &    MaxEnt-All &  $0.399$ &  $35.0$ &  $100.0$ &     295.20 &        294.78 &        - &       0.42 \\
 &     &  MaxEnt-SEALS &  $0.386$ &  $35.1$ &    $0.8$ &      80.61 &          1.56 &    78.63 &       0.43 \\
 &     &       MLP-All &  $0.398$ &  $35.1$ &  $100.0$ &     285.27 &        284.88 &        - &       0.40 \\
 &     &     MLP-SEALS &  $0.386$ &  $35.1$ &    $0.8$ &      82.18 &          1.48 &    80.27 &       0.44 \\
 &     &      ID-All &  - &  - &    100.0 &       $>$24 hours &          $>$24 hours &     - &        - \\
 &     &      ID-SEALS &  $0.359$ &  $29.3$ &    $0.9$ &     129.79 &         48.98 &    80.40 &       0.41 \\
\bottomrule
\end{tabular}
}
\label{table:wallclock}
\end{table*}

\subsection{ImageNet}\label{results:imagenet}
ImageNet~\citep{russakovsky2015imagenet} has 1.28 million training images spread over 1,000 classes.
To simulate rare concepts, we split the data in half, using 500 classes to train the feature extractor $G_z$ and treating the other 500 classes as unseen concepts.
For $G_z$, we used ResNet-50 but added a bottleneck layer before the final output to reduce the dimension of the embeddings to 256.
We kept all of the other hyperparameters the same as in~\citet{he2016deep}.
We extracted features from the bottleneck layer and applied $l^2$ normalization. In total, the 500 unseen concepts had 639,906 training examples that served as the unlabeled pool.
We used 50 concepts for validation, leaving the remaining 450 concepts for our final experiments.
The number of examples for each concept varied slightly, ranging from 0.114-0.203\% of $|U|$.
The 50,000 validation images were used as the test set.

\textbf{Active learning.}
With a labeling budget of 2,000 examples per concept (\textasciitilde 0.31\% of $|U|$), all baseline and \flood approaches ($k=100$) were within 0.011 mAP of the 0.699 mAP achieved with full supervision, as shown in Figure~\ref{fig:imagenet}.
In contrast, random sampling (Random-All) only achieved 0.436 mAP.
MLP-All, MaxEnt-All, and ID-All achieved mAPs of 0.693, 0.695, and 0.688, respectively, while the \flood equivalents were all within 0.001 mAP at 0.692, 0.695, and 0.688 respectively and considered less than 7\% of the unlabeled data.
The resulting selection runtime for MLP-\flood and MaxEnt-\flood dropped by over 25$\times$, leading to a 3.6$\times$ speed-up overall (Table~\ref{table:wallclock}).
The speed-up was even larger for ID-\flood, ranging from about 45$\times$ at 2,000 labels to 3000$\times$ at 200 labels. 
Even at a per-class level, the results were highly correlated with Pearson correlation coefficients of 0.9998 or more (Figure~\ref{fig:imagenet_per_class_ap} in the Appendix).
The reduced skew from the initial seed set only accounted for a small part of the improvement, as Random-\flood achieved an mAP of only 0.498.

\textbf{Active search.}
As expected, MLP-All and MLP-\flood significantly outperformed all other selection strategies for active search.
At 2,000 labeled examples per concept, both approaches recalled over 74\% of the positive examples for each concept at 74.5\% and 74.2\% recall, respectively.
MaxEnt-All and MaxEnt-\flood had a similar gap of 0.3\%, labeling 57.2\% and 56.9\% of positive examples, while ID-All and ID-\flood were even closer with a gap of only 0.1\% (50.8\% vs. 50.9\%).
Nearly all of the gains in recall are due to the selection strategies rather than the reduced skew in the initial seed, as Random-\flood increased the recall by less than 1.0\% over Random-All.

\begin{figure}[]
  \centering
  \begin{subfigure}{0.95\textwidth}
    \includegraphics[width=\columnwidth]{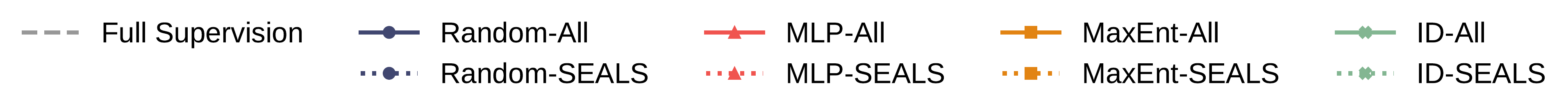}
  \end{subfigure}

  \begin{subfigure}{0.95\textwidth}
    \includegraphics[width=\columnwidth]{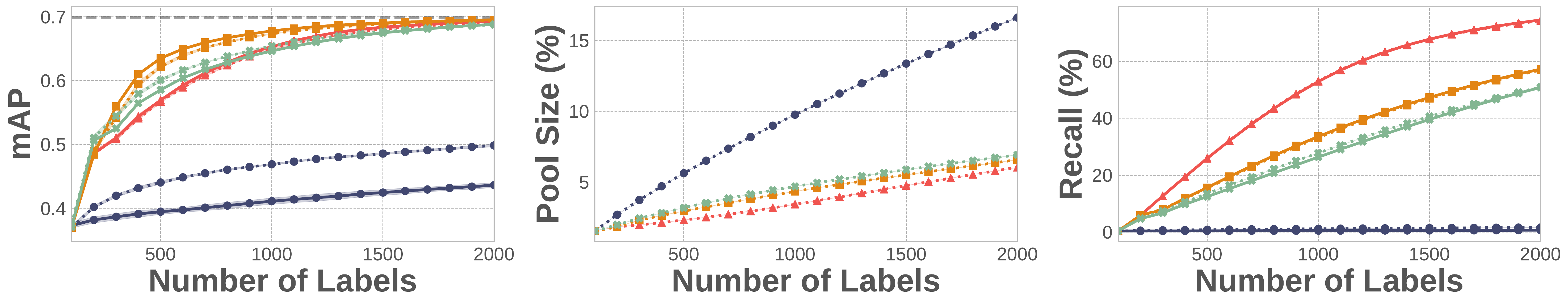}
    \caption{ImageNet ($|U|$=639,906)}
    \label{fig:imagenet}
  \end{subfigure}

  \begin{subfigure}{0.95\textwidth}
    \includegraphics[width=\columnwidth]{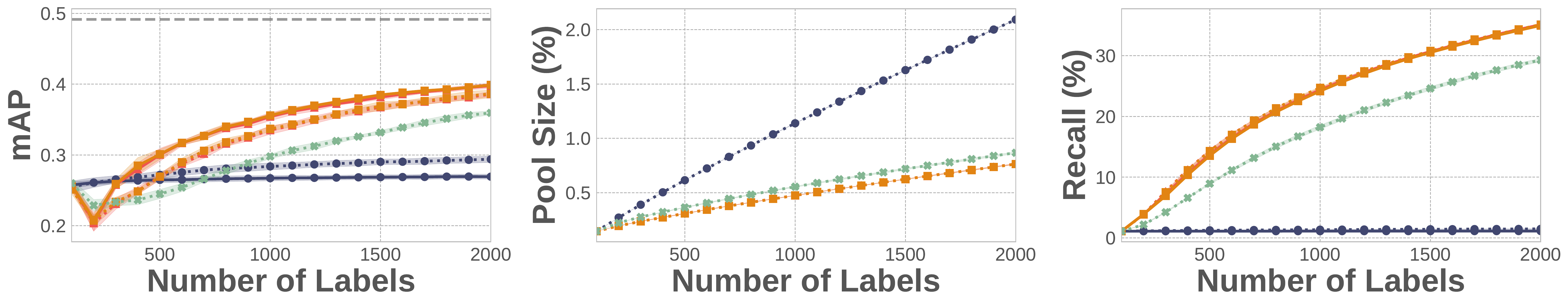}
    \caption{OpenImages ($|U|$=6,816,296)}
    \label{fig:openimages}
    \end{subfigure}

\caption{
Active learning and search on ImageNet (top) and OpenImages (bottom).
Across datasets and strategies, \flood with $k=100$ performed similarly to the baseline approach in terms of both the error the model achieved for active learning (left) and the recall of positive examples for active search (right), while only considering a fraction of the data $U$ (middle).
}
  \label{fig:active_learning_and_search}
\end{figure}

\subsection{OpenImages}\label{results:openimages}
OpenImages~\citep{kuznetsova2020openimages} has 7.34 million images from Flickr.
However, only 6.82 million images were still available in the training set at the time of writing.
The human-verified labels provide partial coverage for over 19,958 classes.
Like~\citet{kuznetsova2020openimages}, we treat examples that are not positively labeled for a given class as negative examples.
This label noise makes the task much more challenging, but all of the selection strategies adjust after a few rounds.
As a feature extractor, we took ResNet-50 pre-trained on all of ImageNet and used the $l^2$ normalized output from the bottleneck layer.
As rare concepts, we randomly selected 200 classes with between 100 to 6,817 positive training examples.
We reviewed the selected classes and removed 47 classes that overlapped with ImageNet.
The remaining 153 classes appeared in 0.002-0.088\% of the data.
We used the predefined test split for evaluation.

\textbf{Active learning.}
At 2,000 labels per concept (\textasciitilde 0.029\% of $|U|$), MaxEnt-All and MLP-All achieved 0.399 and 0.398 mAP, respectively, while MaxEnt-\flood and MLP-\flood both achieved 0.386 mAP and considered less than 1\% of the data (Figure~\ref{fig:openimages}).
This sped-up the selection time by over 180$\times$ and the total time by over 3$\times$, as shown in Table~\ref{table:wallclock}.
Increasing $k$ to 1,000 significantly narrowed this gap for MaxEnt-\flood and MLP-\flood, improving mAP to 0.395, as shown in the Appendix (Figure~\ref{fig:openimages_ks}).
Moreover, the reduced candidate pool from \flood made ID tractable on OpenImages, whereas ID-All ran for over 24 hours in wall-clock time without completing a single round (Table~\ref{table:wallclock}).

\textbf{Active search.}
The gap between the baselines and \flood was even closer on OpenImages than on ImageNet despite considering a much smaller fraction of the overall unlabeled pool.
MLP-All, MLP-\flood, MaxEnt-\flood, and MaxEnt-All were all within 0.1\% with \textasciitilde35\% recall at 2,000 labels per concept.
ID-\flood had a recall of 29.3\% but scaled nearly as well as the linear approaches.

\subsection{10 Billion Images}\label{results:10bimages}
10 billion (10B) publicly shared images from a large internet company were used to test \flood' scalability.
We used the same pre-trained ResNet-50 model as the OpenImages experiments.
We also selected eight additional classes from OpenImages as rare concepts: `rat,' `sushi,' `bowling,' `beach,' `hawk,' `cupcake,' and `crowd.'
This allowed us to use the test split from OpenImages for evaluation.
Unlike the other datasets, we hired annotators to label images as they were selected and used a proprietary index to achieve low latency retrieval times to capture a real-world setting.

\begin{figure}[]
  \centering
  \begin{subfigure}{0.95\textwidth}
    \includegraphics[width=\columnwidth]{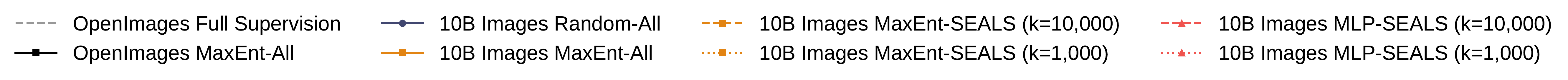}
  \end{subfigure}

  \begin{subfigure}{0.95\textwidth}
    \includegraphics[width=\columnwidth]{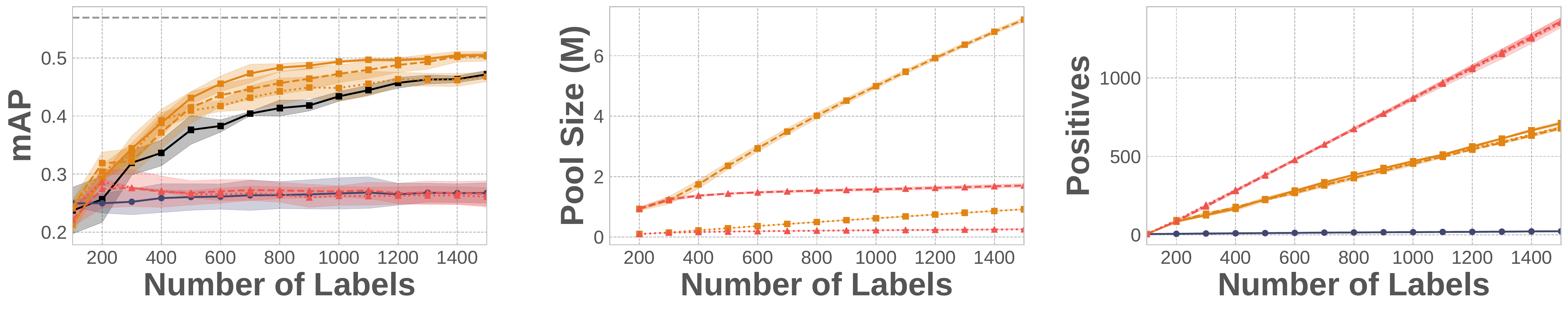}
    \label{fig:10bimages}
  \end{subfigure}

\caption{
Active learning and search on a de-identified and aggregated dataset of 10 billion publicly shared images.
Across strategies, \flood with $k=10,000$ performed similarly to the baseline approach in terms of both the error the model achieved for active learning (left) and the recall of positive examples for active search (right), while only considering a fraction of the data $U$ (middle).
}
  \label{fig:active_learning_and_search_10b}
\end{figure}

\textbf{Active learning.}
Despite the limited pool size, \flood performed similarly to the baseline approaches that scanned all 10 billion images.
At a budget of 1,500 labels, MaxEnt-\flood ($k$=10K) achieved a similar mAP to the baseline (0.504 vs. 0.508 mAP), while considering only about 0.1\% of the data (Figure~\ref{fig:active_learning_and_search_10b}).
This reduction allowed MaxEnt-\flood to finish selection rounds in just seconds on a single 24-core machine, while MaxEnt-All took several minutes on a cluster with tens of thousands of cores.
Unlike the ImageNet and OpenImages experiments, MLP-\flood performed poorly at this scale because there are likely many redundant or near-duplicate examples of little value.

\textbf{Active search.}
SEALS performed well despite considering less than 0.1\% of the data and collected two orders of magnitude more positive examples than random sampling.

\section{Discussion and Conclusion}~\label{sec:discussion}

\begin{figure}[t]
  \centering
  \begin{subfigure}{.95\textwidth}
      \begin{subfigure}{.49\columnwidth}
          \includegraphics[width=0.95\columnwidth]{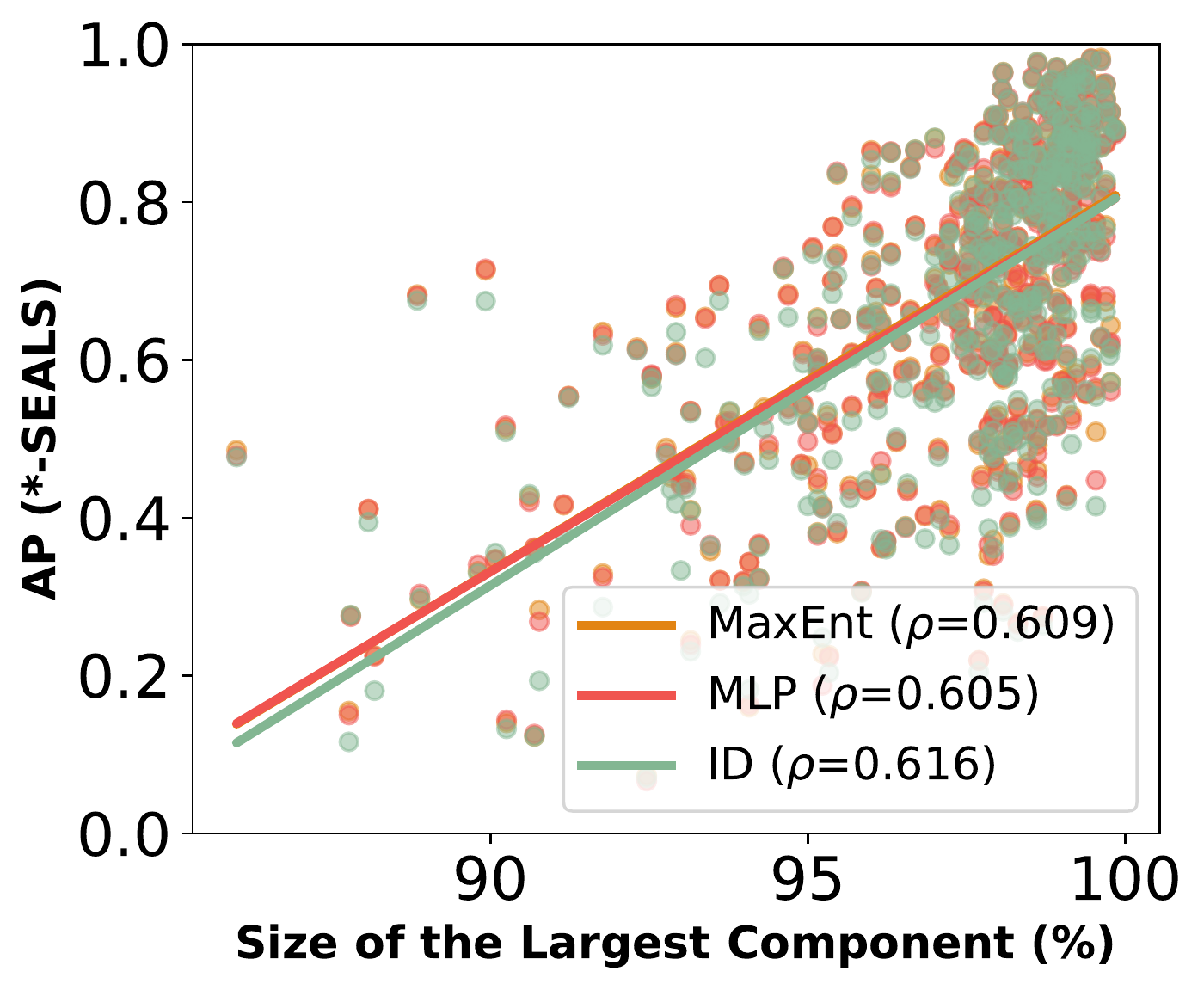}
      \end{subfigure}
      \hfill
      \begin{subfigure}{.49\columnwidth}
          \includegraphics[width=0.95\columnwidth]{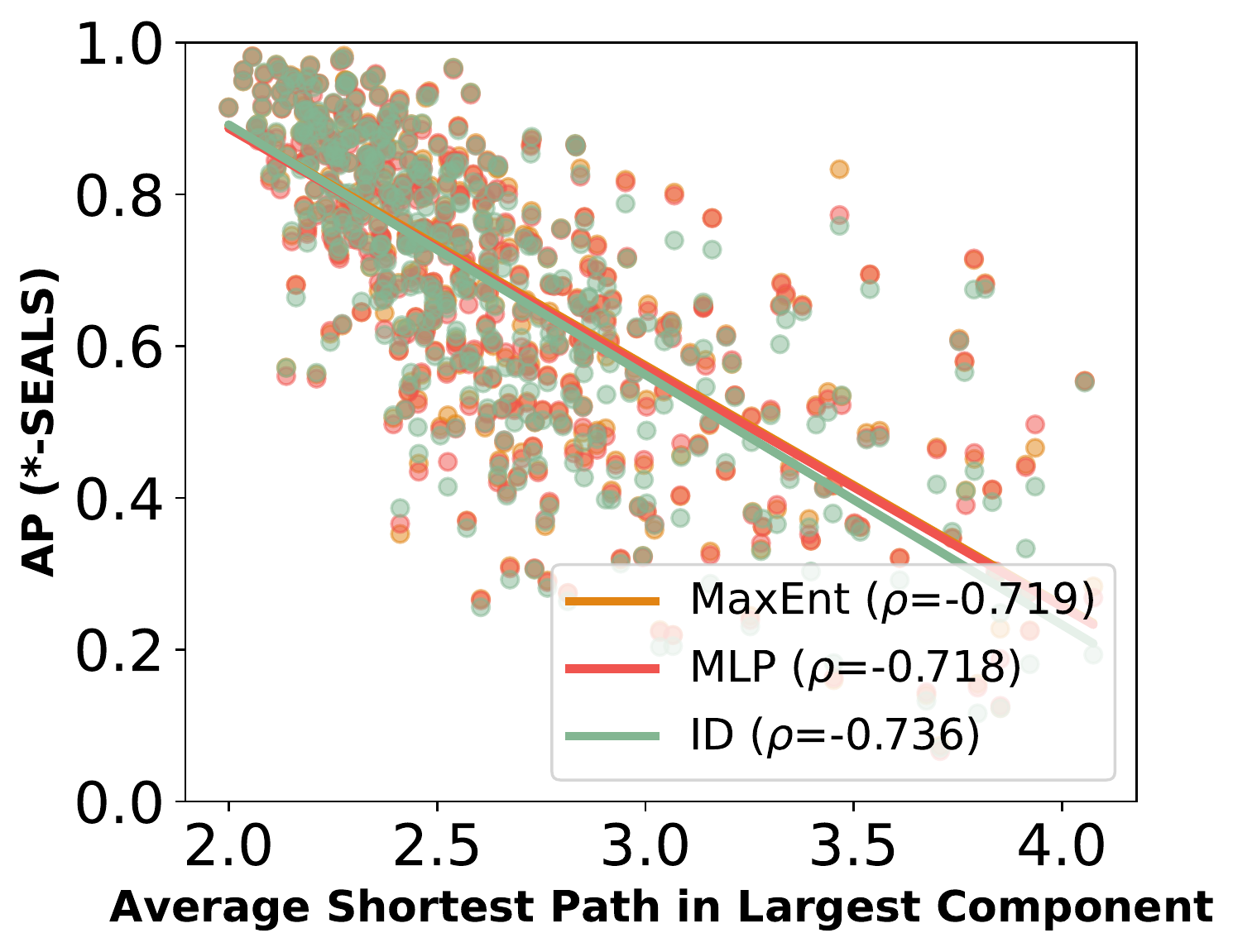}
      \end{subfigure}
      \caption{ImageNet}
      \label{fig:imagenet_per_class_latent_structure}
  \end{subfigure}\vspace{-1mm}
  \begin{subfigure}{0.95\textwidth}
      \begin{subfigure}{.49\columnwidth}
          \includegraphics[width=0.95\columnwidth]{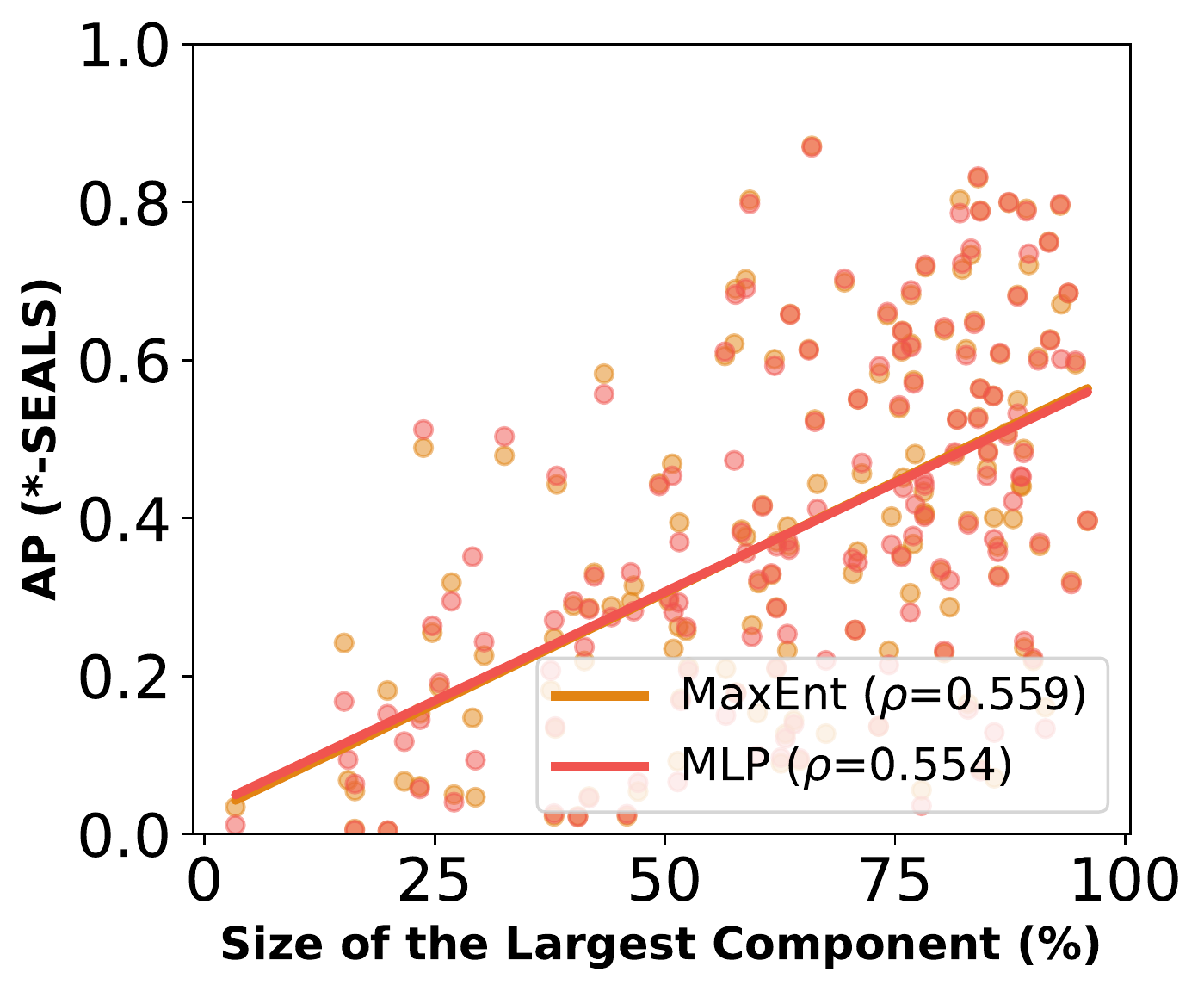}
      \end{subfigure}
      \hfill
      \begin{subfigure}{.49\columnwidth}
          \includegraphics[width=0.95\columnwidth]{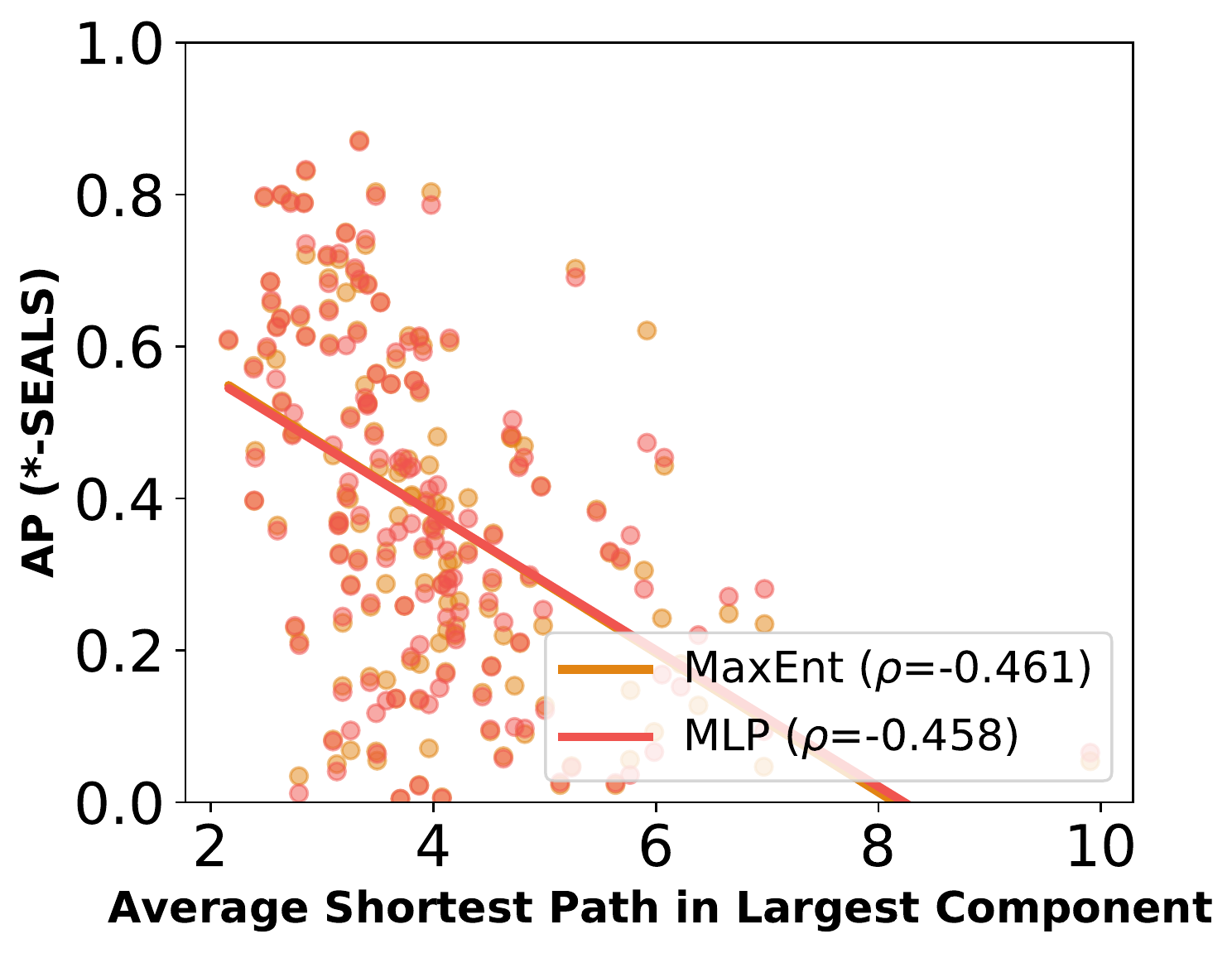}
      \end{subfigure}
      \caption{OpenImages}
      \label{fig:openimages_per_class_latent_structure}
  \end{subfigure}\vspace{-1mm}
\caption{
    Correlations between AP and measurements of the latent structure of unseen concepts. \flood ($k=100$) achieved higher APs for classes that formed larger connected components (left) and had shorter paths between examples (right) in ImageNet (top) and OpenImages (bottom).
}\vspace{-2mm}
  \label{fig:per_class_latent_structure}
\end{figure}

\textbf{Latent structure of unseen concepts.}
To better understand when \flood works, we analyzed the relationship between average precision and the structure of the nearest neighbor graph across concepts.
Overall, \flood performed better for concepts that formed larger connected components and had shorter paths between examples, as shown in Figure~\ref{fig:per_class_latent_structure}.
For most concepts in ImageNet, the largest connected component contained the majority of examples, and the paths between examples were very short.
These tight clusters explain why so few examples were needed to learn accurate binary concept classifiers, as shown in Section~\ref{sec:results}, and why \flood recovered \textasciitilde 74\% of positive examples on average while only labeling \textasciitilde 0.31\% of the data.
If we constructed the candidate pool by randomly selecting examples as in~\citep{ertekin2007learning}, mAP and recall would drop for all strategies (Section~\ref{appendix:random_pool} in the Appendix).
The concepts were so rare that the randomly chosen examples were not close to the decision boundary.
For OpenImages, rare concepts were more fragmented, but each component was fairly tight, leading to short paths between examples.
On a per-class level, concepts like `monster truck' and `blackberry' performed much better than generic concepts like `electric blue' and `meal' that were more scattered (Sections~\ref{appendix:openimages_per_class_active_learning} and~\ref{appendix:openimages_per_class_active_search} in the Appendix).
This fragmentation partly explains the gap between \flood and the baselines in Section~\ref{sec:results}, and why increasing $k$ closed it.
However, even for small values of $k$, \flood led to significant gains over random sampling, as shown in Figures~\ref{fig:imagenet_ks} and~\ref{fig:openimages_ks} in the Appendix.


\textbf{Conclusion.}
We introduced \flood as a simple approach to make the selection rounds of active learning and search methods scale sublinearly with the unlabeled data.
Instead of scanning over all of the data, \flood restricted the candidate pool to the nearest neighbors of the labeled set.
Despite this limited pool, we found that \flood achieved similar average precision and recall while improving computational efficiency by up to three orders of magnitude, enabling \textit{web-scale active learning}.



\bibliographystyle{plainnat}
\bibliography{references}

\clearpage

\section{Appendix}

\subsection{Broader Impact}\label{sec:broader_impact}
Our work attacks both the labeling and computational costs of machine learning and will hopefully make machine learning much more affordable.
Instead of being limited to a small number of large teams and organizations with the budget to label data and the computational resources to train on it, \flood dramatically reduces the barrier to machine learning, enabling small teams or individuals to build accurate classifiers.
\flood does, however, introduce another system component, a similarity search index, which adds some additional engineering complexity to build, tune, and maintain.
Fortunately, several highly optimized implementations like Annoy~\footnote{\url{https://github.com/spotify/annoy}} and Faiss~\footnote{\url{https://github.com/facebookresearch/faiss}} work reasonably well out of the box.
There is a risk that poor embeddings will lead to disjointed components for a given concept.
This failure mode may prevent \flood from reaching all fragments of a concept or take a longer time to do so, as mentioned in Section~\ref{sec:discussion}.
However, active learning and search methods often involve humans in the loop, which could detect biases and correct them by adding more examples.

\subsection{Proof of \flood under idealized conditions}\label{supp:proof}

To begin, we introduce the mathematical setting. Assume the input space is a convex set $\X \subset \R^d$ and that the optimal linear classifier $w_* \in \R^d$ satisfies $\norm{w_*}_2 =1$. We assume the \emph{homogenous setting} where $\H_* = \{x \in \R^d : w_*^\t x = 0\}$ is the hyperplane defining the optimal classification. For each $x \in \X$, let $y_x$ denote its associated label and assume that $y_x = 1$ if $w_*^\t x \geq 0$ and $y_x = -1$ if $w_*^\t x < 0$. Define $\X^+ = \{x \in \X : w_*^\t x \geq 0\}$ and $\X^- = \{x \in \X : w_*^\t x < 0\}$. Let $\delta > 0$ and let $\G = (\X,E)$ be a nearest-neighbor graph where we assume that for each $x,x^\prime \in \X$, if $\norm{x-x^\prime}_2 \leq \delta$, then $(x,x^\prime) \in E$.

Our analysis makes two key assumptions.
First, we assume that the classes are linearly separable. Since the SEALS algorithm uses feature embedding, often extracted from a deep neural network, the classes are likely to be almost linearly separable in many applications. 
Second, we assume a \emph{membership query model} where the algorithm can query any point belonging to the input space $\X$. Since SEALS is typically applied to datasets with billions of examples, this query model is a reasonable approximation of practice. It should be possible to extend our analysis to the pool-based active learning setting. Suppose there are enough points so that for every possible direction there is at least one point in every $\delta$-nearest neighborhood with a component of size $c\delta$, for some $0<c<1$, in that direction. Then we believe that a bound like that in Theorem~\ref{thm:main_thm} should hold by replacing $\delta$ with $c\delta$. Relaxing these two assumptions is left to future theory work.



The main goal of our theory is to quantify the effect of the nearest neighbor restriction. 
Our analysis considers the modified SEALS procedure described in Algorithm \ref{alg:margin_multiple_seeds}. It differs from original SEALS in some minor ways that make it more amenable to analysis, but crucially it too is based on nearest neighbor graph search.  First, we introduce some notation. We let $\data_r \subset \X \times \{-1,1\}$ consist of the examples and their labels queried until round $r$. We let  $\data_r^1 = \{x \in \data_1 : (x,1) \in \data_r \}  $ denote the positive examples queried until round $r$ and $\data_r^{-1} = \{x \in \data_1 : (x,-1) \in \data_r \}$ the negative examples queried until round $r$. Let $A,B \subset \R^d$. The subroutine $\text{MaxMarginSeparator}(A, B)$ finds a maximum margin separator of $A$ and $B$ and returns the hyperplane $H$ and margin~$\gamma$: $(H,\gamma) \longleftarrow \text{MaxMarginSeparator}(A, B)$.  This is a support vector machine.

Now, we present Algorithm \ref{alg:margin_multiple_seeds}. We suppose that the algorithm has an initial set of seed points $ \{x_{1,0},\ldots, x_{d-1,0}\}$, with which it initiates $d-1$ nearest neighbor searches. We note that the algorithm could perform $n$ nearest neighbor searches and the analysis would still go through provided that $d-1 \leq n = O(d)$, and that the initial set of labeled points $\data_r$ may be larger than $d-1$. At each round $r$, Algorithm \ref{alg:margin_multiple_seeds} queries one unlabeled neighbor from each set $C_{i,r}$, $i=1,\dots,d-1$ (at $r=0$ these sets are the seed points themselves). The decision rule is as follows: first, for each $x \in C_{i,r}$, the algorithm computes a max-margin separating hyperplane $H_x$ separating $x$ from the examples with opposing labels. Second, the algorithm selects the example with the smallest margin $\bar{x}_{i,r}$ and queries a neighbor of $\bar{x}_{i,r}$ that is closest to $H_{\bar{x}_{i,r}}$. This is similar to using \emph{MaxEnt} uncertainty sampling in SEALS.
In Algorithm \ref{alg:margin_multiple_seeds}, $\argmin_{x^\prime : (\bar{x}_{i,r},x^\prime) \in E}\, \dist(x^\prime,H_{\bar{x}_{i,r},r})$ may contain several examples if $\dist(\bar{x}_{i,r},H_{\bar{x}_{i,r},r}) \leq  \delta$. In this case, we tiebreak by letting $\tilde{x}_{i,r}$ be the projection of $\bar{x}_{i,r}$ onto $H_{\bar{x}_{i,r},r}$.


\begin{algorithm}[H]
 \caption{Modified SEALS}
 \label{alg:margin_multiple_seeds}
 \begin{algorithmic}[1]
\State \textbf{Input: } seed labeled examples $\data_{1}  \subset \X \times \{-1,1\}$    
\State $r=1$  
\State Initialize the clusters $C_{i,r} = \{x_{i,0}\}$ for $i =1,\ldots, d-1$ 
\For{$r=1,2,\ldots$}
\For{$i=1,\ldots,d-1$}
\State $(H_{x,r},\gamma_{x,r}) =\text{MaxMarginSeparator}(\data_{r}^{-y_x}, x)$ for all $x \in C_{i,r} $
\State Let $\bar{x}_{i,r} \in \argmin_{x \in C_{i,r}} \gamma_{x,r}$ and $\widetilde{x}_{i,r} \in  \argmin_{x^\prime : (\bar{x}_{i,r},x^\prime) \in E} \dist(x^\prime,H_{\bar{x}_{i,r},r})$
\State Query $\widetilde{x}_{i,r}$ 
\State Update $\data_{r+1} = \data_r \cup \{(\widetilde{x}_{i,r}, y_{\widetilde{x}_{i,r}})\}$ and $C_{i,r+1} = C_{i,r} \cup \{\widetilde{x}_{i,r}\}$ 
\EndFor
\State Fit a homogenous max-margin separator with normal vector $\widehat{w}_{r+1}$ to $\data_{r+1}$\;
\EndFor
 \end{algorithmic}
\end{algorithm}


We actually restate the Theorem~\ref{thm:main_thm} a bit more formally in Theorem~\ref{thm:main_thm_supp}, below.

\begin{theorem}\label{thm:main_thm_supp}
Let $\epsilon > 0$. Let $x_{1,0},\ldots, x_{d-1,0}$ denote the seed points. Define $\gamma_i = \dist(x_{i,0}, {{\conv}}(\data^{-y_{x_{i,0}}}_1))$ for $i \in [d-1]$, where ${\conv}(\data^{-y_{x_{i,0}}}_1)$ is the convex hull of the points $\data^{-y_{x_{i,0}}}_1.$ Then, after Algorithm \ref{alg:margin_multiple_seeds} makes $ \max_{i \in [d-1]} d(\frac{\gamma_i}{\delta} + 2\log(\frac{2 d \delta }{\epsilon \min(\sigma,1) })))$ queries, $\norm{\widehat{w}_r-w_*} \leq \epsilon.$
\end{theorem}


The constant $\sigma$ is a measure of the diversity of the initial seed examples, which we now define. For $i \in [d-1]$, define the set
\begin{align*}
    \Z_i & = \{z \in \R^d : \norm{z - x_{i,0}} \leq \gamma_i + 2\delta + \epsilon \text{ and } \dist(z, \{x \in \R^d : x^\t w_* = 0\} ) \leq \epsilon \}.
\end{align*}
Define
\begin{align*}
    \sigma & :=  \min_{z_i \in \Z_{i} : \forall i \in [d-1] } \sigma_{d-1}([z_{1}  \ldots z_{d-1} ]).
\end{align*}
where $\sigma_{d-1}$ denotes the $(d-1)$th singular value of the matrix $[z_{1}  \ldots z_{d-1} ]$.

Here we give a simple example where $\sigma = \Omega(1)$ so as to provide intuition, although there are a wide variety of such cases.

\begin{example}\label{ex:sigma} 
Let $\epsilon \in (0,1)$. Let $\X = \R^d$, $w^* = e_1$, $\delta = 1/2$. Let $M \geq 6 \sqrt{d-1}$. Suppose the seed examples are $x_{i,0} = e_1 + M e_{i+1}$ for $i \in [d-1]$ and suppose the algorithm is given additional examples $v_{i}= -e_1 + M e_{i+1} \in \data_0^{-1}$ for $i \in [d-1]$. Then, $\sigma \geq 1$.
\end{example}

\begin{proof}[Proof of Example \ref{ex:sigma}]
Note that $\gamma_i = 2$ for all $i \in [d-1]$. Define the matrix
\begin{align*}
    Z & = \begin{pmatrix}
        z_1^\t  \\
        \vdots \\
        z_{d-1}^\t
    \end{pmatrix}
\end{align*}
such that $\norm{z_i - x_{i,0}} \leq \gamma_i + 2\delta + \epsilon, \, \forall i \in [d-1]  $. We may write $z_i = x_{i,0} + v_i$ where $\norm{v_i} \leq \gamma_i + 2\delta + \epsilon$. Courant-Fisher's min-max theorem implies that $s_{d-1}(Z) = \max_{\dim E = d-1} \min_{u \in \spa(E) : \norm{u} = 1} \norm{Z u}$ where $E \subset \R^{d}$. Therefore, taking $E = \{e_2,\ldots, e_d\}$, it suffices to lower bound $\norm{Z u}$ for any $u \in \spa(e_2,\ldots, e_d)$ with $\norm{u} = 1$. Since $\norm{u} = 1$, there exists $j \in \{2,\ldots, d\}$ such that $|u_j| \geq \frac{1}{\sqrt{d-1}}$. Suppose wlog that $u_j \geq \frac{1}{\sqrt{d-1}}$ (the other case is similar). Then, by Cauchy-Schwarz, 
\begin{align*}
    \norm{Z u} & \geq \max_{i \in [d-1]}| z_{i}^\t u| \\
    & \geq (e_1 + M e_j + v_{j-1} )^\t u \\
    & \geq M \frac{1}{\sqrt{d}} - 1- \gamma_i - 2\delta - \epsilon \\
    & \geq  \frac{M}{\sqrt{d}} -5 \\
    & \geq 1. 
\end{align*}

\end{proof}


Now, we turn to the proof of Theorem \ref{thm:main_thm_supp}. In the interest of using more compact notation, we define for all $i \in [d-1]$ and $r \in \N$
\begin{align*}
    \rho_{i,r} :=  \dist(\bar{x}_{i,r}, {\conv}(\data_r^{-y_{\bar{x}_{i,r}}})).
\end{align*}
In words, $\rho_{i,r}$ is the distance of the example queried in nearest neighbor search $i$ and at round $r$, $\bar{x}_{i,r}$, to the convex hull of $\data_r^{-y_{\bar{x}_{i,r}}}$, the examples with opposite labels from $\bar{x}_{i,r}$. 

\begin{proof}[Proof of Theorem \ref{thm:main_thm_supp}]
\textbf{Step 1: Bounding the number of queries to find points near the decision boundary.} Define $\bar{\epsilon} = \frac{\min(\sigma,1) \epsilon^2}{2 \sqrt{d}}$, where $\sigma$ is defined as in the Theorem statement. We assume $\sigma > 0$ for the remainder of the proof. Let $C_{i,r} = \{x_{i,0}, x_{i,1},\ldots, x_{i,r}\}$ where $x_{i,l}$ is the queried example in the $l$th round. We show that for all $r \geq \max_{i \in [d-1]} \frac{\gamma_i}{\delta} + \log(\frac{2 \delta}{\bar{\epsilon}  })$ , for all $i \in [d-1]$,  $\rho_{i,r} \leq \bar{\epsilon}$. 

Fix $i \in [d-1]$. Define 
\begin{align*}
    E_r = \{ \text{at round } r, \rho_{i,r} \leq \bar{\epsilon} \}.
\end{align*}
We have that
\begin{align*}
    \sum_{r \in \N} \1 \{E_r^c\} & = \sum_{r \in \N}  \1 \{E_r^c  \cap \{\rho_{i,r}  \geq 2 \delta \} \} \\
    & \hspace{2cm} +\1 \{E_r^c  \cap \{\rho_{i,r}    < 2 \delta \} \}
\end{align*}
If $\rho_{i,r} \geq 2 \delta$, we have by Lemma \ref{lem:reduce_distances} that
\begin{align*}
    \rho_{i,r+1} \leq \rho_{i,r} -  \delta. 
\end{align*}
This implies that 
\begin{align*}
    \sum_{r \in \N} \1 \{E_r^c\  \cap \{\rho_{i,r}  \geq 2 \delta \}\} \leq \frac{\gamma_i}{\delta}.
\end{align*}
Now, suppose that $\rho_{i,r}  < 2 \delta$. Then, by Lemma \ref{lem:reduce_distances}, we have that 
\begin{align*}
    \rho_{i,r+1} \leq \frac{\rho_{i,r} }{2}.
\end{align*}

Unrolling this recurrence implies that 
\begin{align*}
    \sum_{r \in \N} \1 \{E_r^c\ \cap  \{\rho_{i,r}   < 2 \delta \} \}  \leq  \log(\frac{2 \delta}{\bar{\epsilon}  }).
\end{align*}
Putting it altogether, we have that
\begin{align*}
    \sum_{r \in \N} \1 \{E_r^c\} \leq \frac{\gamma_i}{\delta} + \log(\frac{2 \delta}{\bar{\epsilon}  }).
\end{align*}

This implies that after $\max_{i \in [d-1]}  \frac{\gamma_i}{\delta} + \log(\frac{2 \delta}{\bar{\epsilon}  })$ queries, we have that for all $i \in [d-1]$,  $\rho_{i,r} := \dist(\bar{x}_{i,r}, {\conv}(\data_r^{-y_{x_{i,r}}})) \leq \bar{\epsilon}$. 

\textbf{Step 2: Showing $\norm{x_{i,0} - x_{i,r}} \leq \gamma_i + 2 \delta$ for all $i \in [d-1]$}. Fix $i \in [d-1]$. Let $C_{i,r} = \{x_{i,0}, x_{i,1},\ldots, x_{i,r}\}$ where $x_{i,l}$ is the queried example in the $l$th round. 
Note there exists a path of length at most $r$ in the nearest neighbor graph on the nodes in $C_{i,r}$ between $x_{i,0}$ and $x_{i,r}$. In the worst case, this path consists of $x_{i,0}, x_{i,2},\ldots, x_{i,r}$ with $x_{i,s}$ being the child of $x_{i,s-1}$ and thus we suppose that this is the case wlog.

By the argument in step 1 of the proof, we have that after at most $\bar{k}=\frac{\gamma_i}{\delta}$ rounds, 
\begin{align*}
    \dist(x_{i,\bar{k}}, {\conv}(\data_{r}^{-y_{x_{i,\bar{k}}}})) \leq 2\delta. 
\end{align*}
For $s \leq \bar{k}$, we have that $\norm{x_{i,s} - x_{i,{s-1}}} \leq \delta$ by definition of the nearest neighbor graph. Now, consider $s \geq \bar{k}$. By Lemma \ref{lem:project}, we have that 
$\norm{x_{i,s+1} - x_{i,{s}}} \leq \frac{\norm{x_{i,s} - x_{i,{s-1}}}}{2}$ and $\norm{x_{i,\bar{k}+1} - x_{i,{\bar{k}}}} \leq \delta$. Therefore,
\begin{align*}
     \norm{x_{i,r} - x_{i,0}} & = \norm{\sum_{s=1}^{r-1} x_{i,s+1} - x_{i,{s}}} \\
     & \leq \sum_{s=1}^{r-1} \norm{ x_{i,s+1} - x_{i,{s}}} \\
     & \leq \bar{k} \max_{s \leq \bar{k}} \norm{ x_{i,s+1} - x_{i,{s}}} + \sum_{s = \bar{k}+1}^{r-1} \norm{ x_{i,s+1} - x_{i,{s}}} \\
     & \leq \frac{\gamma_i}{\delta} \delta + \delta \sum_{s = \bar{k}+1}^{r-1} \frac{1}{2^{-s}} \\
     & \leq \gamma_i + 2\delta.
\end{align*}

\textbf{Step 3: for all $i \in [d-1]$, there exists $z_i \in \Z_i$ such that $\widehat{w}_r^\t z_i = 0$.} We have that $\rho_{i,r} \leq \bar{\epsilon}$ for all $i \in [d-1]$ and $\norm{x_{i,0} - x_{i,r}} \leq \gamma_i + 2 \delta$ for all $i \in [d-1]$. Now, we show that there exists $z_i \in \Z_i$ such that $\widehat{w}_r^\t z_i = 0$. Fix $i \in [d-1]$. Suppose $\widehat{w}^\t_r \bar{x}_{i,r} > 0$ (the other case is similar). Then, $\rho_{i,r} \leq \bar{\epsilon}$ implies that there exists $\bar{x} \in {\conv}(\data_r^{-1})$ such that $\norm{\bar{x}_{i,r} - \bar{x}} \leq \bar{\epsilon}$. Then,  since $\widehat{w}_r$ separates $\data_r^{-1}$ and $\data_r^{1}$, $\widehat{w}^\t_r \bar{x} \leq 0$. Now, there exists $z_i \in {\conv}(\bar{x}_{i,r}, \bar{x})$ such that $z_i^\t \widehat{w}_r = 0$. By the triangle inequality, we have that $\norm{x_{i,0} - z_i} \leq \gamma_i + 2 \delta + \epsilon$, and $\dist(z_i, \{x : w_*^\t x = 0\}) \leq \bar{\epsilon} \leq \epsilon$, thus, $z_i \in \Z_i$ for all $i \in [d-1]$, completing this step. 

\textbf{Step 4: Pinning down $w_*$.} Since $\sigma > 0$, and each $z_i \in \Z_i$, we have that $z_1,\ldots, z_{d-1}$ are linearly independent. Then, we can write $w_* = \sum_{i=1}^{d-1} \beta_i z_i + \beta_d \widehat{w}_r$ for $\beta_1,\ldots, \beta_d \in \R$ where we used that $ \widehat{w}_r$ is orthogonal to $z_1,\ldots, z_{d-1}$ by construction of $z_1,\ldots, z_{d-1}$. Note that $\widehat{w}_r^\t w_* = \widehat{w}_r^\t (\sum_{i=1}^{d-1} \beta_i z_i + \beta_d \widehat{w}_r) = \beta_d$. Let $P w_*$ denote the projection of $w_*$ onto $\spa(z_1,\ldots,z_{d-1})$. Defining the matrix $Z = [z_1 \, z_2 \, \ldots \, z_{d-1} ]$ and $\beta = (\beta_1 \, \ldots, \beta_{d-1} )^\t$, we can write $Pw_* = Z \beta$. Let $Z = U \Sigma V^\t$ denote the SVD of $Z$ and $Z^\dagger = V \Sigma^{\dagger} U^\t$ the pseudoinverse. Note that $Z^\dagger Pw_* = \beta$. Then,
\begin{align}
    \norm{\beta} = \norm{Z^{\dagger} Pw_*} & \leq \max_{i=1,2,\ldots, d-1} \sigma_{i}(Z^\dagger) = \frac{1}{\sigma_{d-1}(Z )} \label{eq:beta_bound} .
\end{align}

We note that $\dist(z_i, \{x : w_*^\t x = 0\}) \leq \bar{\epsilon} $ implies that $|w_*^\t z_i| \leq \bar{\epsilon}$ for $i=1,2,\ldots, d-1$. Then, we have that
\begin{align}
    1 & = w_*^\t w_* \nonumber \\
    & = \sum_{i=1}^{d-1} \beta_i w_*^\t x_i + \beta_d w_*^\t \widehat{w}_r \nonumber \\
    & =  \sum_{i=1}^{d-1} \beta_i w_*^\t x_i +  (w_*^\t \widehat{w}_r)^2 \label{eq:beta_d} \\
    & \leq \norm{\beta}_1 \bar{\epsilon} +  (w_*^\t \widehat{w}_r)^2 \label{eq:holder}\\
    & \leq \sqrt{d}\norm{\beta}_2 \bar{\epsilon} +  (w_*^\t \widehat{w}_r)^2 \label{eq:norm_inequality} \\
    & \leq \frac{\sqrt{d} \bar{\epsilon}}{\sigma_{d-1}([z_1 \ldots z_{d-1}])} +  (w_*^\t \widehat{w}_r)^2 \label{eq:apply_beta_bound}\\
    & \leq \frac{\sqrt{d} \bar{\epsilon}}{\sigma} +  (w_*^\t \widehat{w}_r)^2 \label{eq:def_sigma}
\end{align}
where \eqref{eq:beta_d} follows by plugging in the previously derived $\beta_d = w_*^\t \widehat{w}_r$, \eqref{eq:holder} follows by Holder's inequality, \eqref{eq:norm_inequality} follows by $\norm{\cdot}_1 \leq \sqrt{d} \norm{\cdot}_2$, \eqref{eq:apply_beta_bound} follows by \eqref{eq:beta_bound}, and \eqref{eq:def_sigma} follows by the definition of $\sigma$. Rearranging, we have that
\begin{align}
    w_*^\t \widehat{w}_r &=|w_*^\t \widehat{w}_r| \label{eq:def_w}\\
    & \geq \sqrt{1 - \frac{\sqrt{d} \bar{\epsilon}}{\sigma}} \nonumber \\
    & \geq 1 - \frac{\sqrt{d} \bar{\epsilon}}{\sigma} \label{eq:less_1} \\
    & \geq 1-\frac{\epsilon^2}{2} \label{eq:e_bar_def}
\end{align}
where \eqref{eq:def_w} follows by the definition of $\widehat{w}_r$, in \eqref{eq:less_1} we use that fact that $1 - \frac{\sqrt{d} \bar{\epsilon}}{\sigma} \leq 1$ and in \eqref{eq:e_bar_def} we use the definition of $\bar{\epsilon}$. Now, we have that
\begin{align*}
    \norm{\widehat{w}_r-w_*}^2 = 2(1 - \widehat{w}_r^\t w_*) \leq \epsilon^2,
\end{align*}
proving the result.


\end{proof}

In Lemma \ref{lem:reduce_distances}, we show that at each round $\bar{x}_{i,r}$ moves closer to the labeled examples of the opposite class. The main idea behind the proof is that the algorithm can always choose a point $\tilde{x}_{i,r}$ in the direction orthogonal to the hyperplane $H_{\bar{x}_{i,r}, r}$, thus guaranteeing a reduction in $\rho_{i,r}$. Early in the execution of Algorithm \ref{alg:margin_multiple_seeds}, the nearest neighbor graph constrains which points are chosen, leading to a reduction in $\rho_{i,r}$ of $\delta$. However, once $\bar{x}_{i,r}$ is close enough to the labeled examples of the opposite class, precisely once $\rho_{i,r} < 2 \delta$, the algorithm begins selecting points that lie on $H_{\bar{x}_{i,r}, r}$ at each round, halving $\rho_{i,r}$ at each round.

\begin{lemma}\label{lem:reduce_distances}
Fix $i \in [d-1] $. Fix $r \in \N$.
\begin{enumerate}
    \item If $\rho_{i,r}  \geq 2 \delta$, then
    \begin{align*}
        \rho_{i,r+1} \leq \rho_{i,r} -  \delta.
    \end{align*}
    \item If $\rho_{i,r}   < 2 \delta$, then 
    \begin{align*}
        \rho_{i,r+1} \leq \frac{\rho_{i,r} }{2}.
    \end{align*}
\end{enumerate}
\end{lemma}

\begin{proof}

\emph{1.} Let $\bar{x}_{i,r} = \min_{x \in C_{i,r}} \gamma_{x,r}$ as defined in the algorithm. Wlog, suppose that $\bar{x}_{i,r} \in \X^+$. 
By Lemma \ref{lem:project}, there exists $z \in {\conv}(\data_r^{-1})$ such that 
$\tilde{x}_{i,r} = \bar{x}_{i,r} + \alpha \frac{(z- \bar{x}_{i,r})}{\norm{z- \bar{x}_{i,r}}} $ for some $\alpha \leq \delta$. 

Let $\beta =  \frac{\alpha }{\norm{z- \bar{x}_{i,r}}}$. Then,
\begin{align*}
    \norm{\tilde{x}_{i,r}-z} & = \norm{\bar{x}_{i,r} + \beta (z- \bar{x}_{i,r}) - z} \\
    & = \norm{(1-\beta)\bar{x}_{i,r} - z} \\
    & = (1-\beta) \norm{\bar{x}_{i,r}-z} \\
    & = \norm{\bar{x}_{i,r} - z} - \alpha .
\end{align*}
If $\dist(\bar{x}_{i,r},\data_{r}^{-1}) \geq 2\delta$, Lemma \ref{lem:project} implies that $\alpha = \delta$ and we have that
\begin{align*}
    \norm{\tilde{x}_{i,r} - z} = \norm{\bar{x}_{i,r}  - z} - \delta = \dist(\bar{x}_{i,r},{\conv}(\data_{r}^{-1})) - \delta
\end{align*}
Thus, if $\tilde{x}_{i,r} \in \X^+$, using the definition of $\bar{x}_{i,r+1}$, we have that
\begin{align*}
    \rho_{i,r+1} & =\dist(\bar{x}_{i,r+1}, {\conv}(\data_{r+1}^{-y_{\bar{x}_{i,r+1}}})) \\
    & \leq \dist(\tilde{x}_{i,r}, {\conv}(\data_{r}^{-1})) \\
    & \leq \dist(\bar{x}_{i,r},{\conv}(\data_{r}^{-1})) - \delta \\
    & = \rho_{i,r} - \delta.
\end{align*}

On the other hand, if $\tilde{x}_{i,r} \in \X^{-}$, we have that 
\begin{align*}
\rho_{i,r+1} & =  \dist(\bar{x}_{i,r+1}, {\conv}(\data_{r+1}^{-y_{\bar{x}_{i,r+1}}}))  \\
& \leq    \dist(\tilde{x}_{i,r}, {\conv}(\data_{r}^{1})) \\
& \leq \delta \\
& \leq \dist(\bar{x}_{i,r},{\conv}(\data_r^{-1})) - \delta 
\end{align*}
which also shows the claim. 

\emph{2.} Now, suppose that $\dist(\bar{x}_{i,r}, {\conv}(\data_{r+1}^{-1})) < 2\delta$. Then, by Lemma \ref{lem:project}, we have that $\alpha = \frac{\norm{z - \bar{x}_{i,r+1}}}{2} < \delta$, implying that 
\begin{align*}
   \rho_{i,r+1} & = \dist(\bar{x}_{i,r+1}, {\conv}(\data_{r}^{-y_{\bar{x}_{i,r+1}}}) \\
   & \leq  \norm{\tilde{x}_{i,r} - z} \\
   & = \norm{\tilde{x}_{i,r} - \bar{x}_{i,r} } \\
    & = \norm{\bar{x}_{i,r} - z} /2 \\
    & = \frac{\dist(\bar{x}_{i,r},{\conv}(\data_r^{-1}))}{2} \\
    & = \frac{\rho_{i,r}}{2}
\end{align*}
and therefore the result.

\end{proof}

Lemma \ref{lem:project} characterizes the example, $\tilde{x}_{i,r}$, queried by Algorithm \ref{alg:margin_multiple_seeds}. It shows that $\tilde{x}_{i,r}$ always belongs to a line segment connecting $\bar{x}_{i,r}$ to some point, $z$, in the convex hull of labeled points of the opposite class, ${\conv}(\data_{r+1}^{-y_{\bar{x}_{i,r}}})$. If $\rho_{i,r} \geq 2 \delta$, then $\tilde{x}_{i,r}$ moves $\delta$ along this line segment towards $z$ and if $\rho_{i,r} < 2 \delta$, $\tilde{x}_{i,r}$ is the midpoint of this line segment.
 
 \begin{lemma}\label{lem:project}
 Fix round $r \in \N$. There exists 
 $z \in {\conv}(\data_{r+1}^{-y_{\bar{x}_{i,r}}})$ such that the following holds.
 If $\rho_{i,r} \geq 2 \delta$, then $\tilde{x}_{i,r} = \bar{x}_{i,r} + \delta \frac{(z-\bar{x}_{i,r})}{\norm{z- \bar{x}_{i,r}}}$. If $\rho_{i,r} < 2 \delta$, then $\tilde{x}_{i,r} = \frac{\bar{x}_{i,r} + z}{2}$.
 \end{lemma}
 
\begin{proof} 
\textbf{Step 1: A formula for the max-margin separator.}

Without loss of generality, suppose that $\bar{x}_{i,r} \in \X^+$ (the other case is similar). Let $\bar{w} \in \R^d$, $\bar{b} \in \R$, and $\bar{t} \in \R$ denote the optimal solutions of the optimization problem
\begin{align}
    \max_{w \in \R^d, b \in \R, t \in \R} \, & t \label{eq:max_margin} \\
    \text{ s.t. } & \bar{x}_{i,r}^\t w - b \leq -t \nonumber \\
    & x^\t w - b \geq t \, \forall x \in \data_{r+1}^{-1} \nonumber \\
    & \norm{w}_2 \leq 1. \nonumber
\end{align}
Then, $\bar{w} \in \R^d$ and $\bar{b} \in \R$ define the max-margin separator separating $\bar{x}_{i,r}$ from $\data_{r+1}^{-1}$. We note that by Lemma \ref{lem:unique}, $\bar{w}$ and $\bar{b} $ are the unique solutions up to scaling. 
By Section 8.6.1. in \cite{boyd2004convex}, \eqref{eq:max_margin} has the same value as 
\begin{align}
    \min_{\alpha_j} & \frac{1}{2} \norm{\sum_{j: x_j \in \data_{r+1}^{-1}}  \alpha_j x_j - \bar{x}_{i,r}} \label{eq:max_margin_dual} \\
    \text{s.t. } &  \alpha_j \geq 0 \,  \forall j  \nonumber\\
    & \sum_j \alpha_j = 1 \nonumber
\end{align}
Let $\{\tilde{\alpha}_j\}$ attain the maximum in the above optimization problem, which exists since the domain is compact and the objective function is continuous. Define $\tilde{x} = \sum_{j: x_j \in \data_{r+1}^{-1}}  \tilde{\alpha}_j x_j$ and
\begin{align*}
    \tilde{w} & = \frac{\tilde{x} - \bar{x}_{i,r}}{\norm{\tilde{x} - \bar{x}_{i,r}}} \\
    \tilde{b} & = \frac{\norm{\tilde{x}}^2 - \norm{\bar{x}_{i,r}}^2}{2\norm{\tilde{x} - \bar{x}_{i,r}}}
\end{align*}
We claim that $\tilde{w} = \bar{w}$ and $\tilde{b} = \bar{b}$. First, we show that that there exists $\tilde{t} \in \R$ such that  $(\tilde{w}, \tilde{b}, \tilde{t})$ satisfy the constraints in \eqref{eq:max_margin}. Arithmetic shows that $\tilde{w}^\t \tilde{x} - \tilde{b} > 0$ and $\tilde{w}^\t \bar{x}_{i,r} - \tilde{b} < 0$. Since $\tilde{x}$ is the projection of $\bar{x}_{i,r}$ onto ${\conv}(\data^+_r)$, by the Projection Lemma, we have that
\begin{align*}
    (\tilde{x} - \bar{x}_{i,r})^\t x \geq (\tilde{x} - \bar{x}_{i,r})^\t \tilde{x}
\end{align*}
for all $x \in {\conv}(\data^+_r)$. Thus, for all $x \in {\conv}(\data^+_r)$, we have that
\begin{align*}
    \tilde{w}^\t x - \tilde{b}  \geq \tilde{w}^\t \tilde{x} - \tilde{b} > 0.
\end{align*}
We conclude that there exists $\tilde{t} \in \R$ such that $(\tilde{w}, \tilde{b}, \tilde{t})$ satisfy the constraints in \eqref{eq:max_margin}.

We have that
\begin{align*}
    \tilde{w }^\t \bar{x}_{i,r} - \tilde{b} = \frac{1}{\norm{\tilde{x}-\bar{x}_{i,r}}}(\bar{x}_{i,r}^\t (\tilde{x} - \bar{x}_{i,r}) - \frac{\norm{\tilde{x}}^2 - \norm{\bar{x}_{i,r}^2}}{2}) = -\frac{1}{2}\norm{\bar{x}_{i,r}-\tilde{x}}.
\end{align*}
A similar calculation shows that $\tilde{w}^\t x - \tilde{b}  \geq  \frac{1}{2}\norm{\bar{x}_{i,r}-\tilde{x}}$ for all $x \in {\conv}(\data^+_r)$. Thus, putting $\tilde{t} = \frac{1}{2} \norm{\tilde{x} - \bar{x}_{i,r}}$, $(\tilde{w},\tilde{b},\tilde{t})$ is feasible to \eqref{eq:max_margin}. By the equivalence in value of \eqref{eq:max_margin} and \eqref{eq:max_margin_dual} and the definition of $\tilde{x}$, we have that $\frac{1}{2} \norm{\tilde{x}-\bar{x}_{i,r}} = \bar{t}$. Thus, by uniqueness of $\bar{w}$ and $\bar{b}$, we have that $\tilde{w} = \bar{w}$ and $\tilde{b} = \bar{b}$. 

\textbf{Step 2: Putting it together.} We have shown that $\tilde{w}$ and $\tilde{b}$ define the max-margin separator. Let $P\bar{x}_{i,r}$ denote the projection of $\bar{x}_{i,r}$ onto $H := \{z \in \R^d : \tilde{w}^\t z = \tilde{b} \}$. We have that
\begin{align*}
    P \bar{x}_{i,r} & = \bar{x}_{i,r} + (\frac{\norm{\tilde{x}}^2 - \norm{\bar{x}_{i,r}}^2}{2 \norm{\tilde{x}-\bar{x}_{i,r}}} - (\frac{\tilde{x} - \bar{x}_{i,r}}{\norm{\tilde{x}-\bar{x}_{i,r}}})^\t \bar{x}_{i,r} \frac{\tilde{x}-\bar{x}_{i,r}}{\norm{\tilde{x}-\bar{x}_{i,r}}} \\
    & = \frac{\tilde{x}+\bar{x}_{i,r}}{2}.
\end{align*}
Suppose that $\dist(\bar{x}_{i,r}, {\conv}(\data_{r+1}^{-y_{\bar{x}_{i,r}}})) \geq 2 \delta$. Then, we have that $\dist(\bar{x}_{i,r}, H) > 2 \delta$. It can be easily seen that $\tilde{x}_{i,r} = \argmin_{x : \norm{x -\bar{x}_{i,r}} \leq \delta} \dist(x,H) = \bar{x}_{i,r} + \delta\frac{\tilde{x}-\bar{x}_{i,r}}{\norm{\tilde{x}-\bar{x}_{i,r}}}$. Note that $\bar{x}_{i,r} + \delta\frac{\tilde{x}-\bar{x}_{i,r}}{\norm{\tilde{x}-\bar{x}_{i,r}}} \in \X$, since $\bar{x}_{i,r}, \tilde{x} \in \X$ and $\X$ is convex.   

Similarly, if $\dist(\bar{x}_{i,r}, {\conv}(\data_{r+1}^{-y_{\bar{x}_{i,r}}})) < 2 \delta$, Then, we have that $\dist(\bar{x}_{i,r}, H) < \delta$. Then, by definition of $\tilde{x}_{i,r}$ we have that $\tilde{x}_{i,r} = \frac{\tilde{x}+\bar{x}_{i,r}}{2} \in \X$, where we have that $\frac{\tilde{x}+\bar{x}_{i,r}}{2} \in \X$ by convexity of $\X$.

\end{proof}

The following result shows that the max-margin separator is unique, and is a standard result on SVMs. 

\begin{lemma}\label{lem:unique}
Let $A,B \subset \R^d$ be disjoint, closed, and convex. The max-margin separator separating $A$ and $B$ is unique.
\end{lemma}

\begin{proof}
There exists a separating hyperplane between $A$ and $B$ by the separating hyperplane Theorem. By a standard argument, the optimization problem for the max-margin separator can be stated as 
\begin{align*}
    \min_{w \in \R^d, b \in \R} & \norm{w}^2 \\
    \text{ s.t. } & w^\t x + b \geq 1 \, \forall x \in A \\
    & w^\t x + b \leq -1 \, \forall x \in B.
\end{align*}
This is a convex optimization problem with a strongly convex objective and, therefore, has a unique solution. 
\end{proof}

\subsection{SEALS with Querying Anywhere Capability}

\begin{algorithm}[h]
 \caption{Modified SEALS: Project onto Hyperplane}
 \label{alg:margin_multiple_seeds_project}
  \begin{algorithmic}[1]
 \State \textbf{Input: } seed labeled points $\data_{1}  \subset \X \times \{-1,1\}$   \; 
$r=1$  
\State Initialize the clusters $C_{i,r} = \{x_{i,0}\}$ for $i =1,\ldots, d-1$ 
\For{$r=1,2,\ldots$}
\For{$i=1,\ldots,d-1$}
\State $(H_{x,r},\gamma_{x,r}) =\text{MaxMarginSeparator}(\data_{r}^{-y_x}, x)$ for all $x \in C_{i,r} $
\State Let $\bar{x}_{i,r} \in \argmin_{x \in C_{i,r}} \gamma_{x,r}$ and $\widetilde{x}_{i,r} \in  \argmin_{x^\prime } \dist(x^\prime,H_{\bar{x}_{i,r},r})$
\State Query $\widetilde{x}_{i,r}$ 
\State Update $\data_{r+1} = \data_r \cup \{(\widetilde{x}_{i,r}, y_{\widetilde{x}_{i,r}})\}$ and $C_{i,r+1} = C_{i,r} \cup \{\widetilde{x}_{i,r}\}$ 
\EndFor

\State Fit a homogenous max-margin separator with normal vector $\widehat{w}_{r+1}$ to $\data_{r+1}$
\EndFor
 \end{algorithmic}
\end{algorithm}

The data structure used in SEALS enables queries of the $k$-nearest neighbors of any point in the input space. Therefore, a natural question is whether we can leverage this querying capability to improve the sample complexity of Algorithm \ref{alg:margin_multiple_seeds}, \emph{at the cost of being less generic than SEALS}. To address this question, we consider the Algorithm \ref{alg:margin_multiple_seeds_project}, which queries the nearest neighbor of the projection onto the max-margin separator. The key takeaway is that this modification of the algorithm removes the slow phase in the sample complexity of SEALS. The analysis requires a slightly different definition of $\Z_i$:
\begin{align*}
    \Z_i & = \{z \in \R^d : \norm{z - x_{i,0}} \leq 2\gamma_i + \epsilon \text{ and } \dist(z, \{x \in \R^d : x^\t w_* = 0\} ) \leq \epsilon \}.
\end{align*}

\begin{theorem}\label{thm:main_thm_supp_project}
Let $\epsilon > 0$. Let $x_{1,0},\ldots, x_{d-1,0}$ denote the seed points. Define $\gamma_i = \dist(x_{i,0}, {{\conv}}(\data^{-y_{x_{i,0}}}_1))$ for $i \in [d-1]$, where ${\conv}(\data^{-y_{x_{i,0}}}_1)$ is the convex hull of the points $\data^{-y_{x_{i,0}}}_1.$ Then, after Algorithm \ref{alg:margin_multiple_seeds_project} makes $ \max_{i \in [d-1]} d( 2\log(\frac{2 d \gamma_i }{\epsilon \min(\sigma,1) })))$ queries, we have that $\norm{\widehat{w}_r-w_*} \leq \epsilon.$
\end{theorem}

\begin{proof}[Proof Sketch]
The proof is very similar to the of Theorem \ref{thm:main_thm_supp}. Step 1 of Theorem \ref{thm:main_thm_supp} is essentially the same, but it does not matter whether $\rho_{i,r} \geq 2 \delta$ since the algorithm is not constrained by the nearest neighbor graph. A similar argument shows that after $\max_{i \in [d-1]}   \log(\frac{2 \gamma_i}{\bar{\epsilon}  })$ queries, we have that for all $i \in [d-1]$,  $\rho_{i,r} := \dist(\bar{x}_{i,r}, {\conv}(\data_r^{-y_{x_{i,r}}})) \leq \bar{\epsilon}$. 

Step 2 is also quite similar, except that we now have using a similar argument about the geometric series that
\begin{align*}
     \norm{x_{i,r} - x_{i,0}} & \leq 2\gamma_i .
\end{align*}

Step 3 is exactly the same. 

\end{proof}

\clearpage

\subsection{Impact of Embedding model ($G_z$) on \flood}\label{supp:varying_gz}

\begin{figure}[H]
  \centering
  \begin{subfigure}{0.95\textwidth}
    \includegraphics[width=\columnwidth]{figures/imagenet_legend_al.pdf}
  \end{subfigure}

  \begin{subfigure}{0.95\textwidth}
    \includegraphics[width=\columnwidth]{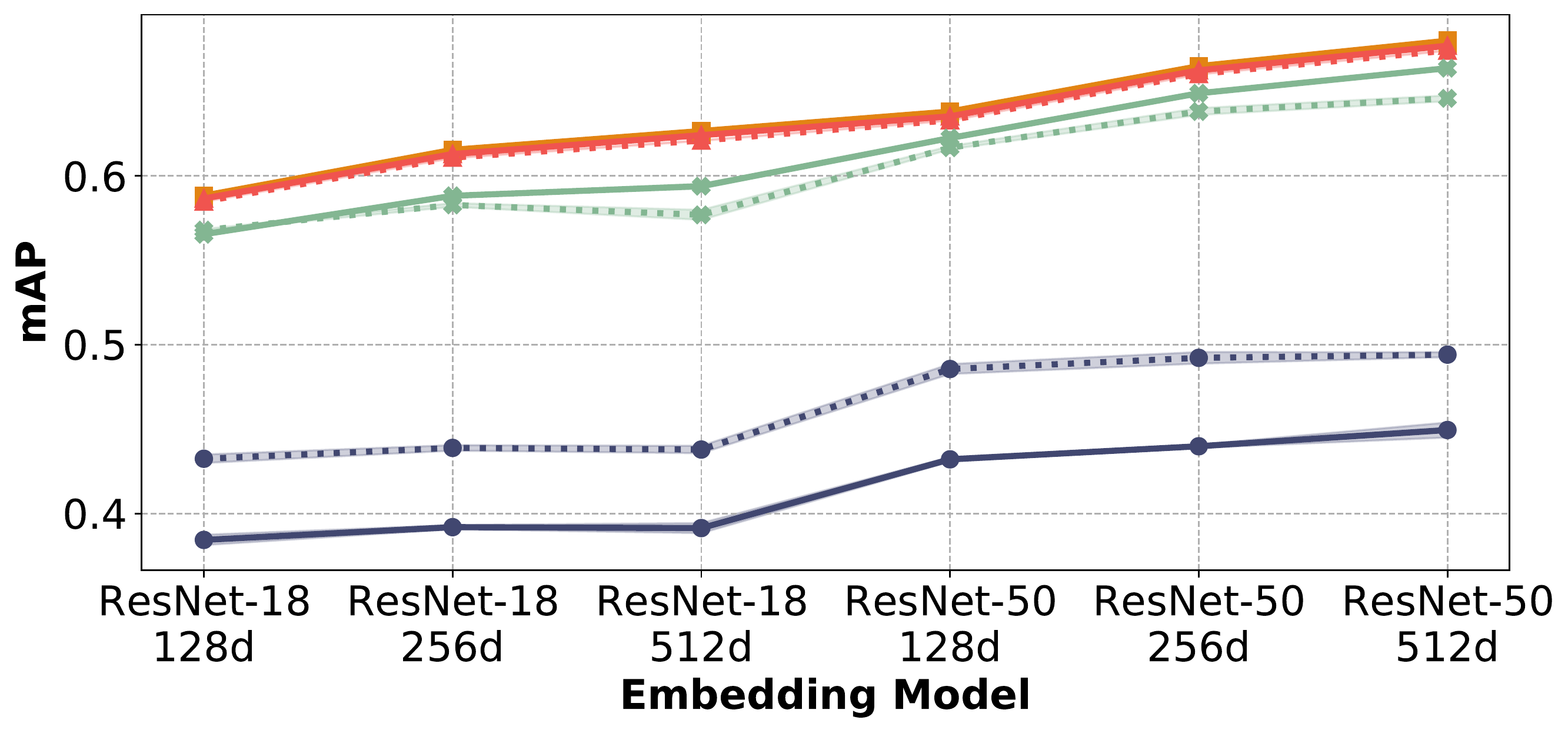}
  \end{subfigure}
\caption{
    Active learning on ImageNet with varying embedding models (ResNet-18 or ResNet-50) and dimensions (128, 256, or 512).
    Performance increases with larger models and higher dimensional embeddings.
    However, \flood achieves similar performance to the baseline approach regardless of the choice of model and dimension, empirically demonstrating  \flood' robustness to the embedding function.
}
\label{fig:imagenet_other_varying_gz}
\end{figure}

\begin{figure}[H]
  \centering
  \begin{subfigure}{0.95\textwidth}
    \includegraphics[width=\columnwidth]{figures/imagenet_legend_al.pdf}
  \end{subfigure}

  \begin{subfigure}{0.95\textwidth}
    \includegraphics[width=\columnwidth]{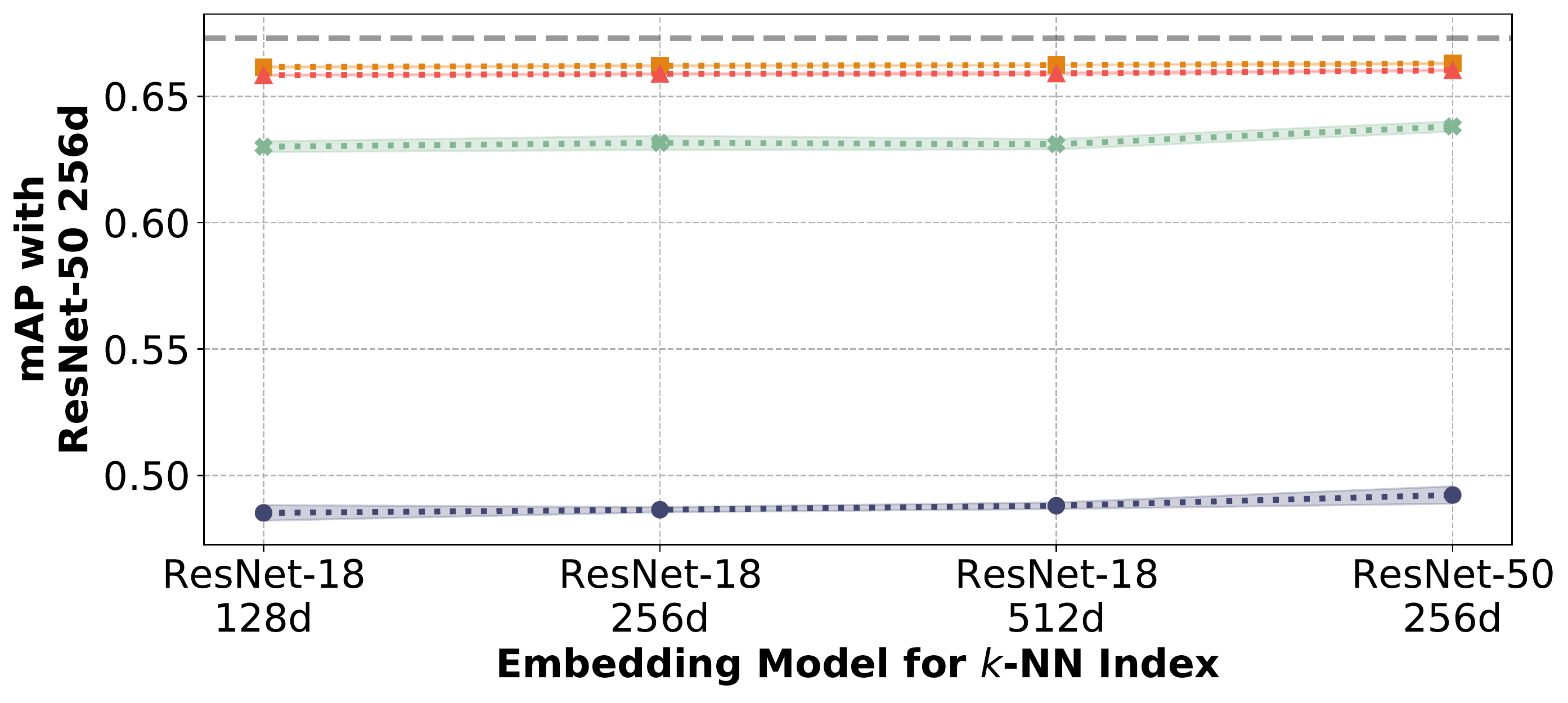}
  \end{subfigure}
\caption{
    Active learning on ImageNet with 256-dimensional ResNet-50 embeddings and varying k-NN indices.
    Different embeddings might be used for learning rare concepts than the embeddings used for similarity search in practice.
    Fortunately, \flood performs similarly for varying $k$-NN indices, as shown above.
    This can also be exploited to reduce further the cost of constructing the index by using a smaller, cheaper model to generate the embedding for similarity search and only applying the larger model to examples added to the candidate pool.
}
\label{fig:imagenet_other_varying_gz_knn}
\end{figure}
\clearpage

\subsection{Impact of $k$ on \flood}\label{supp:varying_k}

\begin{figure}[H]
  \centering
  \begin{subfigure}{0.95\textwidth}
    \includegraphics[width=\columnwidth]{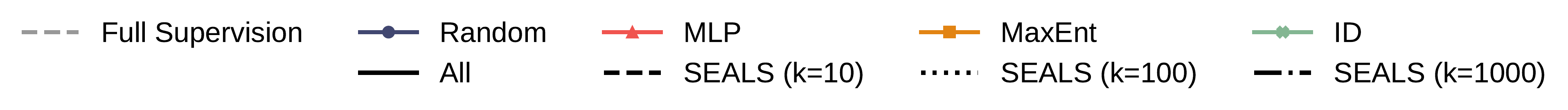}
  \end{subfigure}

  \begin{subfigure}{0.95\textwidth}
    \includegraphics[width=\columnwidth]{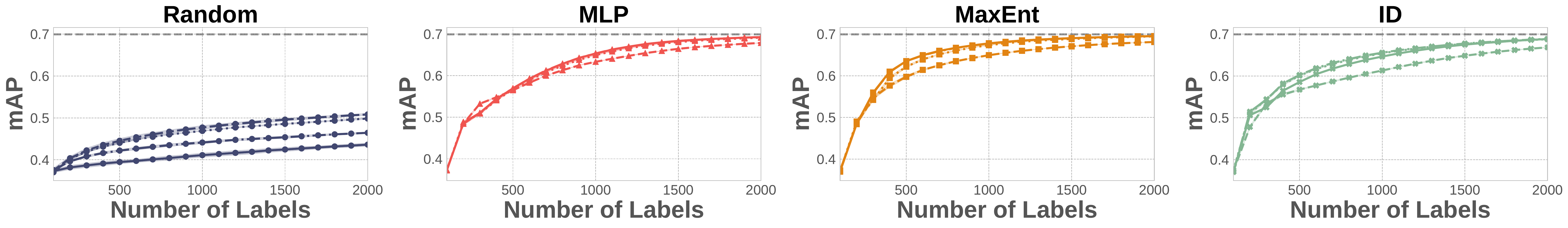}
    \label{fig:imagenet_ks_map}
  \end{subfigure}

  \begin{subfigure}{0.95\textwidth}
    \includegraphics[width=\columnwidth]{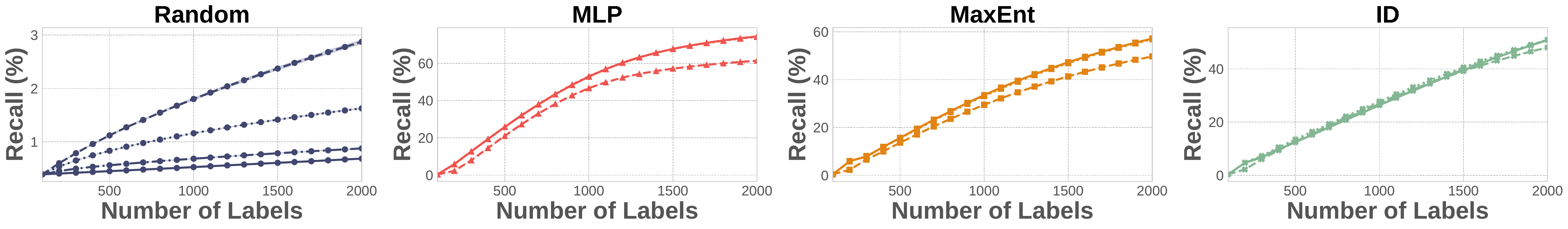}
    \label{fig:imagenet_ks_recall}
    \end{subfigure}

  \begin{subfigure}{0.95\textwidth}
    \includegraphics[width=\columnwidth]{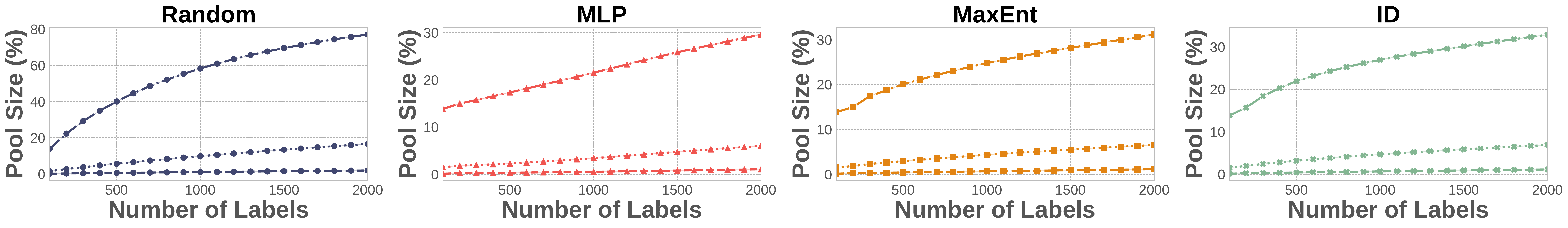}
    \label{fig:imagenet_ks_pool_size}
  \end{subfigure}
\caption{
\textbf{Impact of increasing $k$ on ImageNet ($|U|$=639,906).}
Larger values of $k$ help to close the gap between \flood and the baseline approach for active learning (top) and active search (middle).
However, increasing $k$ also increases the candidate pool size (bottom), presenting a trade-off between labeling efficiency and computational efficiency.
}
  \label{fig:imagenet_ks}
\end{figure}

\begin{figure}[H]
  \centering
  \begin{subfigure}{0.95\textwidth}
    \includegraphics[width=\columnwidth]{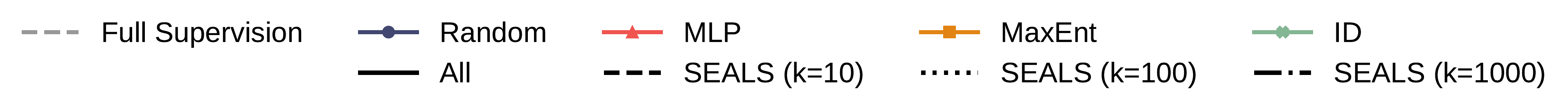}
  \end{subfigure}

  \begin{subfigure}{0.95\textwidth}
    \includegraphics[width=\columnwidth]{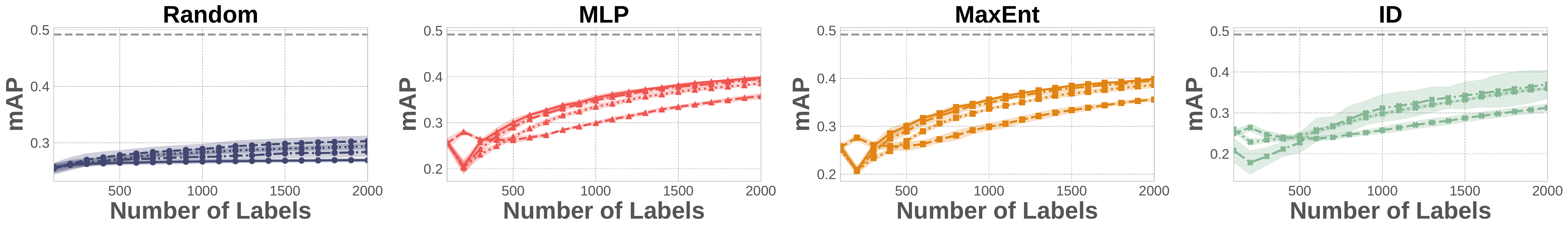}
    \label{fig:openimages_ks_map}
  \end{subfigure}

  \begin{subfigure}{0.95\textwidth}
    \includegraphics[width=\columnwidth]{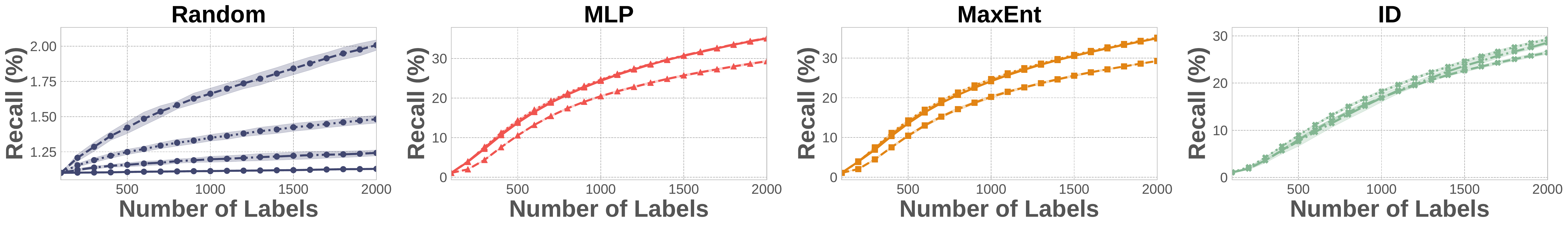}
    \label{fig:openimages_ks_recall}
    \end{subfigure}

  \begin{subfigure}{0.95\textwidth}
    \includegraphics[width=\columnwidth]{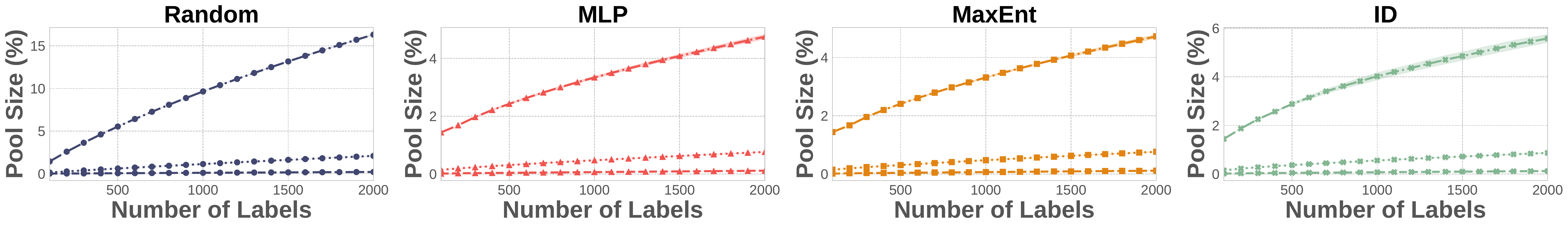}
    \label{fig:openimages_ks_pool_size}
  \end{subfigure}
\caption{
\textbf{Impact of increasing $k$ on OpenImages ($|U|$=6,816,296).}
Larger values of $k$ help to close the gap between \flood and the baseline approach for active learning (top) and active search (middle).
However, increasing $k$ also increases the candidate pool size (bottom), presenting a trade-off between labeling efficiency and computational efficiency.
}
  \label{fig:openimages_ks}
\end{figure}
\clearpage

\subsection{Impact of the number of initial positives on \flood}\label{supp:varying_positives}

\begin{figure}[H]
  \centering
  \begin{subfigure}{0.95\textwidth}
    \includegraphics[width=\columnwidth]{figures/imagenet_legend_al.pdf}
  \end{subfigure}

  \begin{subfigure}{0.95\textwidth}
    \includegraphics[width=\columnwidth]{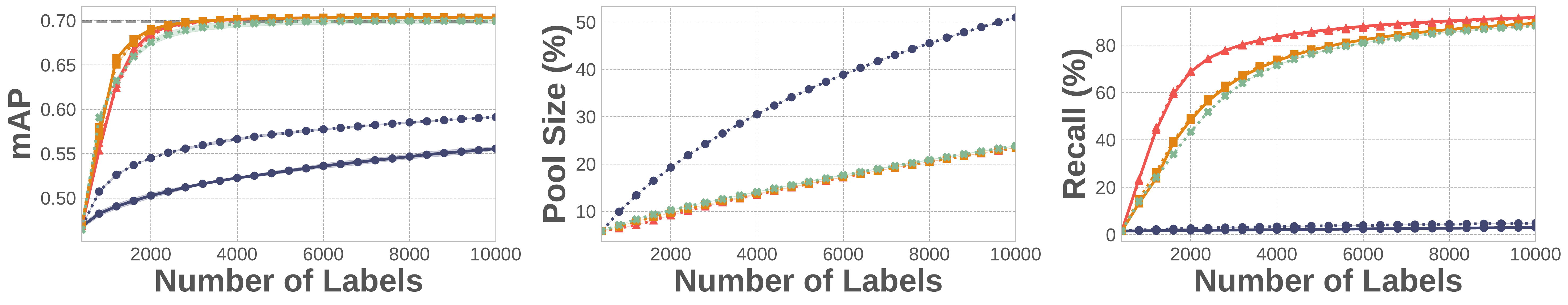}
    \caption{ImageNet ($|U|$=639,906)}
    \label{fig:imagenet_pos20}
  \end{subfigure}\vspace{2mm}

  \begin{subfigure}{0.95\textwidth}
    \includegraphics[width=\columnwidth]{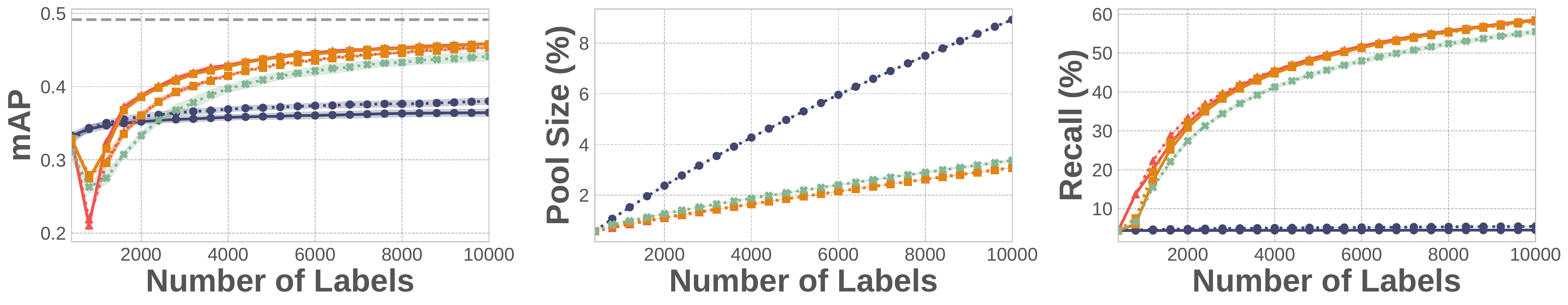}
    \caption{OpenImages ($|U|$=6,816,296)}
    \label{fig:openimages_pos20}
    \end{subfigure}\vspace{2mm}
\caption{
\textbf{Active learning and search with 20 positive seed examples} and a labeling budget of 10,000 examples on ImageNet (top) and OpenImages (bottom).
Across datasets and strategies, \flood with $k=100$ performs similarly to the baseline approach in terms of both the error the model achieves for active learning (left) and the recall of positive examples for active search (right), while only considering a fraction of the unlabeled data $U$ (middle).
}
  \label{fig:active_learning_and_search_pos20}
\end{figure}

\begin{figure}[H]
  \centering
  \begin{subfigure}{0.95\textwidth}
    \includegraphics[width=\columnwidth]{figures/imagenet_legend_al.pdf}
  \end{subfigure}

  \begin{subfigure}{0.95\textwidth}
    \includegraphics[width=\columnwidth]{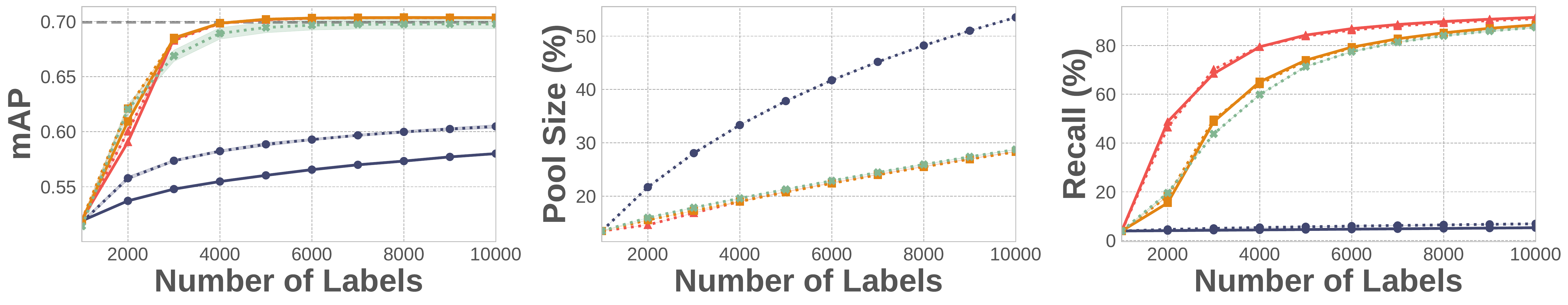}
    \caption{ImageNet ($|U|$=639,906)}
    \label{fig:imagenet_pos50}
  \end{subfigure}\vspace{2mm}

  \begin{subfigure}{0.95\textwidth}
    \includegraphics[width=\columnwidth]{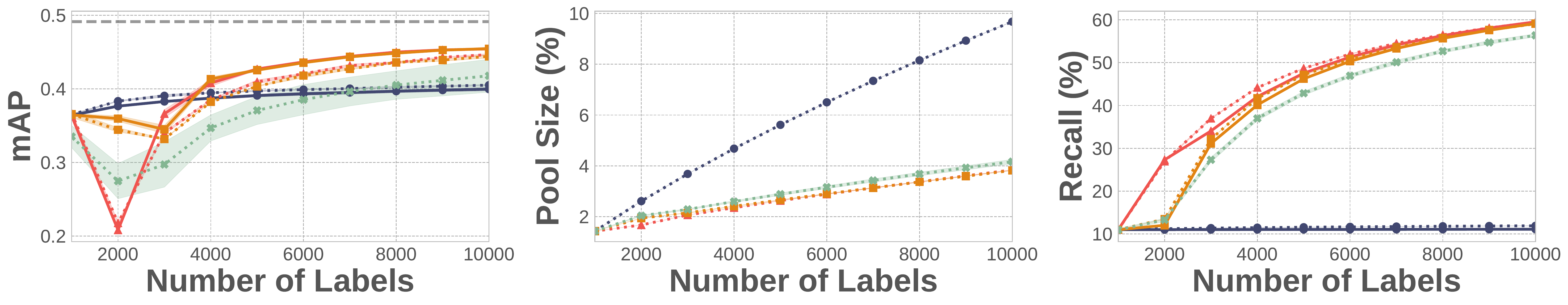}
    \caption{OpenImages ($|U|$=6,816,296)}
    \label{fig:openimages_pos50}
    \end{subfigure}\vspace{2mm}
\caption{
\textbf{Active learning and search with 50 positive seed examples} and a labeling budget of 10,000 examples on ImageNet (top) and OpenImages (bottom).
Across datasets and strategies, \flood with $k=100$ performs similarly to the baseline approach in terms of both the error the model achieves for active learning (left) and the recall of positive examples for active search (right), while only considering a fraction of the unlabeled data $U$ (middle).
}
  \label{fig:active_learning_and_search_pos50}
\end{figure}

\subsection{Latent structure}

\begin{figure}[H]
  \centering
  \begin{subfigure}{0.475\textwidth}
      \includegraphics[width=\columnwidth]{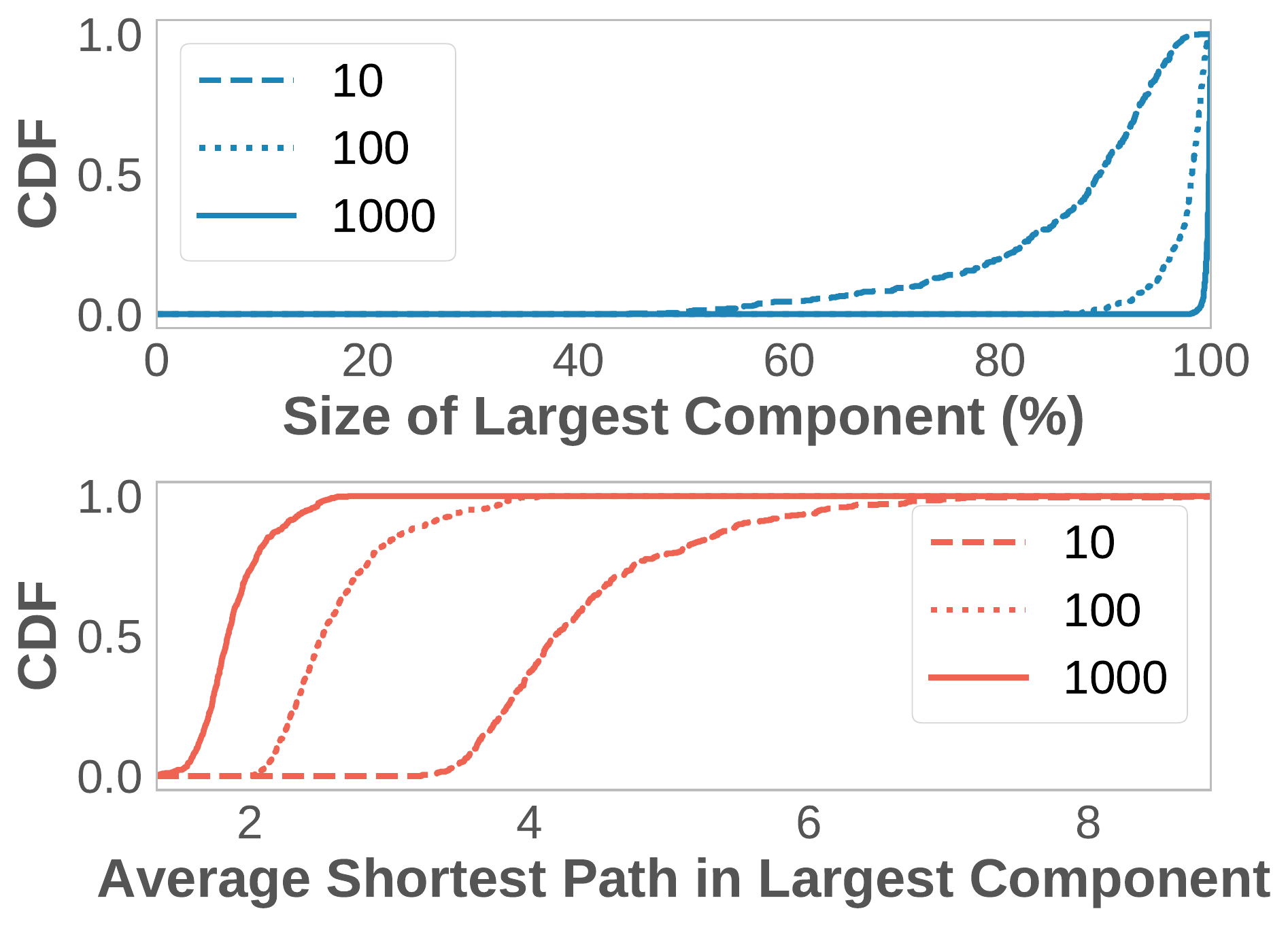}
    \caption{ImageNet}
    \label{fig:imagenet_clusters}
  \end{subfigure}
  \hfill
  \begin{subfigure}{0.475\textwidth}
      \includegraphics[width=\columnwidth]{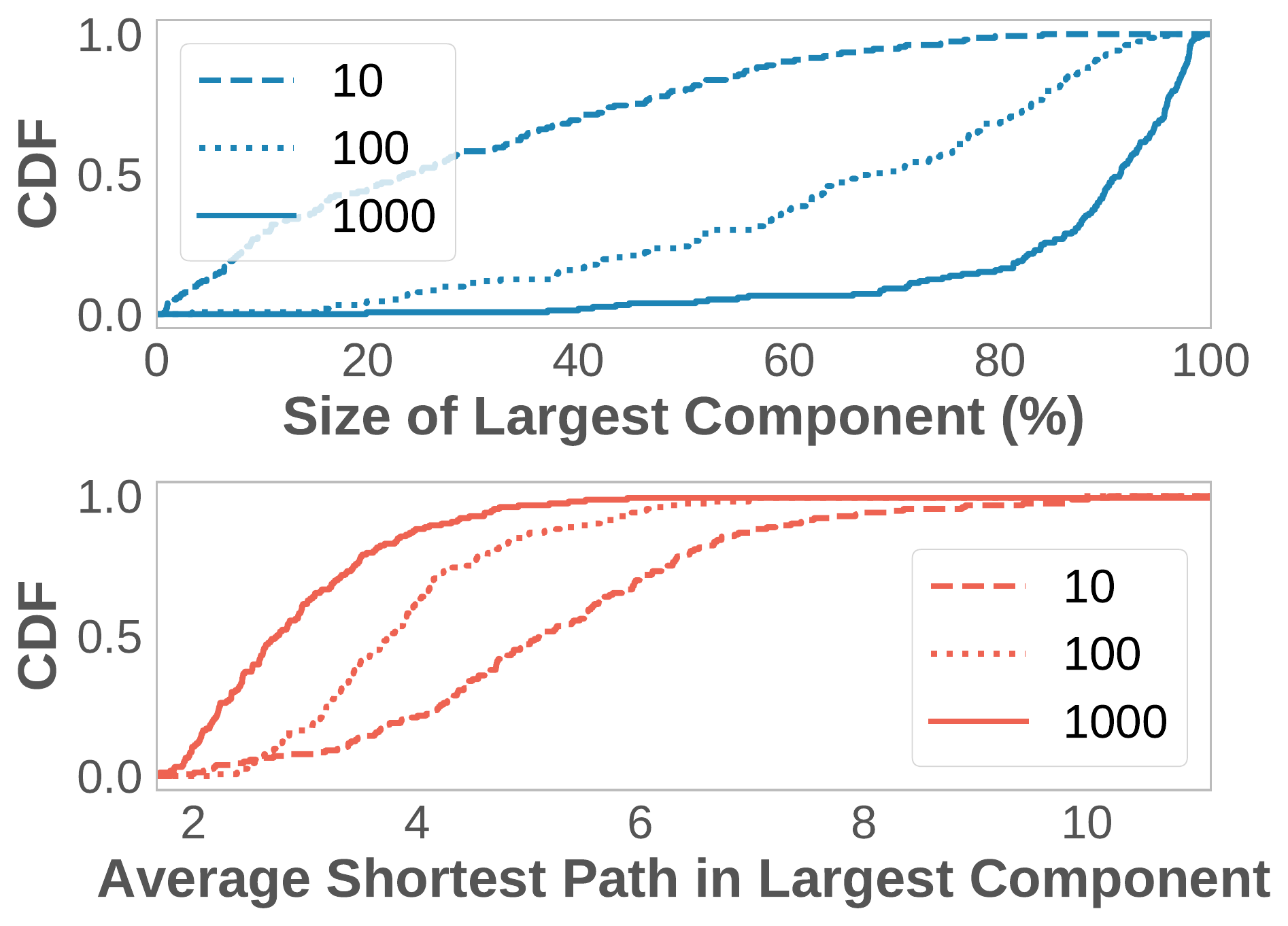}
    \caption{OpenImages}
    \label{fig:openimages_clusters}
  \end{subfigure}
  \caption{
  Measurements of the latent structure of unseen concepts in ImageNet and OpenImages.
  The 10B images dataset was excluded because only a few thousand examples were labeled.
  The largest connected component gives a sense of how much of the concept \flood can reach, while the average shortest path serves as a proxy for how long it will take to explore.
  }
  \label{fig:latent_structure}
\end{figure}

\begin{figure}[H]
  \centering
  \begin{subfigure}{.80\textwidth}
      \begin{subfigure}{.49\columnwidth}
          \includegraphics[width=\columnwidth]{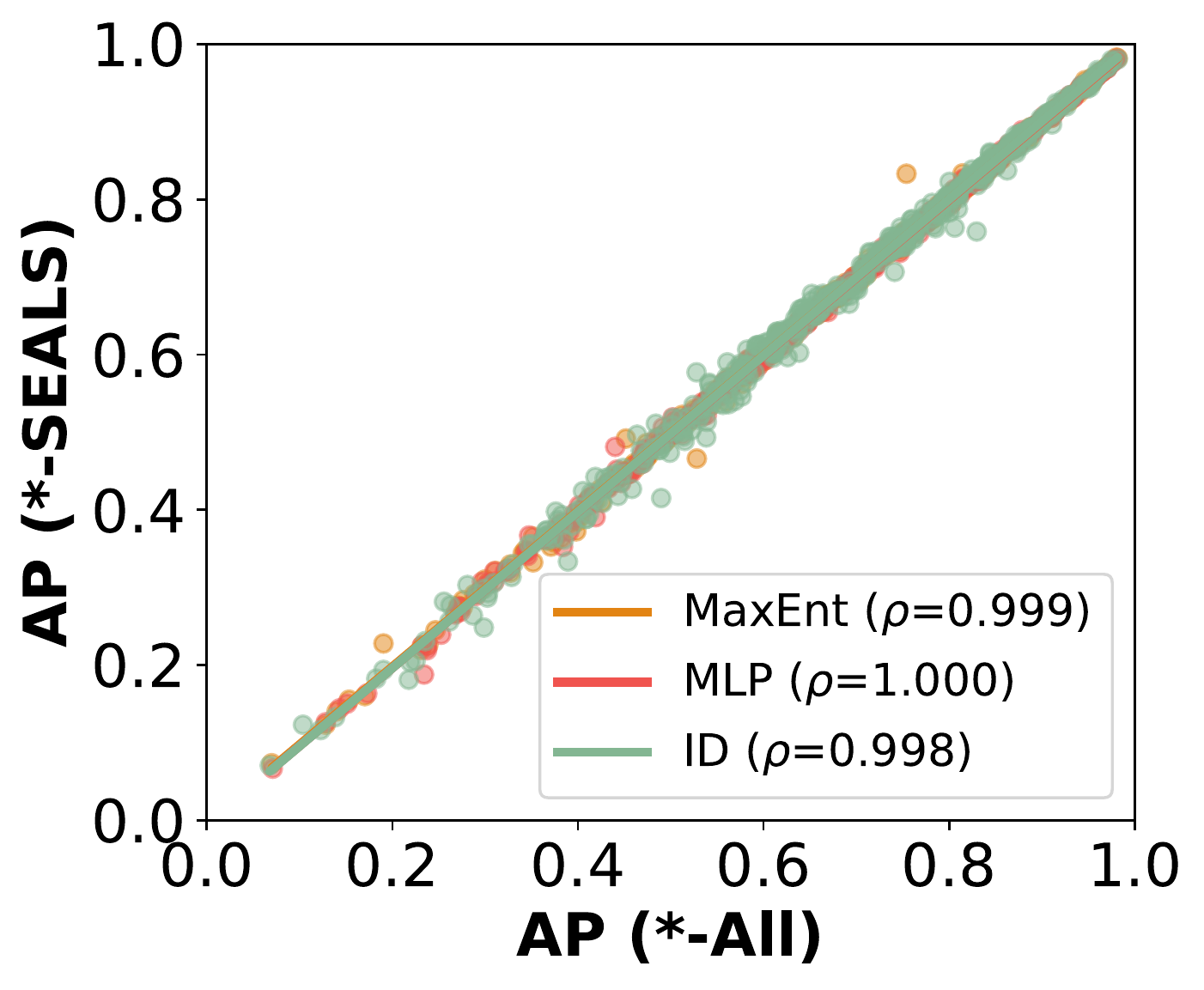}
        \caption{ImageNet}
        \label{fig:imagenet_per_class_ap}
      \end{subfigure}
      \begin{subfigure}{.49\columnwidth}
          \includegraphics[width=\columnwidth]{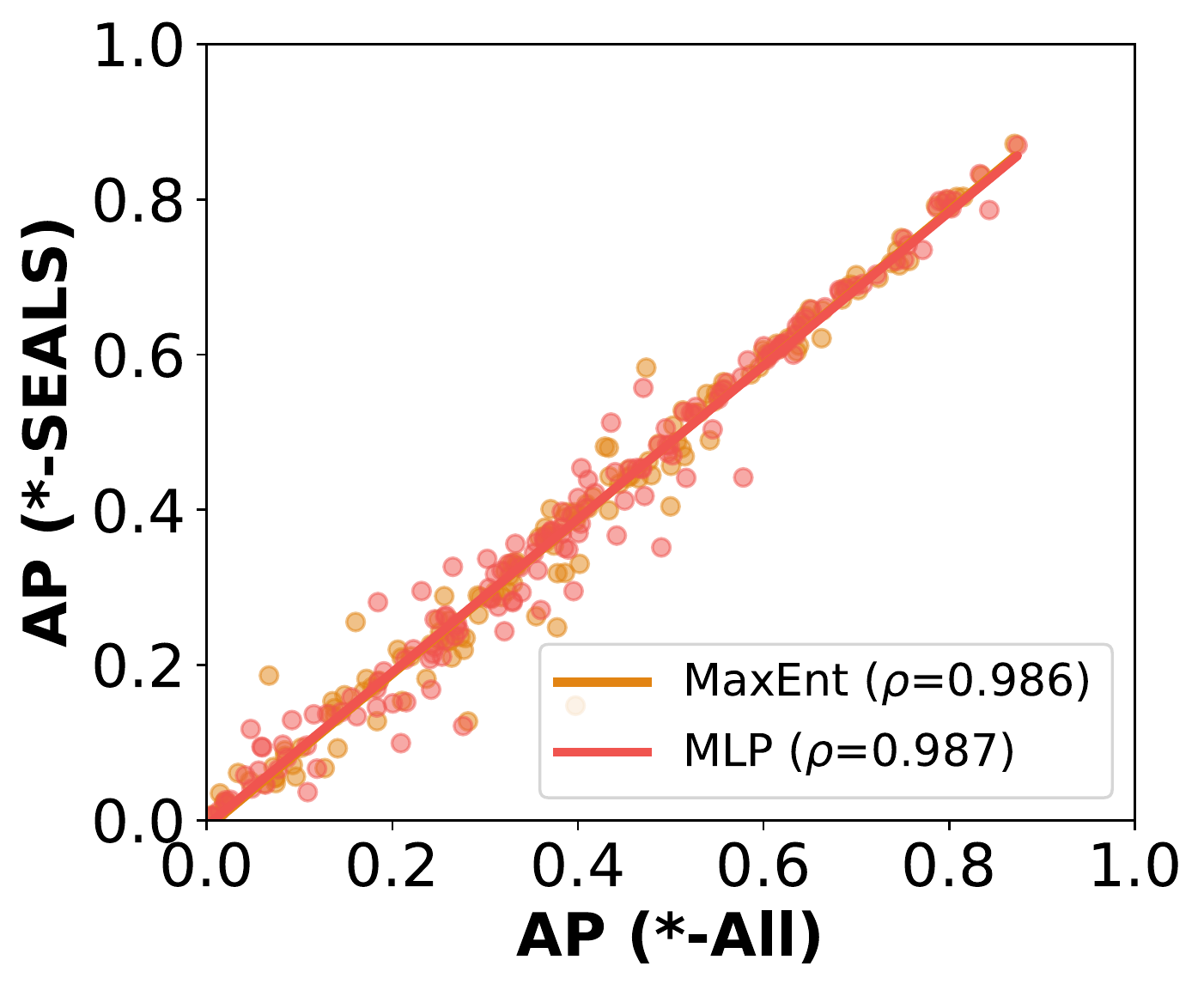}
        \caption{OpenImages}
        \label{fig:openimages_per_class_ap}
      \end{subfigure}
  \end{subfigure}\vspace{-1mm}
\caption{
    The per-class APs of \flood ($k=100$) were highly correlated to the baseline approaches (*-All) for active learning on ImageNet (right) and OpenImages (left).
On OpenImages with a budget of 2,000 labels, the Pearson’s correlation ($\rho$) between the baseline and SEALS for the average precision of individual classes was 0.986 for MaxEnt and 0.987 for MLP. The least-squares fit had a slope of 0.99 and y-intercept of -0.01. On ImageNet, the correlations were even higher.
}\vspace{-2mm}
  \label{fig:per_class_ap}
\end{figure}
\clearpage

\subsection{Comparison to pool of randomly selected examples}\label{appendix:random_pool}

\begin{figure}[H]
  \centering
  \begin{subfigure}{0.95\textwidth}
    \includegraphics[width=\columnwidth]{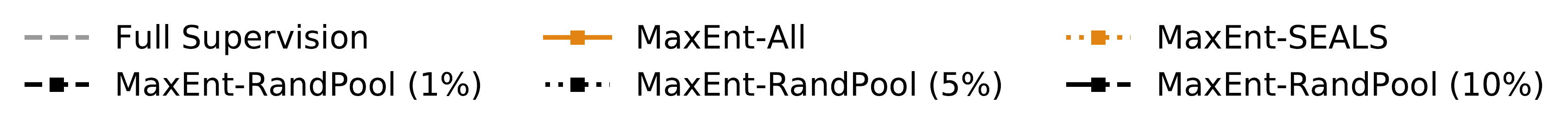}
  \end{subfigure}

  \begin{subfigure}{0.95\textwidth}
    \includegraphics[width=\columnwidth]{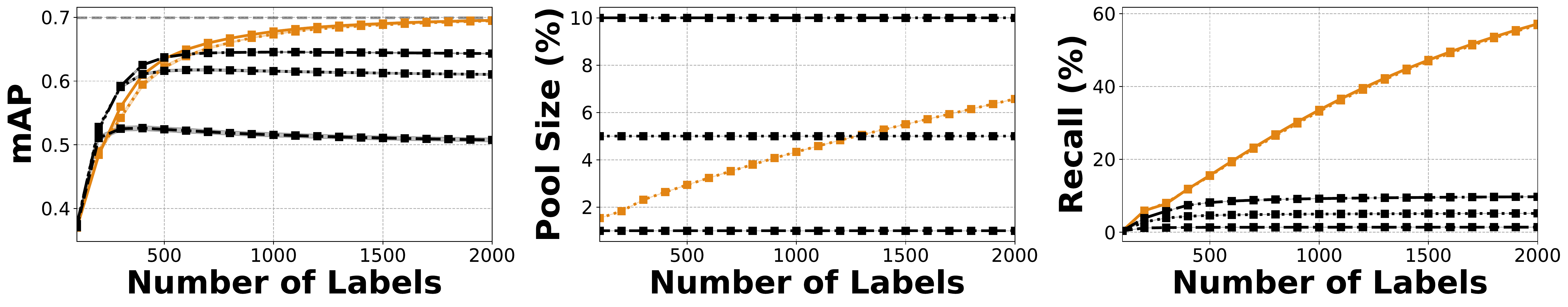}
    \caption{ImageNet ($|U|$=639,906)}
    \label{fig:imagenet_random_pool_entropy}
  \end{subfigure}

  \begin{subfigure}{0.95\textwidth}
    \includegraphics[width=\columnwidth]{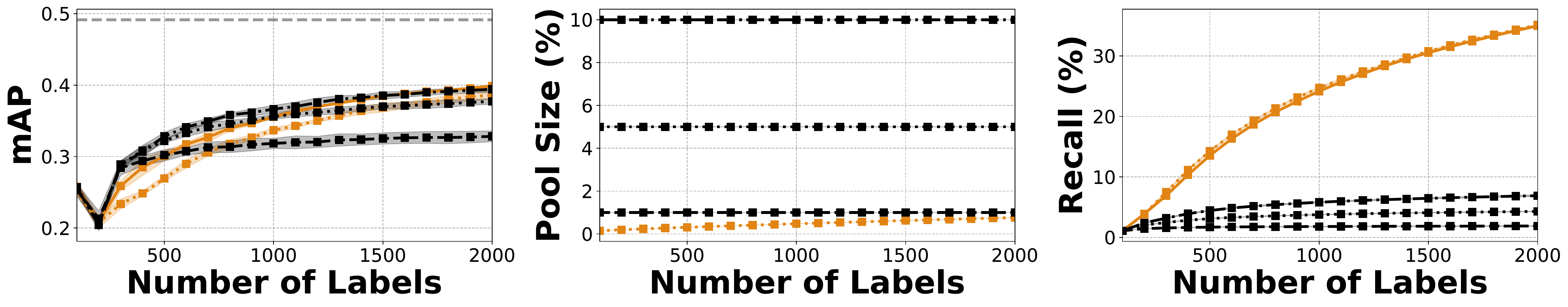}
    \caption{OpenImages ($|U|$=6,816,296)}
    \label{fig:openimages_random_pool_entropy}
    \end{subfigure}

\caption{
MaxEnt-\flood ($k=100$) versus MaxEnt applied to a candidate pool of randomly selected examples (RandPool).
Because the concepts we considered were so rare, as is often the case in practice, randomly chosen examples are unlikely to be close to the decision boundary, and a much larger pool is required to match \flood.
On ImageNet (top), MaxEnt-\flood outperformed MaxEnt-RandPool in terms of both the error the model achieves for active learning (left) and the recall of positive examples for active search (right) even with a pool containing 10\% of the data (middle).
On Openimages (bottom), MaxEnt-RandPool needed at least $5\times$ as much data to match MaxEnt-\flood for active learning and failed to achieve similar recall even with $10\times$ the data.
}
  \label{fig:random_pool_entropy}
\end{figure}
\clearpage

\begin{figure}[H]
  \centering
  \begin{subfigure}{0.95\textwidth}
    \includegraphics[width=\columnwidth]{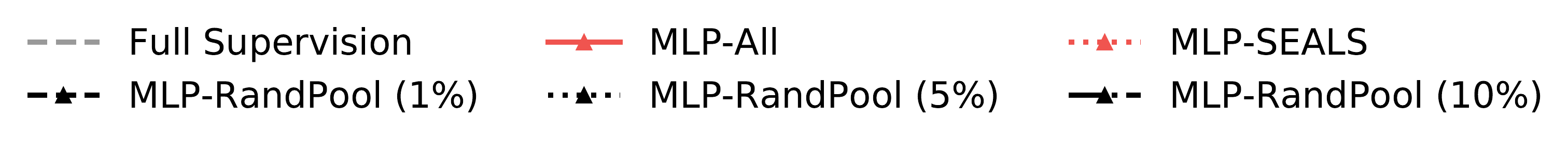}
  \end{subfigure}

  \begin{subfigure}{0.95\textwidth}
    \includegraphics[width=\columnwidth]{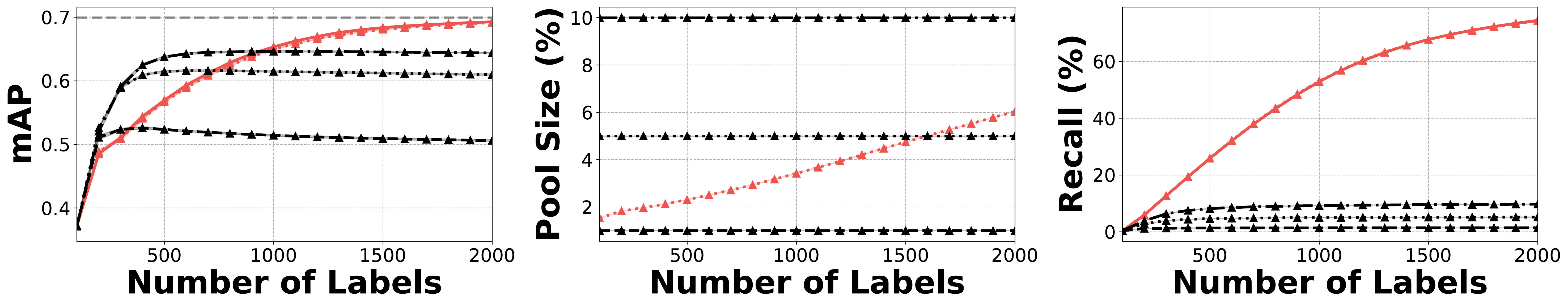}
    \caption{ImageNet ($|U|$=639,906)}
    \label{fig:imagenet_random_pool_mostpos}
  \end{subfigure}

  \begin{subfigure}{0.95\textwidth}
    \includegraphics[width=\columnwidth]{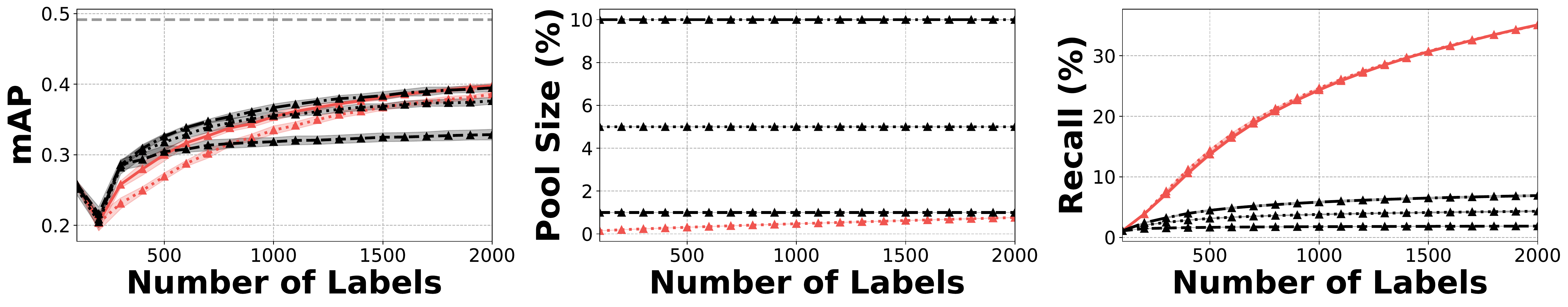}
    \caption{OpenImages ($|U|$=6,816,296)}
    \label{fig:openimages_random_pool_mostpos}
    \end{subfigure}

\caption{
MLP-\flood ($k=100$) versus MLP applied to a candidate pool of randomly selected examples (RandPool).
Because the concepts we considered were so rare, as is often the case in practice, randomly chosen examples are unlikely to be close to the decision boundary, and a much larger pool is required to match \flood.
On ImageNet (top), MLP-\flood outperformed MLP-RandPool in terms of both the error the model achieves for active learning (left) and the recall of positive examples for active search (right) even with a pool containing 10\% of the data (middle).
On Openimages (bottom), MLP-RandPool needed at least $5\times$ as much data to match MLP-\flood for active learning and failed to achieve similar recall even with $10\times$ the data.
}
  \label{fig:random_pool_mostpos}
\end{figure}

\begin{figure}[H]
  \centering
  \begin{subfigure}{0.95\textwidth}
    \includegraphics[width=\columnwidth]{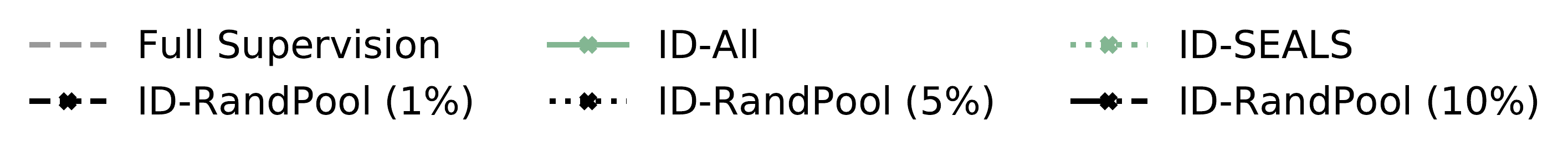}
  \end{subfigure}

  \begin{subfigure}{0.95\textwidth}
    \includegraphics[width=\columnwidth]{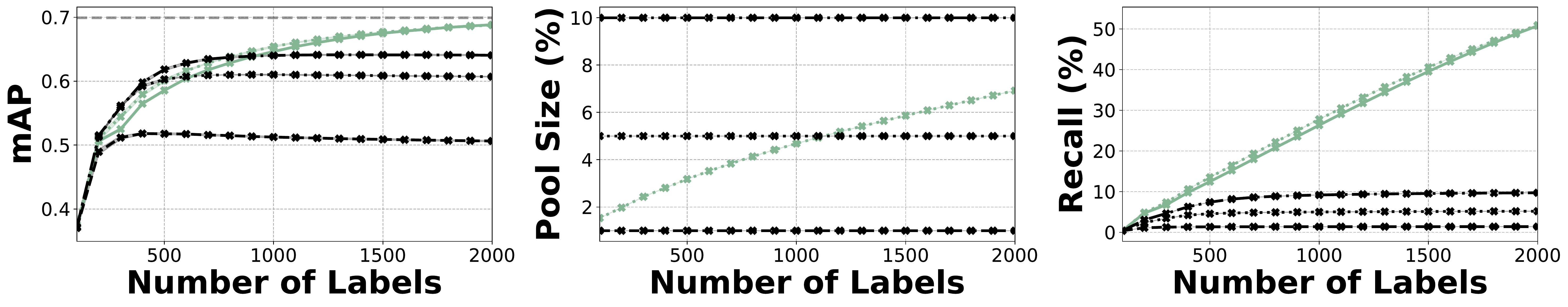}
    \caption{ImageNet ($|U|$=639,906)}
    \label{fig:imagenet_random_pool_information_density}
  \end{subfigure}

  \begin{subfigure}{0.95\textwidth}
    \includegraphics[width=\columnwidth]{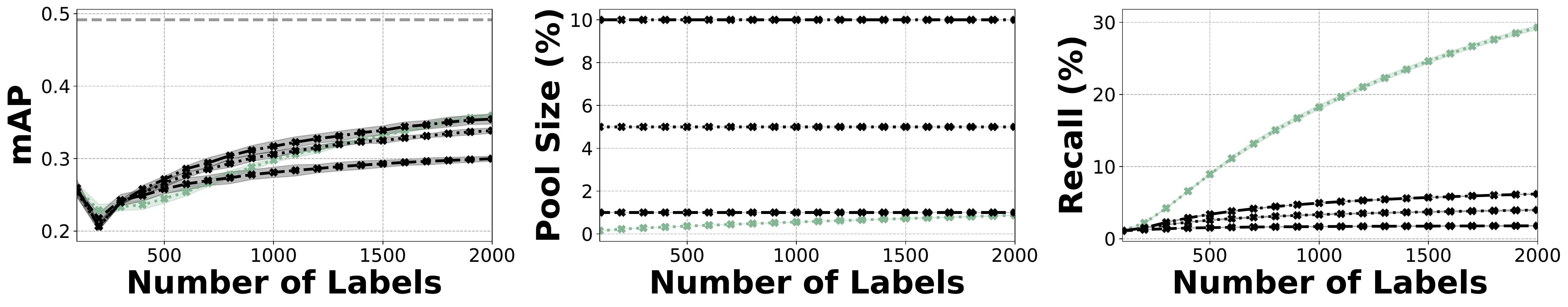}
    \caption{OpenImages ($|U|$=6,816,296)}
    \label{fig:openimages_random_pool_information_density}
    \end{subfigure}

\caption{
ID-\flood ($k=100$) versus ID applied to a candidate pool of randomly selected examples (RandPool).
Because the concepts we considered were so rare, as is often the case in practice, randomly chosen examples are unlikely to be close to the decision boundary, and a much larger pool is required to match \flood.
On ImageNet (top), ID-\flood outperformed ID-RandPool in terms of both the error the model achieves for active learning (left) and the recall of positive examples for active search (right) even with a pool containing 10\% of the data (middle).
On Openimages (bottom), ID-RandPool needed at least $5\times$ as much data to match ID-\flood for active learning and failed to achieve similar recall even with $10\times$ the data.
}
  \label{fig:random_pool_information_density}
\end{figure}

\subsection{Active learning on each selected class from OpenImages}\label{appendix:openimages_per_class_active_learning}
\begin{table}[H]
\centering
\caption{\textbf{Top $\frac{1}{3}$ of classes from Openimages for active learning.} (1 of 3) Average precision and measurements of the largest component (LC) for each selected class (153 total) from OpenImages with a labeling budget of 2,000 examples. Classes are ordered based on MaxEnt-\flood.}
\resizebox{\textwidth}{!}{
    \begin{tabular}{lccccccc}
\toprule
\makecell{\bfseries Display \\ \bfseries Name} &  \makecell{\bfseries Total\\ \bfseries Positives} &  \makecell{\bfseries Size of \\ \bfseries the LC \\ \bfseries (\%)} & \makecell{\bfseries Average \\ \bfseries Shortest \\ \bfseries Path in \\ \bfseries the LC} & \makecell{\bfseries Random \\ \bfseries (All)} &  \makecell{\bfseries MaxEnt \\ \bfseries (SEALS)} &  \makecell{\bfseries MaxEnt \\ \bfseries (All)} &  \makecell{\bfseries Full \\ \bfseries Supervision} \\
\hline
                          Citrus &         796 &  65 &                   3.34 &     $0.34$ &       $0.87$ &     $0.87$ &           $0.87$ \\
                      Cargo ship &         219 &  84 &                   2.85 &     $0.70$ &       $0.83$ &     $0.83$ &           $0.86$ \\
                      Blackberry &         245 &  87 &                   2.64 &     $0.67$ &       $0.80$ &     $0.80$ &           $0.79$ \\
                     Galliformes &         674 &  82 &                   3.98 &     $0.72$ &       $0.80$ &     $0.82$ &           $0.92$ \\
                            Rope &         618 &  59 &                   3.48 &     $0.29$ &       $0.80$ &     $0.81$ &           $0.74$ \\
                        Hurdling &         269 &  92 &                   2.48 &     $0.26$ &       $0.80$ &     $0.79$ &           $0.80$ \\
                    Roman temple &         345 &  89 &                   2.72 &     $0.63$ &       $0.79$ &     $0.79$ &           $0.82$ \\
                   Monster truck &         286 &  84 &                   2.84 &     $0.41$ &       $0.79$ &     $0.80$ &           $0.81$ \\
                           Pasta &         954 &  91 &                   3.21 &     $0.42$ &       $0.75$ &     $0.75$ &           $0.79$ \\
                           Chess &         740 &  83 &                   3.39 &     $0.53$ &       $0.73$ &     $0.74$ &           $0.86$ \\
         Bowed string instrument &         728 &  78 &                   3.05 &     $0.72$ &       $0.72$ &     $0.74$ &           $0.79$ \\
                          Parrot &        1546 &  89 &                   2.85 &     $0.59$ &       $0.72$ &     $0.76$ &           $0.92$ \\
                        Calabaza &         870 &  82 &                   3.15 &     $0.50$ &       $0.71$ &     $0.75$ &           $0.81$ \\
                       Superhero &         968 &  58 &                   5.28 &     $0.17$ &       $0.70$ &     $0.70$ &           $0.67$ \\
                           Drums &         741 &  69 &                   3.30 &     $0.52$ &       $0.70$ &     $0.72$ &           $0.83$ \\
                  Shooting range &         189 &  57 &                   3.06 &     $0.38$ &       $0.69$ &     $0.69$ &           $0.68$ \\
      Ancient roman architecture &         589 &  76 &                   3.34 &     $0.61$ &       $0.68$ &     $0.70$ &           $0.77$ \\
                        Cupboard &         898 &  88 &                   3.41 &     $0.53$ &       $0.68$ &     $0.69$ &           $0.75$ \\
                            Ibis &         259 &  93 &                   2.53 &     $0.29$ &       $0.68$ &     $0.69$ &           $0.66$ \\
                          Cattle &        5995 &  93 &                   3.22 &     $0.37$ &       $0.67$ &     $0.68$ &           $0.74$ \\
                         Galleon &         182 &  74 &                   2.54 &     $0.45$ &       $0.66$ &     $0.66$ &           $0.61$ \\
                   Kitchen knife &         360 &  63 &                   3.52 &     $0.32$ &       $0.66$ &     $0.65$ &           $0.66$ \\
                      Grapefruit &         506 &  83 &                   3.06 &     $0.50$ &       $0.65$ &     $0.65$ &           $0.69$ \\
                          Deacon &         341 &  80 &                   2.80 &     $0.48$ &       $0.64$ &     $0.64$ &           $0.67$ \\
                             Rye &         128 &  75 &                   2.63 &     $0.51$ &       $0.64$ &     $0.64$ &           $0.65$ \\
                       Chartreux &         147 &  91 &                   2.59 &     $0.50$ &       $0.63$ &     $0.63$ &           $0.69$ \\
                San Pedro cactus &         318 &  76 &                   3.32 &     $0.17$ &       $0.62$ &     $0.63$ &           $0.71$ \\
         Skateboarding Equipment &         862 &  57 &                   5.92 &     $0.20$ &       $0.62$ &     $0.66$ &           $0.66$ \\
                  Electric piano &         345 &  56 &                   4.15 &     $0.24$ &       $0.61$ &     $0.60$ &           $0.48$ \\
                           Straw &         547 &  65 &                   2.85 &     $0.33$ &       $0.61$ &     $0.62$ &           $0.61$ \\
                           Berry &         874 &  82 &                   3.78 &     $0.30$ &       $0.61$ &     $0.61$ &           $0.69$ \\
          East-european shepherd &         206 &  86 &                   2.16 &     $0.61$ &       $0.61$ &     $0.62$ &           $0.65$ \\
                            Ring &         676 &  75 &                   3.87 &     $0.15$ &       $0.61$ &     $0.64$ &           $0.64$ \\
                             Rat &        1151 &  94 &                   2.50 &     $0.32$ &       $0.60$ &     $0.60$ &           $0.61$ \\
                 Coral reef fish &         434 &  90 &                   3.07 &     $0.51$ &       $0.60$ &     $0.64$ &           $0.79$ \\
                   Concert dance &         357 &  61 &                   3.91 &     $0.37$ &       $0.60$ &     $0.60$ &           $0.70$ \\
                      Whole food &         708 &  73 &                   3.66 &     $0.18$ &       $0.58$ &     $0.60$ &           $0.57$ \\
               Modern pentathlon &         772 &  43 &                   2.59 &     $0.13$ &       $0.58$ &     $0.47$ &           $0.51$ \\
                         Gymnast &         235 &  77 &                   2.39 &     $0.39$ &       $0.57$ &     $0.59$ &           $0.65$ \\
                 California roll &         368 &  84 &                   3.49 &     $0.05$ &       $0.56$ &     $0.56$ &           $0.58$ \\
                          Shrimp &         907 &  85 &                   3.82 &     $0.07$ &       $0.56$ &     $0.56$ &           $0.58$ \\
                       Log cabin &         448 &  70 &                   3.62 &     $0.44$ &       $0.55$ &     $0.55$ &           $0.62$ \\
                  Formula racing &         351 &  88 &                   3.38 &     $0.33$ &       $0.55$ &     $0.54$ &           $0.60$ \\
                            Herd &         648 &  75 &                   3.88 &     $0.42$ &       $0.54$ &     $0.55$ &           $0.67$ \\
                      Embroidery &         356 &  81 &                   3.41 &     $0.32$ &       $0.53$ &     $0.53$ &           $0.60$ \\
                        Shelving &         810 &  66 &                   3.41 &     $0.27$ &       $0.53$ &     $0.53$ &           $0.51$ \\
                        Downhill &         194 &  84 &                   2.64 &     $0.42$ &       $0.53$ &     $0.51$ &           $0.59$ \\
                         Daylily &         391 &  87 &                   3.25 &     $0.20$ &       $0.51$ &     $0.50$ &           $0.49$ \\
             Automotive exterior &        1060 &  23 &                   2.74 &     $0.65$ &       $0.49$ &     $0.54$ &           $0.69$ \\
                   Ciconiiformes &         426 &  88 &                   3.47 &     $0.33$ &       $0.49$ &     $0.51$ &           $0.48$ \\
                       Monoplane &         756 &  81 &                   4.70 &     $0.13$ &       $0.48$ &     $0.43$ &           $0.48$ \\
\bottomrule
\end{tabular}
}
\end{table}
\clearpage

\begin{table}[H]
\centering
\caption{\textbf{Middle $\frac{1}{3}$ of classes from Openimages for active learning.} (2 of 3) Average precision and measurements of the largest component (LC) for each selected class (153 total) from OpenImages with a labeling budget of 2,000 examples. Classes are ordered based on MaxEnt-\flood.}
\resizebox{\textwidth}{!}{
    \begin{tabular}{lccccccc}
\toprule
\makecell{\bfseries Display \\ \bfseries Name} &  \makecell{\bfseries Total\\ \bfseries Positives} &  \makecell{\bfseries Size of \\ \bfseries the LC \\ \bfseries (\%)} & \makecell{\bfseries Average \\ \bfseries Shortest \\ \bfseries Path in \\ \bfseries the LC} & \makecell{\bfseries Random \\ \bfseries (All)} &  \makecell{\bfseries MaxEnt \\ \bfseries (SEALS)} &  \makecell{\bfseries MaxEnt \\ \bfseries (All)} &  \makecell{\bfseries Full \\ \bfseries Supervision} \\
\hline
                    Seafood boil &         322 &  85 &                   2.73 &     $0.31$ &       $0.48$ &     $0.49$ &           $0.51$ \\
                     Landscaping &         789 &  32 &                   4.71 &     $0.26$ &       $0.48$ &     $0.51$ &           $0.63$ \\
                         Skating &         561 &  77 &                   4.04 &     $0.17$ &       $0.48$ &     $0.43$ &           $0.40$ \\
                      Floodplain &         567 &  50 &                   4.81 &     $0.61$ &       $0.47$ &     $0.52$ &           $0.66$ \\
                        Knitting &         409 &  71 &                   3.10 &     $0.61$ &       $0.46$ &     $0.50$ &           $0.73$ \\
                             Elk &         353 &  84 &                   2.40 &     $0.15$ &       $0.46$ &     $0.48$ &           $0.45$ \\
                        Bilberry &         228 &  75 &                   3.77 &     $0.10$ &       $0.45$ &     $0.45$ &           $0.32$ \\
                            Goat &        1190 &  88 &                   3.72 &     $0.17$ &       $0.44$ &     $0.45$ &           $0.61$ \\
                   Fortification &         287 &  66 &                   3.96 &     $0.43$ &       $0.44$ &     $0.46$ &           $0.52$ \\
                    Annual plant &         677 &  38 &                   6.07 &     $0.39$ &       $0.44$ &     $0.43$ &           $0.58$ \\
 Mcdonnell douglas f/a-18 hornet &         160 &  88 &                   3.51 &     $0.11$ &       $0.44$ &     $0.47$ &           $0.37$ \\
                           Tooth &         976 &  49 &                   4.77 &     $0.16$ &       $0.44$ &     $0.48$ &           $0.56$ \\
                          Briefs &         539 &  78 &                   3.68 &     $0.15$ &       $0.43$ &     $0.44$ &           $0.46$ \\
                   Sirloin steak &         297 &  60 &                   4.97 &     $0.14$ &       $0.42$ &     $0.42$ &           $0.46$ \\
                        Smoothie &         330 &  78 &                   3.22 &     $0.15$ &       $0.41$ &     $0.41$ &           $0.38$ \\
                          Glider &         393 &  82 &                   3.94 &     $0.08$ &       $0.40$ &     $0.40$ &           $0.48$ \\
                Bathroom cabinet &         368 &  95 &                   2.39 &     $0.29$ &       $0.40$ &     $0.39$ &           $0.37$ \\
               White-tailed deer &         238 &  87 &                   3.24 &     $0.34$ &       $0.40$ &     $0.43$ &           $0.43$ \\
                    Bird of prey &         712 &  78 &                   3.81 &     $0.76$ &       $0.40$ &     $0.50$ &           $0.91$ \\
                      Egg (Food) &        1193 &  85 &                   4.31 &     $0.14$ &       $0.40$ &     $0.37$ &           $0.63$ \\
                         Soldier &        1032 &  74 &                   3.80 &     $0.62$ &       $0.40$ &     $0.41$ &           $0.72$ \\
                       Cranberry &         450 &  63 &                   4.10 &     $0.13$ &       $0.39$ &     $0.39$ &           $0.37$ \\
                          Estate &         667 &  51 &                   4.03 &     $0.47$ &       $0.39$ &     $0.40$ &           $0.54$ \\
               Chocolate truffle &         288 &  58 &                   5.47 &     $0.10$ &       $0.39$ &     $0.40$ &           $0.42$ \\
                     Town square &         617 &  58 &                   3.69 &     $0.31$ &       $0.38$ &     $0.36$ &           $0.47$ \\
                           Bakmi &         191 &  76 &                   3.34 &     $0.27$ &       $0.37$ &     $0.37$ &           $0.36$ \\
                    Trail riding &         679 &  90 &                   3.15 &     $0.21$ &       $0.37$ &     $0.37$ &           $0.38$ \\
              Aerial photography &         931 &  63 &                   3.99 &     $0.39$ &       $0.37$ &     $0.37$ &           $0.66$ \\
                          Lugger &         103 &  62 &                   3.14 &     $0.35$ &       $0.37$ &     $0.37$ &           $0.42$ \\
                     Paddy field &         468 &  70 &                   4.02 &     $0.17$ &       $0.36$ &     $0.36$ &           $0.43$ \\
                         Pavlova &         195 &  86 &                   2.60 &     $0.19$ &       $0.36$ &     $0.36$ &           $0.34$ \\
                    Steamed rice &         580 &  75 &                   4.54 &     $0.10$ &       $0.35$ &     $0.37$ &           $0.48$ \\
                          Pancit &         385 &  86 &                   3.16 &     $0.21$ &       $0.33$ &     $0.33$ &           $0.31$ \\
                         Factory &         333 &  61 &                   5.59 &     $0.17$ &       $0.33$ &     $0.34$ &           $0.35$ \\
                             Fur &         834 &  42 &                   4.31 &     $0.08$ &       $0.33$ &     $0.33$ &           $0.31$ \\
                        Stallion &         598 &  70 &                   3.58 &     $0.32$ &       $0.33$ &     $0.40$ &           $0.64$ \\
              Optical instrument &         649 &  79 &                   3.91 &     $0.15$ &       $0.33$ &     $0.33$ &           $0.28$ \\
                           Thumb &         895 &  26 &                   4.18 &     $0.07$ &       $0.32$ &     $0.39$ &           $0.41$ \\
                            Meal &        1250 &  60 &                   5.68 &     $0.52$ &       $0.32$ &     $0.38$ &           $0.59$ \\
              American shorthair &        2084 &  94 &                   3.32 &     $0.12$ &       $0.32$ &     $0.32$ &           $0.24$ \\
                        Bracelet &         770 &  46 &                   4.13 &     $0.09$ &       $0.31$ &     $0.33$ &           $0.24$ \\
      Vehicle registration plate &        5697 &  76 &                   5.89 &     $0.28$ &       $0.31$ &     $0.33$ &           $0.53$ \\
                             Ice &         682 &  50 &                   4.87 &     $0.23$ &       $0.30$ &     $0.32$ &           $0.55$ \\
                          Lamian &         257 &  80 &                   3.57 &     $0.23$ &       $0.29$ &     $0.32$ &           $0.28$ \\
                      Multimedia &         741 &  46 &                   4.12 &     $0.45$ &       $0.29$ &     $0.31$ &           $0.53$ \\
                            Belt &         467 &  41 &                   3.26 &     $0.06$ &       $0.29$ &     $0.31$ &           $0.31$ \\
                         Prairie &         792 &  44 &                   3.92 &     $0.37$ &       $0.29$ &     $0.26$ &           $0.57$ \\
                      Boardsport &         673 &  62 &                   4.08 &     $0.26$ &       $0.29$ &     $0.29$ &           $0.53$ \\
                         Asphalt &        1026 &  40 &                   4.53 &     $0.23$ &       $0.29$ &     $0.29$ &           $0.45$ \\
                  Costume design &         818 &  52 &                   3.44 &     $0.07$ &       $0.26$ &     $0.26$ &           $0.28$ \\
                         Cottage &         670 &  51 &                   4.13 &     $0.36$ &       $0.26$ &     $0.36$ &           $0.61$ \\
\bottomrule
\end{tabular}
}
\end{table}

\begin{table}[H]
\centering
\caption{\textbf{Bottom $\frac{1}{3}$ of classes from Openimages for active learning.} (3 of 3) Average precision and measurements of the largest component (LC) for each selected class (153 total) from OpenImages with a labeling budget of 2,000 examples. Classes are ordered based on MaxEnt-\flood.}
\resizebox{\textwidth}{!}{
    \begin{tabular}{lccccccc}
\toprule
\makecell{\bfseries Display \\ \bfseries Name} &  \makecell{\bfseries Total\\ \bfseries Positives} &  \makecell{\bfseries Size of \\ \bfseries the LC \\ \bfseries (\%)} & \makecell{\bfseries Average \\ \bfseries Shortest \\ \bfseries Path in \\ \bfseries the LC} & \makecell{\bfseries Random \\ \bfseries (All)} &  \makecell{\bfseries MaxEnt \\ \bfseries (SEALS)} &  \makecell{\bfseries MaxEnt \\ \bfseries (All)} &  \makecell{\bfseries Full \\ \bfseries Supervision} \\
\hline
                           Stele &         450 &  70 &                   3.74 &     $0.12$ &       $0.26$ &     $0.25$ &           $0.35$ \\
               Mode of transport &        1387 &  24 &                   4.50 &     $0.15$ &       $0.26$ &     $0.16$ &           $0.54$ \\
     Temperate coniferous forest &         328 &  59 &                   4.23 &     $0.30$ &       $0.26$ &     $0.29$ &           $0.40$ \\
                          Bumper &         985 &  37 &                   6.65 &     $0.49$ &       $0.25$ &     $0.38$ &           $0.64$ \\
                     Interaction &         924 &  15 &                   6.05 &     $0.04$ &       $0.24$ &     $0.25$ &           $0.37$ \\
                Plumbing fixture &        2124 &  89 &                   3.19 &     $0.31$ &       $0.24$ &     $0.27$ &           $0.38$ \\
                       Shorebird &         234 &  80 &                   2.76 &     $0.32$ &       $0.23$ &     $0.26$ &           $0.37$ \\
                           Icing &        1118 &  74 &                   4.20 &     $0.13$ &       $0.23$ &     $0.25$ &           $0.46$ \\
                      Wilderness &        1225 &  30 &                   4.12 &     $0.29$ &       $0.23$ &     $0.24$ &           $0.39$ \\
                    Construction &         515 &  63 &                   4.99 &     $0.13$ &       $0.23$ &     $0.26$ &           $0.34$ \\
                          Carpet &         644 &  50 &                   6.98 &     $0.05$ &       $0.23$ &     $0.28$ &           $0.43$ \\
                           Maple &        2301 &  90 &                   4.19 &     $0.06$ &       $0.22$ &     $0.21$ &           $0.36$ \\
                      Rural area &         921 &  41 &                   4.63 &     $0.33$ &       $0.22$ &     $0.28$ &           $0.50$ \\
                          Singer &         604 &  56 &                   4.06 &     $0.12$ &       $0.21$ &     $0.21$ &           $0.40$ \\
                    Delicatessen &         196 &  52 &                   2.80 &     $0.14$ &       $0.21$ &     $0.22$ &           $0.27$ \\
                           Canal &         726 &  62 &                   4.78 &     $0.22$ &       $0.21$ &     $0.26$ &           $0.46$ \\
                 Organ (Biology) &        1156 &  25 &                   3.80 &     $0.23$ &       $0.19$ &     $0.07$ &           $0.44$ \\
                           Laugh &         750 &  19 &                   6.22 &     $0.06$ &       $0.18$ &     $0.17$ &           $0.26$ \\
                         Plateau &         452 &  37 &                   3.88 &     $0.41$ &       $0.18$ &     $0.24$ &           $0.46$ \\
                           Algae &         426 &  57 &                   4.52 &     $0.15$ &       $0.18$ &     $0.19$ &           $0.26$ \\
                          Cactus &         377 &  51 &                   4.11 &     $0.05$ &       $0.17$ &     $0.18$ &           $0.22$ \\
                          Engine &         656 &  82 &                   3.43 &     $0.16$ &       $0.17$ &     $0.17$ &           $0.26$ \\
                   Marine mammal &        2954 &  91 &                   3.58 &     $0.19$ &       $0.16$ &     $0.15$ &           $0.21$ \\
                           Frost &         483 &  60 &                   4.73 &     $0.20$ &       $0.15$ &     $0.21$ &           $0.47$ \\
                           Paper &         969 &  23 &                   3.18 &     $0.16$ &       $0.15$ &     $0.14$ &           $0.41$ \\
                          Cirque &         347 &  29 &                   5.77 &     $0.43$ &       $0.15$ &     $0.40$ &           $0.55$ \\
                            Pork &         464 &  64 &                   4.44 &     $0.06$ &       $0.14$ &     $0.14$ &           $0.15$ \\
                         Antenna &         545 &  73 &                   3.66 &     $0.10$ &       $0.14$ &     $0.13$ &           $0.29$ \\
                        Portrait &        2510 &  67 &                   6.38 &     $0.23$ &       $0.13$ &     $0.18$ &           $0.43$ \\
                        Flooring &         814 &  38 &                   3.87 &     $0.10$ &       $0.13$ &     $0.14$ &           $0.20$ \\
                         Cycling &         794 &  63 &                   5.00 &     $0.53$ &       $0.13$ &     $0.28$ &           $0.66$ \\
             Chevrolet silverado &         115 &  62 &                   4.82 &     $0.05$ &       $0.09$ &     $0.08$ &           $0.12$ \\
                            Tool &        1549 &  64 &                   4.51 &     $0.08$ &       $0.09$ &     $0.10$ &           $0.13$ \\
                         Liqueur &         539 &  51 &                   5.98 &     $0.26$ &       $0.09$ &     $0.14$ &           $0.38$ \\
               Pleurotus eryngii &         140 &  84 &                   3.10 &     $0.11$ &       $0.08$ &     $0.08$ &           $0.14$ \\
                        Organism &        1148 &  21 &                   3.49 &     $0.05$ &       $0.07$ &     $0.13$ &           $0.26$ \\
                  Pelecaniformes &         457 &  85 &                   3.96 &     $0.30$ &       $0.07$ &     $0.09$ &           $0.32$ \\
                            Icon &         186 &  15 &                   3.26 &     $0.05$ &       $0.07$ &     $0.07$ &           $0.16$ \\
                         Stadium &        1654 &  77 &                   5.77 &     $0.35$ &       $0.06$ &     $0.10$ &           $0.48$ \\
                           Space &        1006 &  23 &                   4.63 &     $0.03$ &       $0.06$ &     $0.03$ &           $0.14$ \\
                 Performing arts &        1030 &  29 &                   6.97 &     $0.12$ &       $0.05$ &     $0.06$ &           $0.53$ \\
                           Mural &         649 &  41 &                   5.24 &     $0.13$ &       $0.05$ &     $0.07$ &           $0.34$ \\
                           Brown &        1427 &  16 &                   3.49 &     $0.02$ &       $0.05$ &     $0.07$ &           $0.20$ \\
                            Wall &        1218 &  27 &                   3.13 &     $0.11$ &       $0.05$ &     $0.05$ &           $0.27$ \\
                      Tournament &         841 &  47 &                   9.90 &     $0.15$ &       $0.05$ &     $0.07$ &           $0.16$ \\
                           White &        1494 &   3 &                   2.79 &     $0.02$ &       $0.03$ &     $0.01$ &           $0.10$ \\
                      Mitsubishi &         511 &  37 &                   5.14 &     $0.01$ &       $0.02$ &     $0.02$ &           $0.04$ \\
                      Exhibition &         513 &  40 &                   3.87 &     $0.03$ &       $0.02$ &     $0.02$ &           $0.14$ \\
                     Scale model &         667 &  45 &                   5.64 &     $0.05$ &       $0.02$ &     $0.02$ &           $0.13$ \\
                            Teal &         975 &  16 &                   4.08 &     $0.01$ &       $0.01$ &     $0.01$ &           $0.04$ \\
                   Electric blue &        1180 &  19 &                   3.70 &     $0.01$ &       $0.00$ &     $0.01$ &           $0.06$ \\
\bottomrule
\end{tabular}
}
\end{table}

\subsection{Active search on each selected class from OpenImages}\label{appendix:openimages_per_class_active_search}
\begin{table}[H]
\centering
\caption{\textbf{Top $\frac{1}{3}$ of classes from Openimages for active search.} (1 of 3) Recall (\%) of positives and measurements of the largest component (LC) for each selected class (153 total) from OpenImages with a labeling budget of 2,000 examples. Classes are ordered based on MLP-\flood.}
\resizebox{\textwidth}{!}{
    \begin{tabular}{lcccccc}
\toprule
\makecell{\bfseries Display \\ \bfseries Name} &  \makecell{\bfseries Total\\ \bfseries Positives} &  \makecell{\bfseries Size of \\ \bfseries the LC \\ \bfseries (\%)} & \makecell{\bfseries Average \\ \bfseries Shortest \\ \bfseries Path in \\ \bfseries the LC} & \makecell{\bfseries Random \\ \bfseries (All)} &  \makecell{\bfseries MLP \\ \bfseries (SEALS)} &  \makecell{\bfseries MLP \\ \bfseries (All)} \\
\hline
                       Chartreux &         147 &  91 &                   2.59 &      $3.5$ &    $83.9$ &  $84.6$ \\
                            Ibis &         259 &  93 &                   2.53 &      $2.0$ &    $83.9$ &  $83.9$ \\
                        Hurdling &         269 &  92 &                   2.48 &      $1.9$ &    $83.5$ &  $86.2$ \\
          East-european shepherd &         206 &  86 &                   2.16 &      $2.4$ &    $78.2$ &  $78.3$ \\
                      Blackberry &         245 &  87 &                   2.64 &      $2.0$ &    $77.5$ &  $78.5$ \\
                Bathroom cabinet &         368 &  95 &                   2.39 &      $1.4$ &    $76.8$ &  $77.1$ \\
                             Rat &        1151 &  94 &                   2.50 &      $0.5$ &    $75.1$ &  $75.2$ \\
                             Rye &         128 &  75 &                   2.63 &      $3.9$ &    $74.7$ &  $74.5$ \\
                             Elk &         353 &  84 &                   2.40 &      $1.5$ &    $73.4$ &  $74.3$ \\
                         Pavlova &         195 &  86 &                   2.60 &      $2.6$ &    $70.8$ &  $71.3$ \\
                    Seafood boil &         322 &  85 &                   2.73 &      $1.6$ &    $70.4$ &  $70.6$ \\
                    Roman temple &         345 &  89 &                   2.72 &      $1.5$ &    $69.2$ &  $68.3$ \\
                   Monster truck &         286 &  84 &                   2.84 &      $1.7$ &    $68.1$ &  $67.8$ \\
                        Downhill &         194 &  84 &                   2.64 &      $2.6$ &    $67.2$ &  $69.0$ \\
                       Shorebird &         234 &  80 &                   2.76 &      $2.1$ &    $66.8$ &  $66.4$ \\
 Mcdonnell douglas f/a-18 hornet &         160 &  88 &                   3.51 &      $3.2$ &    $66.0$ &  $67.9$ \\
                San Pedro cactus &         318 &  76 &                   3.32 &      $1.6$ &    $65.8$ &  $64.9$ \\
               Pleurotus eryngii &         140 &  84 &                   3.10 &      $3.6$ &    $65.7$ &  $66.1$ \\
                 California roll &         368 &  84 &                   3.49 &      $1.4$ &    $65.3$ &  $68.0$ \\
                         Gymnast &         235 &  77 &                   2.39 &      $2.2$ &    $64.0$ &  $64.0$ \\
                         Galleon &         182 &  74 &                   2.54 &      $2.7$ &    $62.4$ &  $61.5$ \\
                      Cargo ship &         219 &  84 &                   2.85 &      $2.3$ &    $61.1$ &  $61.7$ \\
                    Trail riding &         679 &  90 &                   3.15 &      $0.8$ &    $59.7$ &  $60.7$ \\
                         Daylily &         391 &  87 &                   3.25 &      $1.3$ &    $59.4$ &  $59.5$ \\
                      Grapefruit &         506 &  83 &                   3.06 &      $1.0$ &    $59.4$ &  $60.4$ \\
                        Bilberry &         228 &  75 &                   3.77 &      $2.2$ &    $58.9$ &  $55.2$ \\
                        Smoothie &         330 &  78 &                   3.22 &      $1.5$ &    $58.0$ &  $59.8$ \\
                      Embroidery &         356 &  81 &                   3.41 &      $1.5$ &    $57.6$ &  $57.2$ \\
                          Deacon &         341 &  80 &                   2.80 &      $1.5$ &    $57.1$ &  $57.9$ \\
                  Shooting range &         189 &  57 &                   3.06 &      $2.6$ &    $56.3$ &  $55.6$ \\
                          Glider &         393 &  82 &                   3.94 &      $1.3$ &    $55.8$ &  $57.6$ \\
               White-tailed deer &         238 &  87 &                   3.24 &      $2.2$ &    $55.8$ &  $55.9$ \\
                 Coral reef fish &         434 &  90 &                   3.07 &      $1.3$ &    $54.8$ &  $54.9$ \\
             Chevrolet silverado &         115 &  62 &                   4.82 &      $4.3$ &    $54.1$ &  $54.6$ \\
                          Lugger &         103 &  62 &                   3.14 &      $4.9$ &    $53.8$ &  $53.8$ \\
                          Pancit &         385 &  86 &                   3.16 &      $1.3$ &    $52.8$ &  $53.1$ \\
                           Chess &         740 &  83 &                   3.39 &      $0.7$ &    $51.9$ &  $50.9$ \\
                           Bakmi &         191 &  76 &                   3.34 &      $2.6$ &    $51.8$ &  $51.2$ \\
                   Kitchen knife &         360 &  63 &                   3.52 &      $1.5$ &    $50.9$ &  $53.9$ \\
                           Straw &         547 &  65 &                   2.85 &      $1.0$ &    $50.3$ &  $51.0$ \\
      Ancient roman architecture &         589 &  76 &                   3.34 &      $0.8$ &    $48.5$ &  $47.1$ \\
                          Lamian &         257 &  80 &                   3.57 &      $1.9$ &    $47.8$ &  $48.2$ \\
                         Antenna &         545 &  73 &                   3.66 &      $1.0$ &    $47.4$ &  $48.0$ \\
                        Calabaza &         870 &  82 &                   3.15 &      $0.6$ &    $46.0$ &  $45.8$ \\
                            Ring &         676 &  75 &                   3.87 &      $0.7$ &    $45.2$ &  $45.4$ \\
                   Ciconiiformes &         426 &  88 &                   3.47 &      $1.2$ &    $45.2$ &  $45.2$ \\
                       Log cabin &         448 &  70 &                   3.62 &      $1.1$ &    $44.9$ &  $45.7$ \\
         Bowed string instrument &         728 &  78 &                   3.05 &      $0.7$ &    $44.4$ &  $44.7$ \\
                           Pasta &         954 &  91 &                   3.21 &      $0.5$ &    $43.7$ &  $43.8$ \\
                        Knitting &         409 &  71 &                   3.10 &      $1.3$ &    $43.5$ &  $42.8$ \\
                            Rope &         618 &  59 &                   3.48 &      $0.8$ &    $43.0$ &  $42.8$ \\
\bottomrule
\end{tabular}
}
\end{table}
\clearpage

\begin{table}[H]
\centering
\caption{\textbf{Middle $\frac{1}{3}$ of classes from Openimages for active search.} (2 of 3) Recall (\%) of positives and measurements of the largest component (LC) for each selected class (153 total) from OpenImages with a labeling budget of 2,000 examples. Classes are ordered based on MLP-\flood.}
\resizebox{\textwidth}{!}{
    \begin{tabular}{lcccccc}
\toprule
\makecell{\bfseries Display \\ \bfseries Name} &  \makecell{\bfseries Total\\ \bfseries Positives} &  \makecell{\bfseries Size of \\ \bfseries the LC \\ \bfseries (\%)} & \makecell{\bfseries Average \\ \bfseries Shortest \\ \bfseries Path in \\ \bfseries the LC} & \makecell{\bfseries Random \\ \bfseries (All)} &  \makecell{\bfseries MLP \\ \bfseries (SEALS)} &  \makecell{\bfseries MLP \\ \bfseries (All)} \\
\hline
                  Formula racing &         351 &  88 &                   3.38 &      $1.4$ &    $42.6$ &  $41.4$ \\
                     Paddy field &         468 &  70 &                   4.02 &      $1.1$ &    $42.6$ &  $44.2$ \\
                          Engine &         656 &  82 &                   3.43 &      $0.8$ &    $41.7$ &  $40.6$ \\
                  Electric piano &         345 &  56 &                   4.15 &      $1.5$ &    $40.9$ &  $42.1$ \\
                          Shrimp &         907 &  85 &                   3.82 &      $0.6$ &    $40.4$ &  $40.8$ \\
                            Goat &        1190 &  88 &                   3.72 &      $0.4$ &    $39.6$ &  $39.6$ \\
               Chocolate truffle &         288 &  58 &                   5.47 &      $1.8$ &    $39.6$ &  $39.9$ \\
                        Cupboard &         898 &  88 &                   3.41 &      $0.6$ &    $39.6$ &  $39.6$ \\
                          Citrus &         796 &  65 &                   3.34 &      $0.7$ &    $39.3$ &  $39.6$ \\
                          Parrot &        1546 &  89 &                   2.85 &      $0.4$ &    $39.2$ &  $38.8$ \\
                    Delicatessen &         196 &  52 &                   2.80 &      $2.6$ &    $38.2$ &  $39.0$ \\
                           Berry &         874 &  82 &                   3.78 &      $0.6$ &    $37.8$ &  $37.6$ \\
                          Briefs &         539 &  78 &                   3.68 &      $1.0$ &    $37.1$ &  $37.2$ \\
                   Concert dance &         357 &  61 &                   3.91 &      $1.4$ &    $36.6$ &  $36.1$ \\
               Modern pentathlon &         772 &  43 &                   2.59 &      $0.6$ &    $35.9$ &  $32.6$ \\
                   Fortification &         287 &  66 &                   3.96 &      $1.7$ &    $35.7$ &  $37.6$ \\
                        Stallion &         598 &  70 &                   3.58 &      $0.9$ &    $35.7$ &  $36.3$ \\
                            Belt &         467 &  41 &                   3.26 &      $1.1$ &    $35.2$ &  $34.9$ \\
                   Sirloin steak &         297 &  60 &                   4.97 &      $1.8$ &    $33.9$ &  $32.7$ \\
                           Stele &         450 &  70 &                   3.74 &      $1.1$ &    $33.9$ &  $32.7$ \\
                     Galliformes &         674 &  82 &                   3.98 &      $0.7$ &    $33.9$ &  $33.9$ \\
                           Algae &         426 &  57 &                   4.52 &      $1.2$ &    $33.8$ &  $33.1$ \\
                            Herd &         648 &  75 &                   3.88 &      $0.8$ &    $33.5$ &  $33.7$ \\
                  Pelecaniformes &         457 &  85 &                   3.96 &      $1.1$ &    $33.4$ &  $37.5$ \\
                          Cactus &         377 &  51 &                   4.11 &      $1.3$ &    $33.4$ &  $35.2$ \\
                        Shelving &         810 &  66 &                   3.41 &      $0.7$ &    $33.2$ &  $33.3$ \\
                           Drums &         741 &  69 &                   3.30 &      $0.7$ &    $32.9$ &  $32.7$ \\
                       Cranberry &         450 &  63 &                   4.10 &      $1.2$ &    $32.9$ &  $33.7$ \\
                         Factory &         333 &  61 &                   5.59 &      $1.5$ &    $32.0$ &  $31.7$ \\
                  Costume design &         818 &  52 &                   3.44 &      $0.6$ &    $30.9$ &  $30.6$ \\
              Optical instrument &         649 &  79 &                   3.91 &      $0.8$ &    $30.3$ &  $32.8$ \\
                    Construction &         515 &  63 &                   4.99 &      $1.0$ &    $30.1$ &  $31.1$ \\
     Temperate coniferous forest &         328 &  59 &                   4.23 &      $1.5$ &    $30.1$ &  $27.6$ \\
                         Skating &         561 &  77 &                   4.04 &      $1.0$ &    $28.8$ &  $30.4$ \\
                      Egg (Food) &        1193 &  85 &                   4.31 &      $0.4$ &    $28.8$ &  $28.6$ \\
                    Steamed rice &         580 &  75 &                   4.54 &      $0.9$ &    $28.1$ &  $30.2$ \\
                Plumbing fixture &        2124 &  89 &                   3.19 &      $0.3$ &    $27.9$ &  $27.9$ \\
                      Whole food &         708 &  73 &                   3.66 &      $0.7$ &    $27.7$ &  $27.5$ \\
                      Boardsport &         673 &  62 &                   4.08 &      $0.8$ &    $26.8$ &  $26.5$ \\
                            Pork &         464 &  64 &                   4.44 &      $1.1$ &    $26.3$ &  $26.6$ \\
              Aerial photography &         931 &  63 &                   3.99 &      $0.6$ &    $25.8$ &  $26.1$ \\
                     Town square &         617 &  58 &                   3.69 &      $0.8$ &    $25.7$ &  $26.1$ \\
                          Estate &         667 &  51 &                   4.03 &      $0.9$ &    $24.8$ &  $25.9$ \\
                           Maple &        2301 &  90 &                   4.19 &      $0.2$ &    $24.3$ &  $24.4$ \\
                          Cattle &        5995 &  93 &                   3.22 &      $0.1$ &    $23.8$ &  $23.6$ \\
                       Superhero &         968 &  58 &                   5.28 &      $0.6$ &    $23.4$ &  $23.3$ \\
                        Bracelet &         770 &  46 &                   4.13 &      $0.6$ &    $23.2$ &  $24.8$ \\
                           Frost &         483 &  60 &                   4.73 &      $1.0$ &    $23.1$ &  $22.5$ \\
                     Scale model &         667 &  45 &                   5.64 &      $0.8$ &    $22.9$ &  $23.7$ \\
                         Plateau &         452 &  37 &                   3.88 &      $1.1$ &    $22.7$ &  $19.1$ \\
                    Bird of prey &         712 &  78 &                   3.81 &      $0.7$ &    $22.4$ &  $22.0$ \\
\bottomrule
\end{tabular}
}
\end{table}

\begin{table}[H]
\centering
\caption{\textbf{Bottom $\frac{1}{3}$ of classes from Openimages for active search.} (3 of 3) Recall (\%) of positives and measurements of the largest component (LC) for each selected class (153 total) from OpenImages with a labeling budget of 2,000 examples. Classes are ordered based on MLP-\flood.}
\resizebox{\textwidth}{!}{
    \begin{tabular}{lcccccc}
\toprule
\makecell{\bfseries Display \\ \bfseries Name} &  \makecell{\bfseries Total\\ \bfseries Positives} &  \makecell{\bfseries Size of \\ \bfseries the LC \\ \bfseries (\%)} & \makecell{\bfseries Average \\ \bfseries Shortest \\ \bfseries Path in \\ \bfseries the LC} & \makecell{\bfseries Random \\ \bfseries (All)} &  \makecell{\bfseries MLP \\ \bfseries (SEALS)} &  \makecell{\bfseries MLP \\ \bfseries (All)} \\
\hline
                           Canal &         726 &  62 &                   4.78 &      $0.7$ &    $22.4$ &  $20.9$ \\
                      Exhibition &         513 &  40 &                   3.87 &      $1.0$ &    $21.9$ &  $23.1$ \\
                          Carpet &         644 &  50 &                   6.98 &      $0.8$ &    $21.9$ &  $22.7$ \\
                       Monoplane &         756 &  81 &                   4.70 &      $0.7$ &    $21.8$ &  $20.1$ \\
                             Ice &         682 &  50 &                   4.87 &      $0.8$ &    $21.6$ &  $23.1$ \\
                             Fur &         834 &  42 &                   4.31 &      $0.6$ &    $21.2$ &  $17.3$ \\
                           Icing &        1118 &  74 &                   4.20 &      $0.4$ &    $20.5$ &  $20.1$ \\
                        Flooring &         814 &  38 &                   3.87 &      $0.6$ &    $20.4$ &  $16.9$ \\
                            Icon &         186 &  15 &                   3.26 &      $2.7$ &    $19.9$ &  $17.2$ \\
                         Prairie &         792 &  44 &                   3.92 &      $0.6$ &    $19.0$ &  $19.2$ \\
                           Tooth &         976 &  49 &                   4.77 &      $0.5$ &    $18.6$ &  $18.0$ \\
         Skateboarding Equipment &         862 &  57 &                   5.92 &      $0.6$ &    $18.1$ &  $19.3$ \\
             Automotive exterior &        1060 &  23 &                   2.74 &      $0.5$ &    $17.7$ &  $11.9$ \\
                         Cottage &         670 &  51 &                   4.13 &      $0.7$ &    $17.6$ &  $17.3$ \\
                         Soldier &        1032 &  74 &                   3.80 &      $0.5$ &    $17.3$ &  $16.8$ \\
                   Marine mammal &        2954 &  91 &                   3.58 &      $0.2$ &    $17.3$ &  $17.2$ \\
                            Tool &        1549 &  64 &                   4.51 &      $0.3$ &    $17.0$ &  $16.9$ \\
                      Multimedia &         741 &  46 &                   4.12 &      $0.7$ &    $16.8$ &  $17.1$ \\
              American shorthair &        2084 &  94 &                   3.32 &      $0.3$ &    $16.5$ &  $16.7$ \\
                         Asphalt &        1026 &  40 &                   4.53 &      $0.5$ &    $15.1$ &  $11.5$ \\
                          Singer &         604 &  56 &                   4.06 &      $0.9$ &    $14.6$ &  $13.6$ \\
                      Floodplain &         567 &  50 &                   4.81 &      $0.9$ &    $14.6$ &  $14.0$ \\
                      Rural area &         921 &  41 &                   4.63 &      $0.6$ &    $14.2$ &  $13.2$ \\
                      Mitsubishi &         511 &  37 &                   5.14 &      $1.0$ &    $12.6$ &  $11.8$ \\
                 Organ (Biology) &        1156 &  25 &                   3.80 &      $0.5$ &    $12.1$ &  $15.9$ \\
                           Paper &         969 &  23 &                   3.18 &      $0.5$ &    $12.0$ &  $14.8$ \\
                    Annual plant &         677 &  38 &                   6.07 &      $0.7$ &    $11.8$ &  $10.7$ \\
                   Electric blue &        1180 &  19 &                   3.70 &      $0.5$ &    $11.5$ &   $9.4$ \\
                         Stadium &        1654 &  77 &                   5.77 &      $0.3$ &    $10.8$ &   $9.3$ \\
                           Mural &         649 &  41 &                   5.24 &      $0.8$ &    $10.4$ &  $10.3$ \\
                            Teal &         975 &  16 &                   4.08 &      $0.5$ &     $9.9$ &  $10.4$ \\
                          Cirque &         347 &  29 &                   5.77 &      $1.5$ &     $9.9$ &   $9.8$ \\
                            Wall &        1218 &  27 &                   3.13 &      $0.4$ &     $9.3$ &  $12.0$ \\
                           Thumb &         895 &  26 &                   4.18 &      $0.6$ &     $9.3$ &  $13.8$ \\
                     Landscaping &         789 &  32 &                   4.71 &      $0.7$ &     $9.2$ &   $9.3$ \\
      Vehicle registration plate &        5697 &  76 &                   5.89 &      $0.1$ &     $8.7$ &   $8.3$ \\
                            Meal &        1250 &  60 &                   5.68 &      $0.4$ &     $8.5$ &   $9.1$ \\
                      Wilderness &        1225 &  30 &                   4.12 &      $0.4$ &     $8.5$ &   $9.8$ \\
                         Liqueur &         539 &  51 &                   5.98 &      $1.0$ &     $8.0$ &  $12.8$ \\
                           Space &        1006 &  23 &                   4.63 &      $0.5$ &     $7.8$ &   $6.3$ \\
                         Cycling &         794 &  63 &                   5.00 &      $0.6$ &     $7.3$ &   $7.8$ \\
                           Brown &        1427 &  16 &                   3.49 &      $0.4$ &     $7.2$ &   $2.6$ \\
                        Organism &        1148 &  21 &                   3.49 &      $0.4$ &     $6.8$ &   $2.0$ \\
                           Laugh &         750 &  19 &                   6.22 &      $0.7$ &     $6.6$ &   $8.6$ \\
                          Bumper &         985 &  37 &                   6.65 &      $0.5$ &     $5.9$ &   $8.3$ \\
                        Portrait &        2510 &  67 &                   6.38 &      $0.2$ &     $5.8$ &   $5.3$ \\
               Mode of transport &        1387 &  24 &                   4.50 &      $0.4$ &     $5.1$ &   $3.6$ \\
                     Interaction &         924 &  15 &                   6.05 &      $0.6$ &     $4.5$ &   $4.6$ \\
                      Tournament &         841 &  47 &                   9.90 &      $0.6$ &     $4.3$ &   $5.1$ \\
                 Performing arts &        1030 &  29 &                   6.97 &      $0.5$ &     $2.3$ &   $2.5$ \\
                           White &        1494 &   3 &                   2.79 &      $0.3$ &     $2.0$ &   $0.5$ \\
\bottomrule
\end{tabular}
}
\end{table}

\subsection{Self-supervised embeddings (SimCLR) on ImageNet}\label{appendix:simclr}

\begin{figure}[H]
  \centering
  \begin{subfigure}{0.95\textwidth}
    \includegraphics[width=\columnwidth]{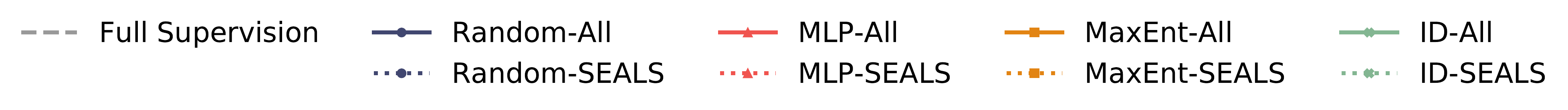}
  \end{subfigure}

  \begin{subfigure}{0.95\textwidth}
    \includegraphics[width=\columnwidth]{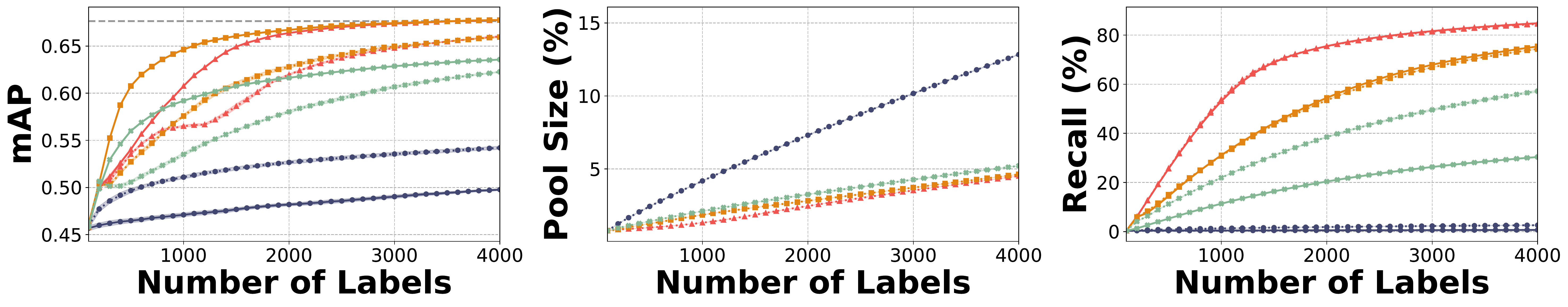}
    \label{fig:imagenet_simclr}
  \end{subfigure}

\caption{
Active learning and search on ImageNet with self-supervised embeddings from SimCLR~\citep{chen2020simple}.
Because the self-supervised training for the embeddings did not use the labels, results are average across all 1,000 classes and $|U|$=1,281,167. 
To compensate for the larger unlabeled pool, we extended the total labeling budget to 4,000 compared to the 2,000 used in Figure~\ref{fig:active_learning_and_search}.
Across strategies, \flood with $k=100$ substantially outperforms random sampling in terms of both the mAP the model achieves for active learning (left) and the recall of positive examples for active search (right), while only considering a fraction of the data $U$ (middle).
For active learning, the gap between the baseline and SEALS approaches is slightly larger than in Figure~\ref{fig:active_learning_and_search}, which is likely due to the larger pool size and increased average shortest paths (see Figure~\ref{fig:latent_structure_simclr}).
}
  \label{fig:active_learning_and_search_simclr}
\end{figure}

\begin{figure}[H]
  \centering
  \begin{subfigure}{.475\textwidth}
      \includegraphics[width=\columnwidth]{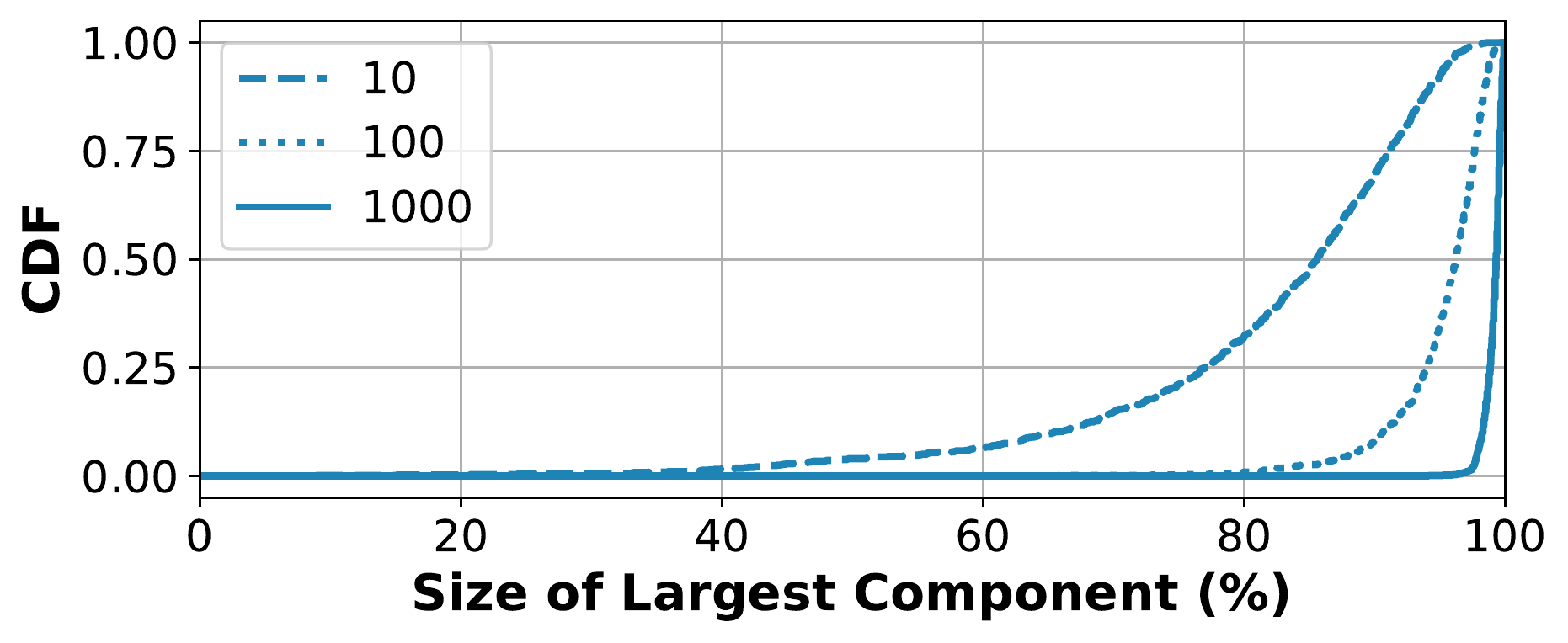}
    \label{fig:imagenet_simclr_clusters_lc}
  \end{subfigure}
  \begin{subfigure}{.475\textwidth}
      \includegraphics[width=\columnwidth]{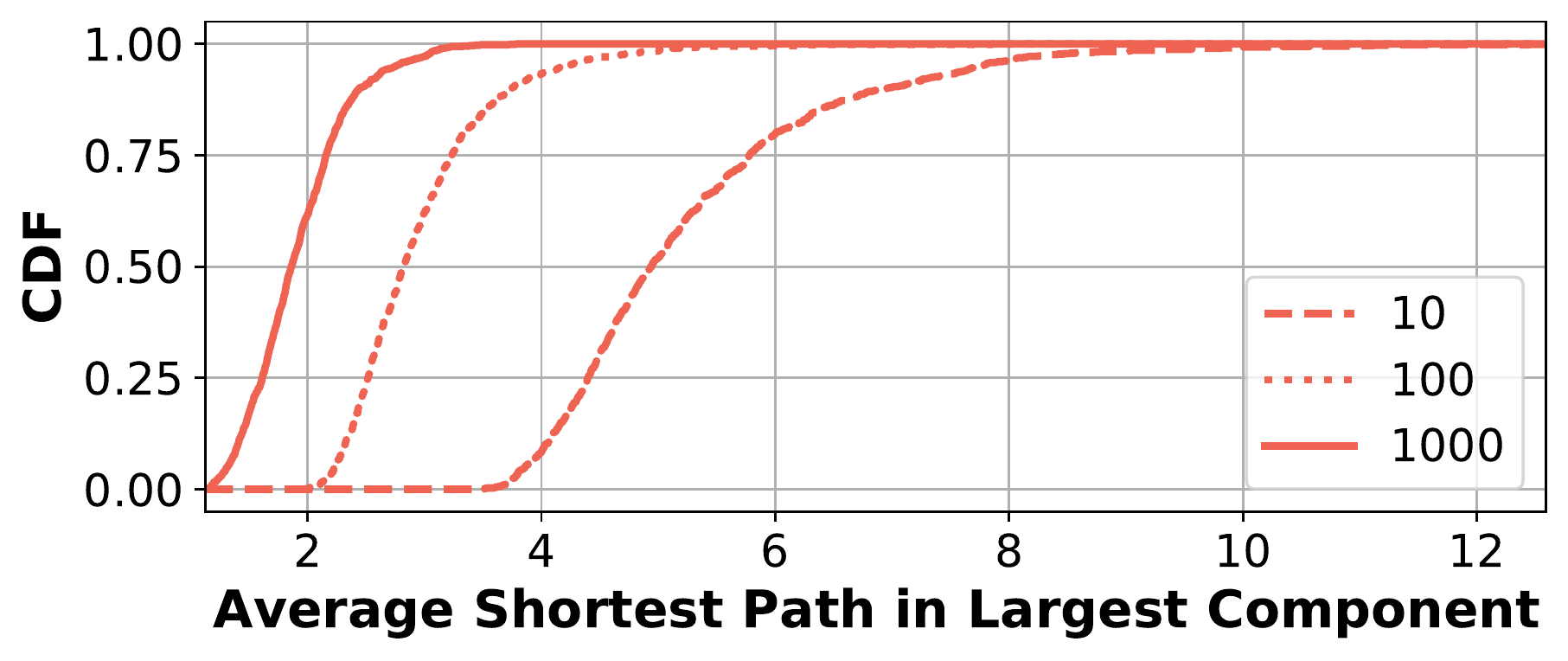}
    \label{fig:imagenet_simclr_clusters_asp}
  \end{subfigure}
\caption{
Measurements of the latent structure of unseen concepts in ImageNet with self-supervised embeddings from SimCLR~\citep{chen2020simple}.
In comparison to Figure~\ref{fig:imagenet_clusters}, the $k$-nearest neighbor graph for unseen concepts was still well connected, forming large connected components (left) for even moderate values of $k$, but the average shortest path between examples was slightly longer (right).
The increased path length is not too surprising considering the fully supervised model still outperformed the linear evaluation of the self-supervised embeddings in~\citet{chen2020simple}.
  }\vspace{-2mm}
  \label{fig:latent_structure_simclr}
\end{figure}

\clearpage

\subsection{Self-supervised embedding (Sentence-BERT) on Goodreads}\label{sec:goodreads}

We followed the same general procedure described in Section~\ref{sec:experimental_setup}, aside from the dataset specific details below.
Goodreads spoiler detection~\citep{wan2019spoiler} had 17.67 million sentences with binary spoiler annotations.
Spoilers made up 3.224\% of the data, making them much more common than the rare concepts we evaluated in the other datasets.
Following~\citet{wan2019spoiler}, we used 3.53 million sentences for testing (20\%), 10,000 sentences as the validation set, and the remaining 14.13 million sentences as the unlabeled pool.
We also switched to the area under the ROC curve (AUC) as our primary evaluation metric for active learning to be consistent with~\citet{wan2019spoiler}.
For $G_z$, we used a pre-trained Sentence-BERT model (SBERT-NLI-base)~\citep{reimers2019sentence}, applied PCA whitening to reduce the dimension to 256, and performed $l^2$ normalization.

\subsubsection{Active search}\label{sec:goodreads_active_search}
\flood achieved the same recall as the baseline approaches, but only considered less than 1\% of the unlabeled data in the candidate pool, as shown in Figure~\ref{fig:goodreads_active_learning_and_search}.
At a labeling budget of 2,000, MLP-ALL and MLP-\flood recalled $0.15\pm0.02$\% and $0.17\pm0.05$\%, respectively, while MaxEnt-All and MaxEnt-\flood achieved $0.14\pm0.04$\% and $0.11\pm0.06$\% recall respectively.
Increasing the labeling budget to 50,000 examples, increased recall to \textasciitilde3.7\% for MaxEnt and MLP but maintained a similar relative improvement over random sampling, as shown in Figure~\ref{fig:goodreads_xl}.
ID-\flood performed worse than the other strategies.
However, all of the active selection strategies outperformed random sampling by up to an order of magnitude.

\begin{figure}[H]
  \centering
  \begin{subfigure}{0.95\textwidth}
    \includegraphics[width=\columnwidth]{figures/imagenet_legend_al.pdf}
  \end{subfigure}

  \begin{subfigure}{0.95\textwidth}
    \includegraphics[width=\columnwidth]{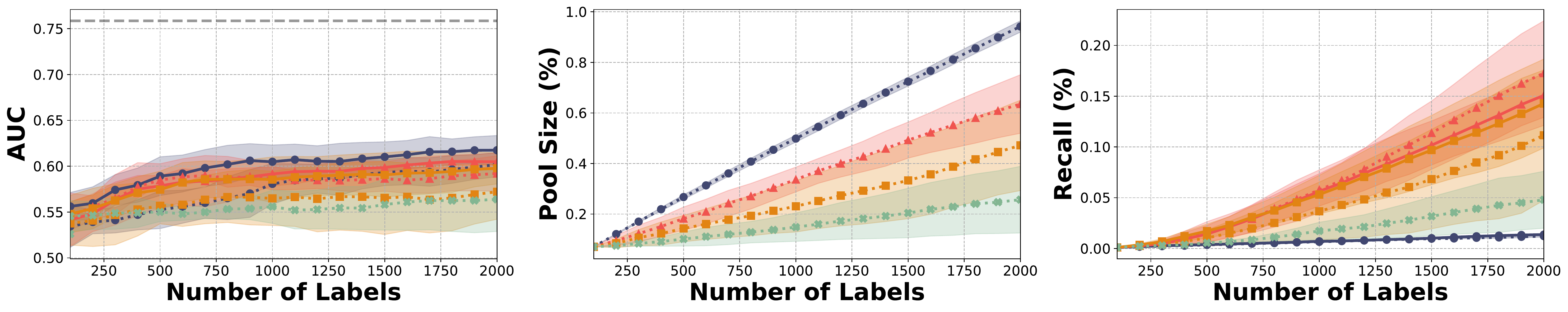}
    \label{fig:goodreads_pos5}
    \end{subfigure}
\caption{
Active learning and search on Goodreads with Sentence-BERT embeddings.
Across datasets and strategies, \flood with $k=100$ performs similarly to the baseline approach in terms of both the error the model achieves for active learning (left) and the recall of positive examples for active search (right), while only considering a fraction of the data $U$ (middle).
}
  \label{fig:goodreads_active_learning_and_search}
\end{figure}

\subsubsection{Active learning}\label{sec:goodreads_active_learning}
At a labeling budget of 2,000 examples, all the selection strategies were indistinguishable from random sampling.
Increasing the labeling budget did not help, as shown in Figure~\ref{fig:goodreads_xl}.
Unlike ImageNet and OpenImages, Goodreads had a much higher fraction of positive examples (3.224\%), and the examples were not tightly clustered as described in Section~\ref{sec:goodreads_latent_structure}.

\begin{figure}[H]
  \centering
  \begin{subfigure}{0.95\textwidth}
    \includegraphics[width=\columnwidth]{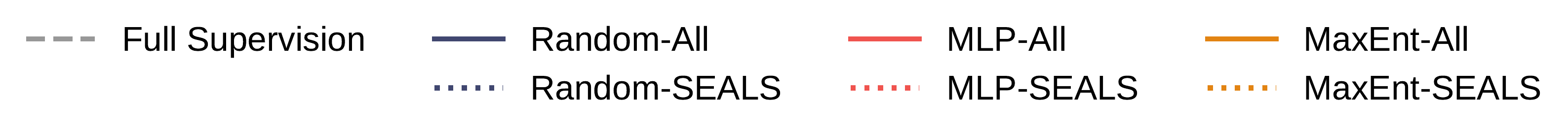}
  \end{subfigure}

  \begin{subfigure}{0.95\textwidth}
    \includegraphics[width=\columnwidth]{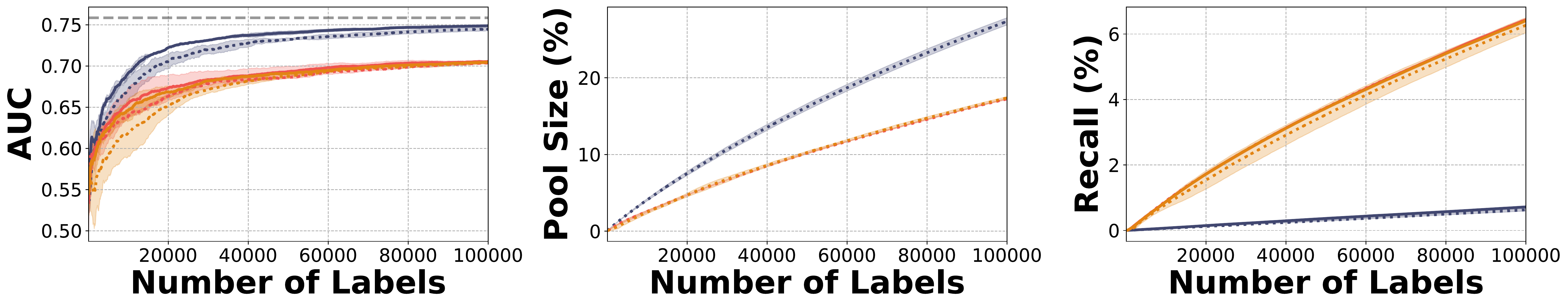}
    \label{fig:goodreads_sentence_level_pos5_xl}
  \end{subfigure}
\caption{
Active learning and search on Goodreads with a labeling budget of 100,000 examples.
Across strategies, \flood with $k=100$ performed similarly to the baseline approach in terms of both the error the model achieved for active learning (left) and the recall of positive examples for active search (right), while only considering a fraction of the data $U$ (middle).
ID was excluded because of the growing pool size and computation.
For active search, MaxEnt and MLP continued to improve recall.
For active learning, all the selection strategies (both with and without SEALS) performed worse than random sampling despite the larger labeling budget.
This gap was likely due to spoilers being book specific and the higher fraction of positive examples in the unlabeled pool, causing relevant examples to be spread almost uniformly across the space (see Section~\ref{sec:goodreads_latent_structure}).
}
  \label{fig:goodreads_xl}
\end{figure}

\subsubsection{Latent structure}\label{sec:goodreads_latent_structure}

The large number of positive examples in the Goodreads dataset limited the analysis we could perform.
We could only calculate the size of the largest connected component in the nearest neighbor graph (Figure~\ref{fig:latent_structure_goodreads}).
For $k=10$, only 28.4\% of the positive examples could be reached directly, but increasing $k$ to 100 improved that dramatically to 96.7\%.
For such a large connected component, one might have expected active learning to perform better in Section~\ref{sec:goodreads_active_learning}.
By analyzing the embeddings, however, we found that examples are spread almost uniformly across the space with an average cosine similarity of 0.004.
For comparison, the average cosine similarity for concepts in ImageNet and OpenImages was $0.453\pm0.077$ and $0.361\pm0.105$ respectively.
This uniformity was likely due to the higher fraction of positive examples and spoilers being book specific while Sentence-BERT is trained on generic data.
As a result, even if spoilers were tightly clustered within each book, the books were spread across a range of topics and consequently across the embedding space, illustrating a limitation and opportunity for future work.

\begin{figure}[H]
\centering
\includegraphics[width=0.80\textwidth]{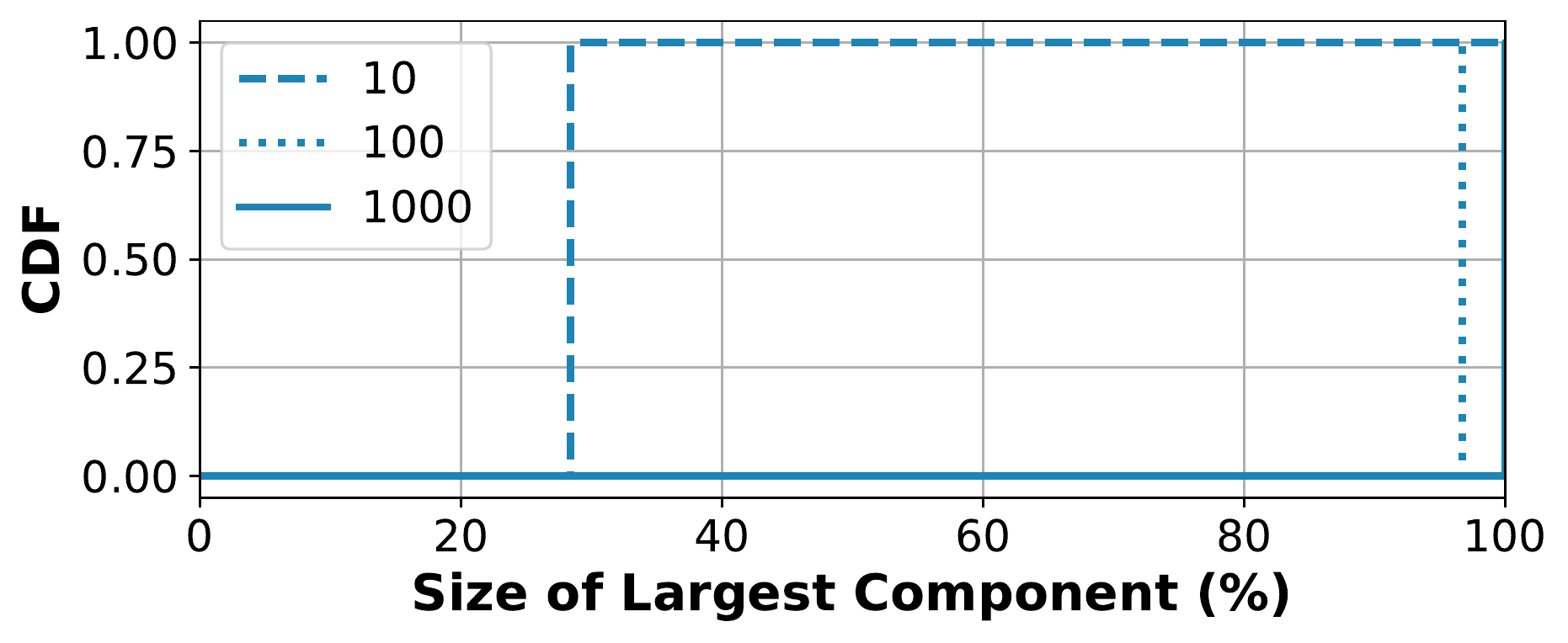}
\caption{Cumulative distribution function (CDF) for the largest connected component in the Goodreads dataset with varying values of $k$.}
\label{fig:latent_structure_goodreads}
\end{figure}
\clearpage

\end{document}